\documentclass{article}
\usepackage{iclr2021_conference,times}
\iclrfinalcopy

\usepackage{amsmath, amsfonts}
\usepackage{amsthm}
\usepackage{bm}
\usepackage{color}
\usepackage{enumitem}
\usepackage{amssymb}
\usepackage{mathtools}
\usepackage{subcaption}
\usepackage[normalem]{ulem}

\newtheorem{lemma}{Lemma}

\newtheorem{corollary}{Corollary}

\newtheorem{assumption}{Assumption}
\newcommand{\Gk}{\ensuremath{{\mathcal G}_k}}
\newcommand{\Hk}{\ensuremath{{\mathcal H}_k}}

\usepackage{mwe}
\usepackage{graphicx,caption}
\graphicspath{ {.} }
\usepackage{empheq}
\usepackage{algorithm}
\usepackage{algpseudocode}
\usepackage{thmtools}
\usepackage{thm-restate}

\DeclareMathOperator{\Tr}{Tr}
\DeclareMathOperator{\M}{\mathcal{M}}
\DeclareMathOperator{\D}{\mathcal{D}}
\DeclareMathOperator{\C}{\mathcal{C}}
\DeclareMathOperator{\Scal}{\mathcal{S}}

\DeclareMathOperator{\N}{\mathcal{N}}

\DeclareMathOperator{\Z}{\textbf{Z}}
\DeclareMathOperator{\V}{\textbf{V}}

\DeclareMathOperator{\T}{\textbf{T}}
\DeclareMathOperator{\X}{\textbf{X}}
\DeclareMathOperator{\F}{F_{\bm a}}
\DeclareMathOperator{\G}{\textbf{G}}
\DeclareMathOperator{\A}{\textbf{A}}
\DeclareMathOperator{\Cm}{\textbf{C}}
\DeclareMathOperator{\Dm}{\textbf{D}}
\DeclareMathOperator{\I}{\textbf{I}}

\DeclareMathOperator{\Q}{\textbf{Q}}
\DeclareMathOperator{\Y}{\textbf{Y}}

\DeclareMathOperator{\one}{\textbf{1}}

\DeclareMathOperator{\Ppos}{P}

\DeclareMathOperator{\diag}{diag}

\newcommand{\xdim}{n}

\newcommand{\Hm}{\mathbf{H}}

\algblock{ParFor}{EndParFor}
\algnewcommand\algorithmicparfor{\textbf{parallel for}}
\algnewcommand\algorithmicpardo{\textbf{do}}
\algnewcommand\algorithmicendparfor{\textbf{end\ parallel for}}
\algrenewtext{ParFor}[1]{\algorithmicparfor\ #1\ \algorithmicpardo}
\algrenewtext{EndParFor}{\algorithmicendparfor}

\title{Multi-Level Local SGD: Distributed SGD for Heterogeneous Hierarchical Networks}
\author{Timothy Castiglia, Anirban Das, and Stacy Patterson%
    \thanks{T. Castiglia, A. Das, and S. Patterson are with the Department of Computer Science,
        Rensselaer Polytechnic Institute, 110 8th St, Troy, NY 12180,
        {\tt\small castit@rpi.edu, dasa2@rpi.edu, sep@cs.rpi.edu}.}%
}

\begin{document}
\maketitle

\begin{abstract}
    We propose Multi-Level Local SGD, a distributed stochastic gradient method for learning a smooth, non-convex objective in a multi-level communication network with heterogeneous workers. Our network model consists of a set of disjoint sub-networks, with a single hub and multiple workers; further, workers may have different operating rates. The hubs exchange information with one another via a connected, but not necessarily complete, communication network. In our algorithm, sub-networks execute a distributed SGD algorithm, using a hub-and-spoke paradigm, and the hubs periodically average their models with neighboring hubs. We first provide a unified mathematical framework that describes the Multi-Level Local SGD algorithm. We then present a theoretical analysis of the algorithm; our analysis shows the dependence of the convergence error on the worker node heterogeneity, hub network topology, and the number of local, sub-network, and global iterations. We illustrate the effectiveness of our algorithm in a multi-level network with slow workers via simulation-based experiments. 
\end{abstract}

\section{Introduction}

Stochastic Gradient Descent (SGD) is a key algorithm in modern
Machine Learning and optimization~\citep{amari1993backpropagation}. 
To support
distributed data as well as reduce training time, 
\citet{zinkevich2010parallelized} introduced a distributed form of SGD. 
Traditionally, distributed SGD is run within a hub-and-spoke
network model: a central parameter server (hub) coordinates with
worker nodes. At each iteration, the hub sends a model to
the workers. The workers each train on their local data, taking a gradient step, then 
return their locally trained model to the hub to be averaged. 
Distributed SGD can be an efficient training mechanism when message latency 
is low between the hub and workers, allowing gradient updates
to be transmitted quickly at each iteration. 
However, as noted in~\citet{moritz2016sparknet}, message transmission latency is 
often high in distributed settings, which 
causes a large increase in overall training
time.  
A practical way to reduce this communication overhead is to allow 
the workers to take multiple local gradient steps before communicating
their local models to the hub.
This form of distributed SGD 
is referred to as Local SGD~\citep{lin2018don,stich2018local}. 
There is a large body of work that analyzes the convergence 
of Local SGD and the benefits of multiple local training rounds~\citep{pmlr-v54-mcmahan17a,
wang2018cooperative,li2019convergence}.

Local SGD is not applicable to all scenarios.
Workers may be heterogeneous in terms of their computing capabilities, 
and thus the time required for local training is not uniform.
For this reason, it can be either costly or impossible for workers
to train in a fully synchronous manner, as stragglers may hold up
global computation.
However, the vast majority of previous work uses a synchronous model, 
where all clients train for the same number of rounds before sending updates 
to the hub~\citep{dean2012large,ho2013more,cipar2013solving}. 
Further, most works assume a hub-and-spoke model, but this does not 
capture many real world settings.
For example, devices in an ad-hoc
network may not all be able to communicate to a central hub
in a single hop due to network or communication range limitations. 
In such settings, a multi-level communication network model may be beneficial.
In flying ad-hoc networks (FANETs), a network architecture has been 
proposed to improve scalability by partitioning the UAVs into mission areas~\citep{bekmezci2013flying}.
Here, clusters of UAVs have their own clusterheads, or hubs, and these hubs communicate through an upper level network, e.g., via satellite. 
Multi-level networks have also been utilized in Fog and Edge computing,
a paradigm designed to improve data aggregation and analysis 
    in wireless sensor networks,
autonomous vehicles, power systems, and more~\citep{10.1145/2342509.2342513,
NREL2017,satyanarayanan2017emergence}.

Motivated by these observations, we propose
Multi-Level Local SGD (MLL-SGD), a distributed learning algorithm 
for heterogeneous multi-level networks.
Specifically, we consider a two-level network structure. The lower
level consists of a disjoint set of hub-and-spoke \emph{sub-networks},  
each with a single hub server and a set of workers.
The upper level network consists of a connected, but not necessarily complete,
\emph{hub network} by which the hubs communicate.
For example, in a Fog Computing application, 
the sub-network workers may be edge devices connected to their
local data center, and the data centers act as hubs communicating
over a decentralized network.
Each sub-network runs one or more Local SGD rounds, in which its workers 
train for a local training period, 
followed by model averaging at the sub-network's hub.
Periodically, the hubs average their 
models with neighbors in the hub network. 
We model heterogeneous workers using a stochastic approach; 
each worker executes a local training iteration in each 
time step with a probability proportional to its computational resources.
Thus, different workers may take different numbers of gradient 
steps within each local training period. 
Note since MLL-SGD averages every local training period, 
regardless of how many gradient steps each worker takes, 
slow workers do not slow algorithm execution.

We prove the convergence of MLL-SGD for smooth and potentially 
non-convex loss functions.
We assume data is distributed in an IID manner to all workers.
Further, we analyze the relationship between 
the  convergence error and algorithm parameters
and find that, for a fixed step size, the error is quadratic
in the number of local training iterations and the number 
of sub-network training iterations, and
linear in the average worker operating
rate.
Our algorithm and analysis are general enough to encompass several variations of 
SGD as special cases, including
classical SGD~\citep{amari1993backpropagation}, 
SGD with weighted workers~\citep{pmlr-v54-mcmahan17a}, 
and Decentralized Local SGD with an arbitrary hub communication 
network~\citep{wang2018cooperative}. 
Our work provides novel analysis of a distributed learning algorithm 
in a multi-level network model with heterogeneous workers. 

The specific contributions of this paper are as follows.
$1)$ We formalize the multi-level network model with heterogeneous workers, and we define
        the MLL-SGD algorithm for training models
        in such a network. 
$2)$ We provide theoretical analysis of the convergence guarantees of MLL-SGD with heterogeneous workers. 
$3)$ We present an experimental evaluation that highlights our theoretical 
        convergence guarantees. 
        The experiments show that in multi-level networks, MLL-SGD achieves a 
        marked improvement in convergence rate over algorithms 
        that do not exploit the network hierarchy. Further, when 
        workers have heterogeneous operating rates, MLL-SGD converges 
        more quickly than algorithms that require all 
        workers to execute the same number of training steps in each local training period.

The rest of the paper is structured as follows.
In Section~\ref{related.sec}, we discuss related
work. Section~\ref{problem.sec} introduces the
system model and problem formulation. 
We describe MLL-SGD in Section~\ref{alg.sec}, and we present our main theoretical results in Section~\ref{main.sec}.
Proofs of these results are deferred to the appendix.
We provide experimental results in Section~\ref{exp.sec}.
Finally, we conclude in Section~\ref{conclusion.sec}.

\section{Related Work}\label{related.sec}

Distributed SGD is a well studied subject 
in Machine Learning.
\citet{zinkevich2010parallelized} introduced 
parallel SGD in a hub-and-spoke model. Variations on Local SGD in the
hub-and-spoke model
have been studied in several works~\citep{moritz2016sparknet,
zhang2016parallel, pmlr-v54-mcmahan17a}.
Many works have provided convergence bounds of SGD within this 
model~\citep{wang2019adaptive,li2019convergence}. 
There is also a large body of work on decentralized approaches
for optimization using gradient based methods, dual averaging,
and deep learning~\citep{tsitsiklis1986distributed,
jin2016scale,wang2019matcha}.
These previous works, however, do not address a multi-level network
structure. 

In practice, workers may be heterogeneous in nature, which means that 
they may execute training iterations at different rates.
\citet{lian2017asynchronous} 
addressed this heterogeneity by defining a gossip-based 
asynchronous SGD algorithm. 
In \citet{stich2018local}, workers are modeled to
take gradient steps at an arbitrary subset of all iterations.
However, neither of these works address a multi-level network model. 
Grouping-SGD~\citep{jiang2019novel} considers a scenario where workers
can be clustered into groups, for example, based on their operating rates. 
Workers within a group train in a synchronous manner, while
the training across different groups may be asynchronous. The system model
differs significantly from that in MLL-SGD in that as the model parameters are partitioned vertically across
multiple hubs, and workers communicate with every hub.

Several recent works analyze Hierarchical Local SGD (HL-SGD), an algorithm
for training a model in a hierarchical network. 
Different from MLL-SGD, HL-SGD assumes the hub network topology 
is a hub-and-spoke and also that
 workers are homogeneous. 
\citet{zhou2019distributed} and 
\citet{liu2020client} analyze the 
convergence error of HL-SGD, 
while \citet{abad2020hierarchical} analyzes convergence time. 
Unlike HL-SGD, MLL-SGD accounts for an arbitrary hub communication
    graph, and MLL-SGD algorithm execution does not slow down in the
presence of heterogeneous worker operating rates.

Several other works
seek to encapsulate many variations of SGD under
a single framework. 
\citet{koloskova2020unified} created a generalized
model that considers a gossip-based decentralized 
SGD algorithm where the communication network is
time-varying.
However, this work does not account for a multi-level
network model nor worker heterogeneity. 
Wang et al. introduced the Cooperative SGD 
framework~\citep{wang2018cooperative}, a model
that includes communication reduction through local SGD steps
and decentralized mixing between homogeneous workers.
Cooperative SGD also allows for auxiliary 
variables. These auxiliary variables can be used to model SGD 
in a multi-level network, but only when sub-network averaging is
immediately followed by hubs averaging with their neighbors in the hub network.
Our model is more general; it considers heterogeneous workers
and it allows for an arbitrary number of averaging 
rounds within each sub-network between averaging rounds across sub-networks, 
which is more practical in multi-level networks where inter-hub 
communication is slow or costly.

\section{System Model and Problem Formulation}\label{problem.sec}
In this section, we introduce our system model, 
the objective function that we seek to minimize,
and the assumptions we make about the function.


We consider a set of $D$ sub-networks $\D=\{1,\ldots ,D\}$. 
Each sub-network $d \in \D$ has a single hub and 
a set of workers $\M^{(d)}$,  
with $|\M^{(d)}| = N^{(d)}$. 
Workers in $\M^{(d)}$ only communicate with
their own hub and not with any other workers or hubs.
We define the set of all workers in the system as
$\M = \bigcup_{d=1}^{D} \M^{(d)}$. Let $|\M| = N$.
Each worker $i$ holds a set $\Scal^{(i)}$ of local training data.
Let $\Scal = \bigcup_{i=1}^{N} \Scal^{(i)}$. 
The set of all $D$ hubs is denoted $\C$.
The hubs communicate with one another via an undirected,
connected communication graph $G=(\C,E)$. 
Let $\mathcal{N}_d=\{j~|~e_{d,j} \in E\}$ denote the set of neighbors 
of the hub in sub-network $d$ in the hub graph $G$. 



Let the model parameters be denoted by $\bm x \in \mathbb{R}^{\xdim}$.
Our goal is to find an $\bm x$ that minimizes 
the following objective function over the training set:
\begin{align}
    F(\bm x) = \frac{1}{|\Scal|} \sum_{s \in \Scal} f(\bm x ; s)
\end{align}
where $f(\cdot)$ is the loss function. 
The workers collaboratively minimize this loss function, in part by
executing local iterations of SGD over their training sets.
For each executed local iteration, a worker samples a 
mini-batch of data uniformly at random from its local data.
Let $\xi$ be a randomly sampled mini-batch of data 
and let $g(\bm x ; \xi) = \frac{1}{|\xi|} 
\sum_{s \in \xi} \nabla f(\bm x ; s)$ be the mini-batch gradient.
For simplicity, we use $g(\bm x)$ instead of 
$g(\bm x ; \xi)$ from here on.

\begin{assumption} \label{F.assum}
The objective function and the mini-batch gradients satisfy the following:
\begin{enumerate}[label=1\alph*] 
    \item \label{smooth.asp} 
        The objective function $F: \mathbb{R}^{\xdim} \rightarrow \mathbb{R}$ 
        is continuously differentiable, and the gradient is Lipschitz 
        with constant $L  > 0$, i.e., 
        $\| \nabla F(\bm x) - \nabla F(\bm y) \|_2 \leq L \| \bm x - \bm y \|_2$
        for all $\bm x, \bm y \in  \mathbb{R}^{\xdim}$.
    \item \label{lower.asp} 
        The function $F$ is lower bounded, i.e.,  
        $F(\bm x) \geq F_{inf} > - \infty$
        for all $\bm x \in \mathbb{R}^{\xdim}$.
    \item \label{unbias.asp} 
        The mini-batch gradients are unbiased, i.e., 
        $\mathbb{E}_{\xi|\bm x}[g(\bm x)] = \nabla F(\bm x)$
        for all $\bm x \in \mathbb{R}^{\xdim}$.         
    \item \label{variance.asp} 
        There exist scalars $\beta \geq 0$ and $\sigma \geq 0$ 
        such that
        $\mathbb{E}_{\xi|\bm x} \| g(\bm x)- \nabla F(\bm x) \|_2^2 
                \leq \beta ||\nabla F(\bm x)||_2^2 + \sigma^2$
        for all $\bm x \in \mathbb{R}^{\xdim}$. 
\end{enumerate}
\end{assumption}
Assumption~\ref{smooth.asp} requires that the gradients do not 
change too rapidly, and Assumption~\ref{lower.asp} requires 
that our objective function is lower bounded by some $F_{inf}$.
Assumptions~\ref{unbias.asp} and \ref{variance.asp} 
assume that the local data at
each worker can be used as an unbiased estimate for the
full dataset with the same bounded variance.
These assumptions are common in convergence analysis of SGD 
algorithms (e.g., ~\cite{bottou2018optimization}).

\begin{algorithm}[t]
    \begin{algorithmic}[1]
        \State {\textbf{Initialize:}} $\bm y^{(d)}_1$ for hubs $d=1, \ldots, D$
     \For {$k = 1, \ldots, K$}
     \ParFor {$d \in \D$}
        \If{$k$ mod $\tau = 1$}
            \State $\bm x^{(i)}_{k} \gets \bm y^{(d)}_k$ \label{localsgd_start.line}
            \Comment Workers receive updated model from hub
        \EndIf
		\ParFor {$i \in \M^{(d)}$}
        \State $\bm x^{(i)}_{k+1} \gets \bm x^{(i)}_k - \eta \bm g^{(i)}_k$ \label{local.line}
		 	\Comment Local iteration (probabilistic)
		\EndParFor  
     \If{$k$ mod $\tau = 0$}
         \State $\bm z^{(d)} \gets \sum_{i \in \M^{(d)}} v^{(i)} \bm x^{(i)}_{k+1}$     \label{localsgd_end.line}
         \Comment Hub $d$ computes average of its workers' models
     \EndIf
     \If{$k$ mod $q\cdot \tau = 0$}
     \State  $\bm y^{(d)}_{k+1} \gets \sum_{j \in \N^{(d)}} \Hm_{j,d} \bm z^{(j)}$ \label{globalavg.line}
	     	\Comment Hub $d$ averages its model with neighboring hubs
	\Else
        \State $\bm y^{(d)}_{k+1} \gets \bm z^{(d)}$
     \EndIf
  \EndParFor
   \EndFor
    \end{algorithmic}
    \caption{Multi-Level Local SGD}
    \label{mll-sgd.alg}
\end{algorithm}

\section{Algorithm}\label{alg.sec}

We now present our Multi-Level Local SGD (MLL-SGD) algorithm. The
pseudocode is shown in Algorithm~\ref{mll-sgd.alg}.
Each sub-network trains in parallel
and, periodically, the hubs average their models with neighboring hubs.
The steps corresponding to Local SGD are shown in lines \ref{localsgd_start.line}-\ref{localsgd_end.line}.
Each hub and worker stores a copy of the model.
For worker $i \in \M^{(d)}$, we denote its copy of the local model by $\bm x^{(i)}$.
We denote the model at hub $d$ by  $\bm y^{(d)}$.
The hub first sends its model to its workers, and the workers update their local models to match their hub's model.  
Workers then execute multiple local training iterations, 
shown in line \ref{local.line}, to refine their local models independently.
To represent the different rates of computation at each worker, we use a probabilistic approach.
We assume that, in expectation, a worker $i$ execute $\tau^{(i)}$ local 
iterations for every $\tau$ time steps ($\tau^{(i)} \leq \tau$). We thus define
the N-vector $\bm p$ where each entry
$\bm p_i=\frac{\tau^{(i)}}{\tau}$ is the probability with which worker $i$ 
executes a local gradient step in each iteration $k$.
Worker $i$ updates its local model at iteration $k$ as follows: 
\begin{align}
    \bm x^{(i)}_{k+1} = \bm x^{(i)}_k - \eta \bm g_k^{(i)} 
\end{align}
where $\eta$ is the step size and $\bm g_k^{(i)}$ is a random variable such that
\begin{align}
    \bm g_k^{(i)} =   
    \begin{cases}
        g(\bm x_k^{(i)}) & \text{w/ probability } \bm p_{i} \\
        \bm 0 & \text{w/ probability } 1-\bm p_{i}.
   \end{cases}
\end{align}
After $\tau$ time steps, the hub updates its model based on the models of its workers (line \ref{localsgd_end.line}).
For each worker $i$, we assign a positive weight $w^{(i)}$. 
Let $v^{(i)}$ be the weight for worker $i$ normalized within its 
sub-network:  $ v^{(i)} = \frac{w^{(i)}}{\sum_{j \in \M^{(d(i))}} w^{(j)}}$,
where $d(i)$ denotes the sub-network of worker $i$.
Each hub's updates its model to be a weighted average over the workers' models in its sub-network:
$\bm y^{(d)} = \sum_{i \in \M^{(d)}} v^{(i)} \bm x^{(i)}$.
Weights may be assigned for
different reasons. If all worker gradients are 
treated equally, then $w^{(i)} = 1$ and $v^{(i)}=\frac{1}{N^{(d(i))}}$. 
We may also weight a worker's gradient proportional to its local dataset
size, in which case
$w^{(i)} = |\Scal^{(i)}|$ and
$v^{(i)} = \frac{|\Scal^{(i)}|}{\sum_{r \in \M^{(d(i))}} |\Scal^{(r)}|}$. The latter approach 
is used in Federated Averaging~\citep{pmlr-v54-mcmahan17a}.

After $q$ iterations of Local SGD in each sub-network 
($q \cdot \tau$ time steps), the hubs average their 
models with their neighbors in the hub communication network (line  \ref{globalavg.line}). 
The weight assigned to each hub's model is defined by 
a $D \times D$ matrix $\Hm$ so that:  
\begin{align}
    \bm y^{(d)} = \sum_{j \in \N^{(d)}} \Hm_{j,d} \bm y^{(j)} .
\end{align}
Define the total weight in the network to be 
${w_{tot} = \sum_{i \in \M} w^{(i)}}$.
Let $\bm b$ be a $D$-vector with each component $d$ given 
by $\bm b_d = (\sum_{i \in \M^{(d)}} w^{(i)}) / w_{tot}$.
We assume $\Hm$ meets the following requirements.
\begin{assumption} \label{H.assum}
    The matrix $\Hm$ satisfies the following:
    \begin{enumerate}[label=2\alph*]
        \item If $(i,j) \in E$, then $\Hm_{i,j} > 0$. Otherwise, $\Hm_{i,j}  = 0$.
        \item $\Hm$ is column stochastic, i.e.,  $\sum_{i=1}^D \Hm_{i,j} = 1$.
        \item For all $i,j \in \D$, we have $\Hm_{i,j} \bm b_j  = \Hm_{j,i} \bm b_i$. 
    \end{enumerate}
\end{assumption}
Assumption~\ref{H.assum} implies that $\Hm$ has one as a simple 
eigenvalue, with corresponding right eigenvector $\bm b$ and 
left eigenvector $\one_{D}$.
Further, all of its other eigenvalues have magnitude strictly less than $1$ 
(since $G$ is connected) \citep{rotaru2004dynamic}.
By defining $\Hm$ in this way, we ensure that the contributions 
from the workers' gradients in each hub are incorporated in proportion to the workers' weights.
This weighted averaging approach allows us to naturally extend Federated Averaging 
to the multi-level network model. 

\section{Analysis}\label{main.sec}

We note that hubs are essentially stateless in MLL-SGD,
as the hub models are copied to all workers after each
sub-network or hub averaging. Thus, our analysis focuses
on how worker models evolve. We first present an equivalent 
formulation of the MLL-SGD algorithm in terms of the evolution of the worker models.
We then present our main result on the convergence of MLL-SGD.

The system behavior can be summarized by the
following update rule for worker models:
\begin{align}
\X_{k+1} = (\X_k - \eta \G_k)\T_k
\label{x_recur.eq}
\end{align}
where $n \times N$ matrix $\X_k = [\bm x_k^{(1)}, \ldots ,\bm x_k^{(N)}]$,
$n \times N$ matrix $\G_k = [{\bm g}_k^{(1)}, \ldots ,{\bm g}_k^{(N)}]$,
and $N \times N$ matrix $\T_k$ is a time-varying operator that captures the three stages in
MLL-SGD: local iterations, hub-and-spoke
 averaging within each sub-network, 
and averaging across the hub network. 
We define  $\T_k$ as follows:
\begin{align} \label{Tk.def}
    \T_k = 
    \begin{cases}
        \Z & \text{if } k \bmod q\tau = 0 \\
        \V & \text{if } k \bmod \tau = 0~\text{and}~ k \bmod q\tau \neq 0\\
        \I & \text{otherwise}. 
   \end{cases}
\end{align}
For
local iterations, $\T_k = \I$, as there are no interactions
between workers or hubs.
For sub-network averaging,
$\V$ is an $N \times N$ block diagonal matrix, with each block 
$\V^{(d)}$ corresponding to a single sub-network $d$.
The matrix $\V^{(d)}$ is an $N^{(d)} \times N^{(d)}$ matrix where each
entry is $\V^{(d)}_{i,j} = v^{(i)}$.
Finally,
we define
an $N \times N$  matrix $\Z$
that captures the sub-network averaging and hub network
averaging in one operation that involves all workers.
The components of $\Z$ are given by
    \begin{align}
        \Z_{i,j} = \Hm_{d(i),d(j)} v^{(i)}.
    \end{align}
Let $\bm a$ be an $N$-vector with each component 
$\bm a_i = \frac{w^{(i)}}{w_{tot}}$ representing the weight of
worker $i$, normalized over all worker weights. 
We observe that  $\Z$ and $\V$ satisfy the following:
each have a right eigenvector of $\bm a$ and left eigenvector of 
$\one_{N}^T$ with eigenvalue $1$ and all other eigenvalues 
have magnitude strictly less than $1$.
The proof of these properties can be found in the appendix.
These properties are necessary (but not sufficient) to ensure that 
the worker models converge to a consensus model, 
where each worker's updates have been incorporated according 
to the worker's weight.

As is common, 
we study an averaged model
over all workers in the system~\citep{
yuan2016convergence,wang2018cooperative}. 
Specifically, we define a weighted average model: 
\begin{align}
\bm u_k = \X_k \bm a.
\end{align}
We identify the recurrence relation of $\bm u_k$. If we multiply $\bm a$
on both sides of (\ref{x_recur.eq}):
\begin{align}
    \X_{k+1} \bm a &= (\X_k - \eta \G_k)\T_k \bm a \\ 
    \bm u_{k+1} &= \bm u_k - \eta \G_k \bm a \label{Tgone.eq} \\ 
    \bm u_{k+1} &= \bm u_k - \eta \sum_{i=1}^{N} \bm a_i \bm g_k^{(i)} 
\end{align}
where (\ref{Tgone.eq}) follows from $\bm a$ being a right eigenvector
of $\V$ and $\Z$ with eigenvalue $1$. 
We note that $\bm u_k$ is updated via a
stochastic gradient descent step using a weighted average 
of several mini-batch gradients. 
Since $F(\cdot)$ may be non-convex, SGD may converge to a local minimum
or saddle point. Thus, we study the gradients of $\bm u_k$ as $k$ increases.


We next provide the main theoretical result of the paper.
\begin{restatable}{thm}{main} \label{main.thm}
    Under Assumptions~\ref{F.assum} and \ref{H.assum}, 
   if $\eta$ satisfies the following for all $i \in \M$:
\begin{align}
       (4\bm p_i - \bm p_i^2 - 2) \geq \eta L \left(\bm a_i\bm p_{i}(\beta+1)-\bm a_i\bm p_i^2 + \bm p_i^2\right) + 8L^2\eta^2q^2\tau^2 \Gamma \label{probreq.eq}
\end{align}
where $\Gamma = \frac{\zeta}{1 - \zeta^2} + \frac{2}{1-\zeta} + \frac{\zeta}{(1-\zeta)^2}$ 
and $\zeta=\max\{|\lambda_2(\Hm)|, |\lambda_N(\Hm)|\}$,
then the 
expected square norm of the average model gradient, averaged over $K$ iterations, 
is bounded as follows: 
\begin{align} 
    \mathbb{E}\Bigg[ \frac{1}{K} \sum_{k=1}^K &\Vert\nabla F(\bm u_k)\Vert_2^2 \Bigg]
\leq \frac{2\left(F(\bm x_1) - F_{inf}]\right)}{\eta K}
        + \sigma^2 \eta L\sum_{i=1}^N  \bm a_i^2 \bm p_i  
        \nonumber \\ &
        + 4L^2 \eta^2\sigma^2  
        q^3\tau^3 \left(\frac{1}{q\tau}-\frac{1}{K} \right) 
        \left(\frac{\zeta^2}{1 - \zeta^2} + \frac{2\zeta}{1-\zeta} + \frac{1}{(1-\zeta)^2} \right) \Ppos 
        \nonumber \\ &
        + 4L^2\eta^2\sigma^2 \left( \frac{2-\zeta}{1-\zeta} \right)
        \left(\tau^2\frac{(q-1)(2q-1)}{6}
        +\frac{(\tau-1)(2\tau-1)}{6}\right) \Ppos 
    \label{con.thm} \\
    \xrightarrow{\mathit{K \rightarrow \infty}}~~~&
        \sigma^2 \eta L\sum_{i=1}^N  \bm a_i^2 \bm p_i  
        + 4L^2 \eta^2\sigma^2  
        q^2\tau^2  
        \left(\frac{\zeta^2}{1 - \zeta^2} + \frac{2\zeta}{1-\zeta} + \frac{1}{(1-\zeta)^2} \right) \Ppos 
        \nonumber \\ &
        + 4L^2\eta^2\sigma^2 \left( \frac{2-\zeta}{1-\zeta} \right)
        \left(\tau^2\frac{(q-1)(2q-1)}{6}
        +\frac{(\tau-1)(2\tau-1)}{6}\right) \Ppos 
\end{align}
where $\Ppos = \sum_{i=1}^N \bm a_i \bm p_i$.
\end{restatable}
The proof of Theorem~\ref{main.thm} is provided in the appendix.
The first term in (\ref{con.thm}) is the same as in 
centralized SGD~\citep{bottou2018optimization}. As $K \rightarrow \infty$,
this term goes to zero. The second term is similar to centralized
SGD as well. 
If the stochastic gradients have high variance, 
then the convergence
error will be larger. This term is also related to the convergence
error in distributed SGD~\citep{bottou2018optimization}, which is equivalent 
to MLL-SGD when there is one sub-network,
$q=\tau=1$, $\bm a_i=1/N$, and $\bm p_i=1$ for all $i$.
MLL-SGD has a dependence on 
the probabilities of gradient steps and worker weights,
replacing the $\frac{1}{N}$ in the equivalent
term in distributed SGD.

The third and fourth terms in (\ref{con.thm}) are additive errors that 
depend on the topology of the hub network.
The value of $\zeta$ is given by 
the second largest eigenvalue of $\Hm$, by magnitude, which is
 an indication of the sparsity of the hub network.
When worker weights are uniform, a fully connected hub graph $G$ will 
have $\zeta=0$, while a sparse $G$ will typically have $\zeta$ close to $1$. 
It is interesting to note that $\zeta$ only depends on $\Hm$, and not
$\Z$ or $\V$, meaning the convergence error does not depend on how 
worker weights are distributed within sub-networks.

We also note the third and fourth terms depend on $\Ppos$, the weighted average probability
of the workers. 
The convergence error increases
as the average worker operating rate increases. This relation is expected
as more local iterations will increase convergence error~\citep{wang2018cooperative}.
It is interesting to note that the convergence error does not depend on 
the distribution of $\bm p$, meaning that a skewed and uniform distribution with
the same average probability would have the same convergence error.
We observe that the condition on $\eta$ in (\ref{probreq.eq})
cannot always be satisfied
given certain probabilities. Specifically, when there exists a 
$\bm p_i \leq 2 - \sqrt{2} \approx 0.59$, then the left-hand side will be non-positive,
and the inequality can no longer be satisfied. 
Although this may be a conservative bound, intuitively, 
when $\bm p_i$'s are below this threshold, the algorithm
may not make sufficient progress in each time step to guarantee 
convergence.

The third and fourth terms also grow with $q$ and $\tau$, 
the number of local iterations per hub network averaging and 
sub-network averaging steps, respectively.
The longer workers train locally without reconciling their models, the more 
their models will diverge, leading to larger convergence error. 
We can see that $\tau$ plays a slightly larger role in convergence
error than $q$. For a given $q \cdot \tau$, meaning a given 
number of time steps between hub averaging steps,
a larger $\tau$ leads to higher convergence error
than a larger $q$ would. Thus, there is a slight penalty to 
performing more local iterations between sub-network averaging steps.
We explore this more in Section~\ref{exp.sec}.

We note that when setting $a_i=1/N$ and $p_i=1$ for all workers $i$, and 
    setting $q=1$, MLL-SGD reduces to Cooperative SGD. 
    However, the bound in Theorem~\ref{main.thm}
    differs from that of Cooperative SGD. Specifically, Theorem~\ref{main.thm} 
    has error terms dependent on $\tau^2$ as opposed to $\tau$ 
    in Cooperative SGD. 
This discrepancy is due to accommodating all possible values of $p_i$.
More details can be found in Appendix~\ref{coop.sec}.  

    In the following corollary, we analyze the convergence rate of 
    Algorithm~\ref{mll-sgd.alg} when $\eta=\frac{1}{L\sqrt{K}}$.
    \begin{corollary} \label{main.cor}
    Let $\eta = \frac{1}{L\sqrt{K}}$ and let $q^2\tau^2 \leq \sqrt{K}$. 
    If $q\tau < K$, then 
\begin{align} 
    \mathbb{E}\Bigg[ \frac{1}{K} \sum_{k=1}^K \Vert\nabla F(\bm u_k)\Vert_2^2 \Bigg]
    &\leq O\left(\frac{L}{\sqrt{K}}\right) \left(F(\bm x_1) - F_{inf}]\right)
    + O\left(\frac{\sigma^2}{\sqrt{K}}\right) 
\end{align} 
\end{corollary}
Under the conditions given in Corollary~\ref{main.cor}, MLL-SGD achieves the same
asymptotic convergence rate as Local SGD 
and HL-SGD. 

\begin{figure}[t]
    \begin{subfigure}{0.5\textwidth}
        \centering
        \includegraphics[width=0.65\textwidth]{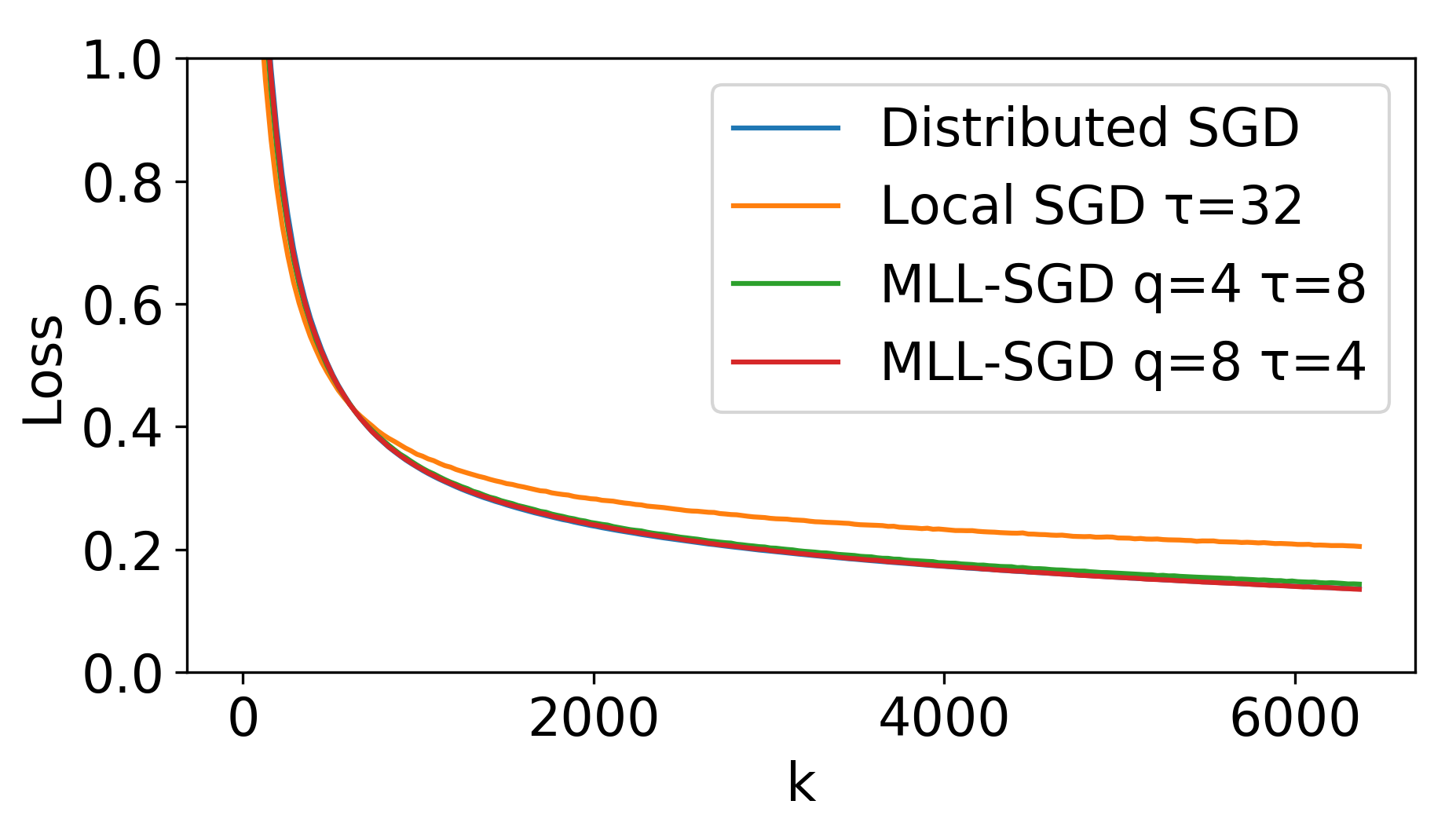}
        \caption{Training loss of CNN trained on EMNIST.}
        \label{l_qtau1.fig}
    \end{subfigure}
    \begin{subfigure}{0.5\textwidth}
        \centering
        \includegraphics[width=0.65\textwidth]{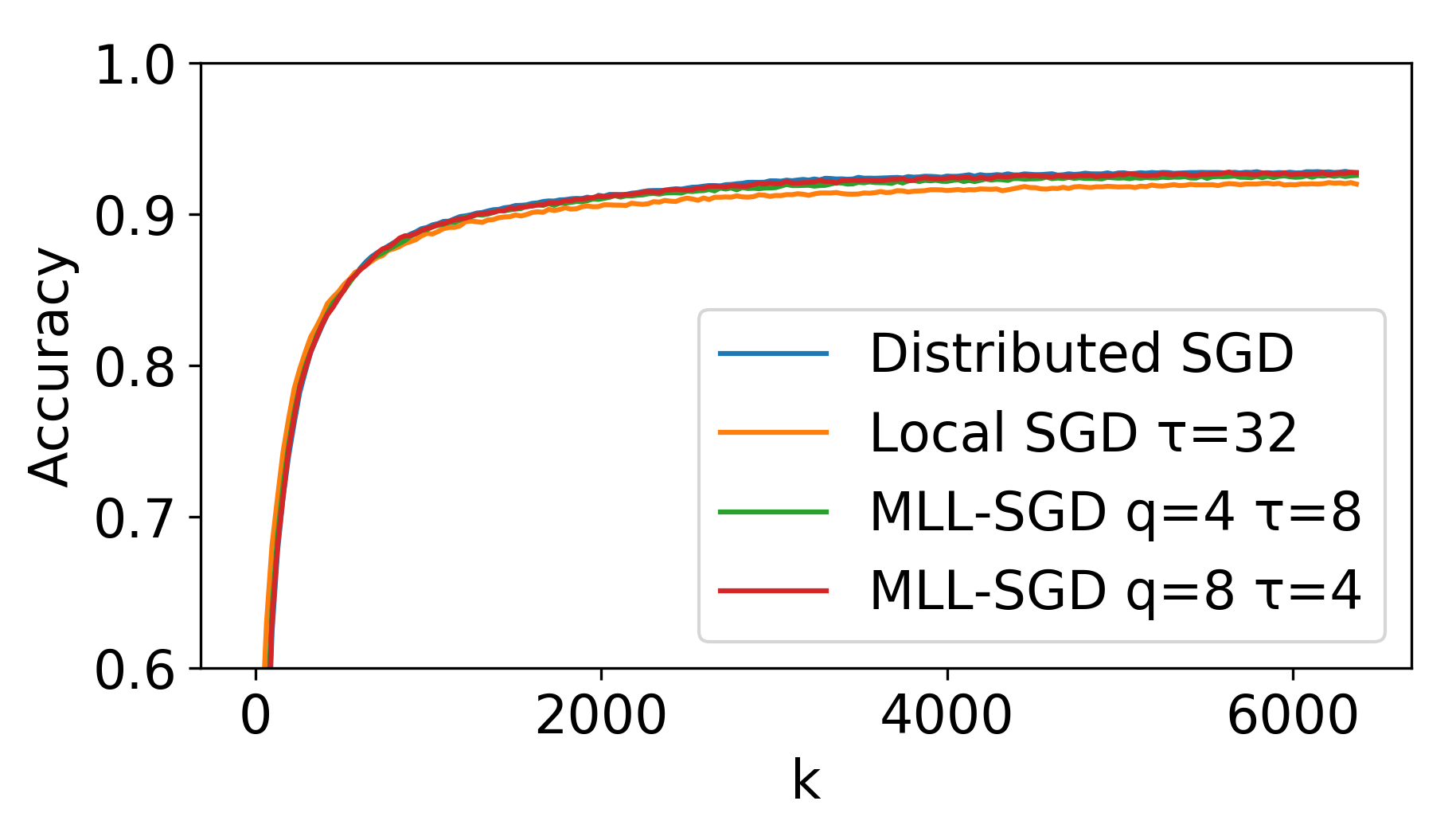}
        \caption{Test accuracy of CNN trained on EMNIST.}
        \label{a_qtau1.fig}
    \end{subfigure}
    \begin{subfigure}{0.5\textwidth}
        \centering
        \includegraphics[width=0.65\textwidth]{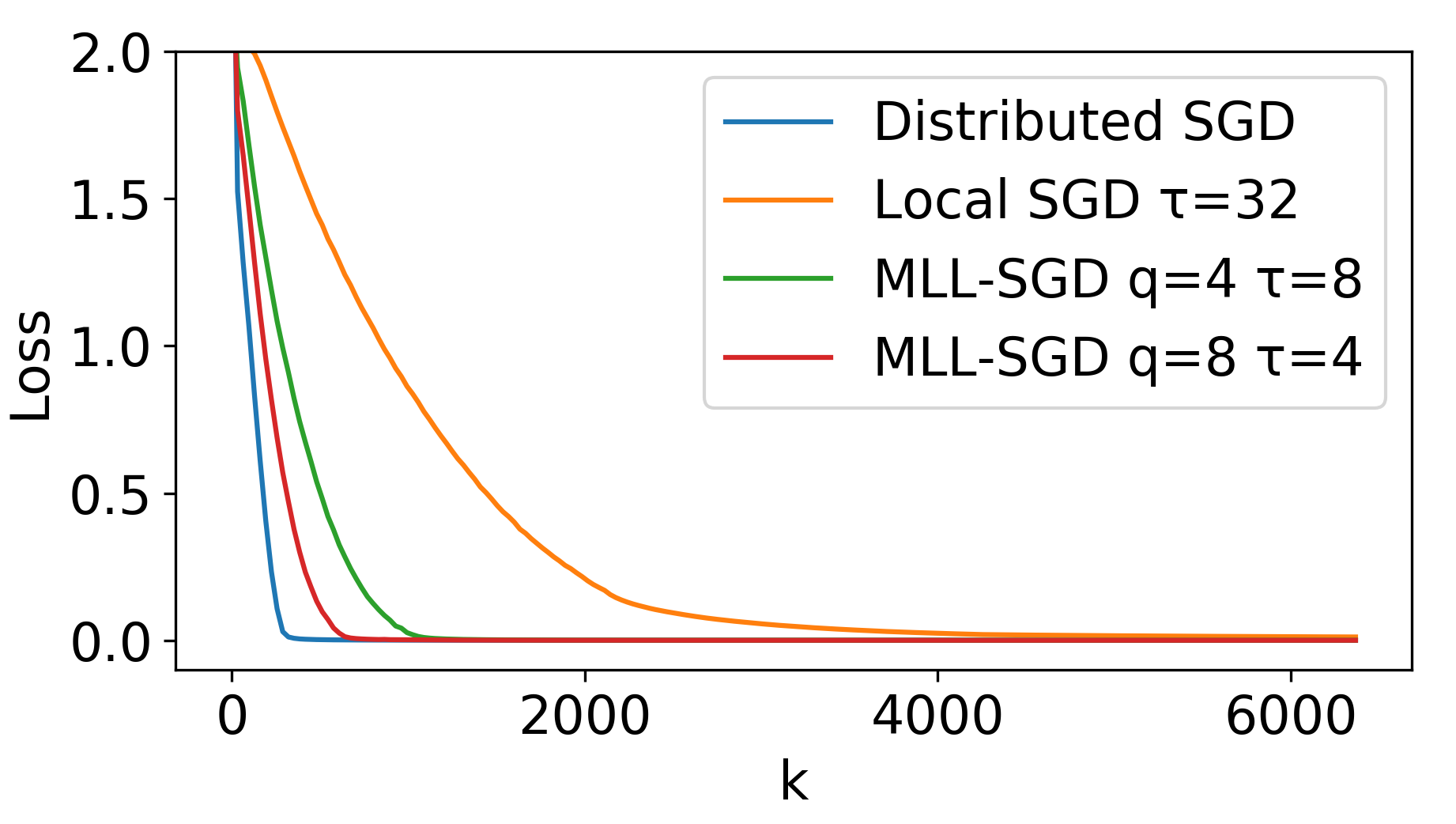}
        \caption{Training loss of ResNet-18 trained on CIFAR-10.}
        \label{l_qtau2.fig}
    \end{subfigure}
    \begin{subfigure}{0.5\textwidth}
        \centering
        \includegraphics[width=0.65\textwidth]{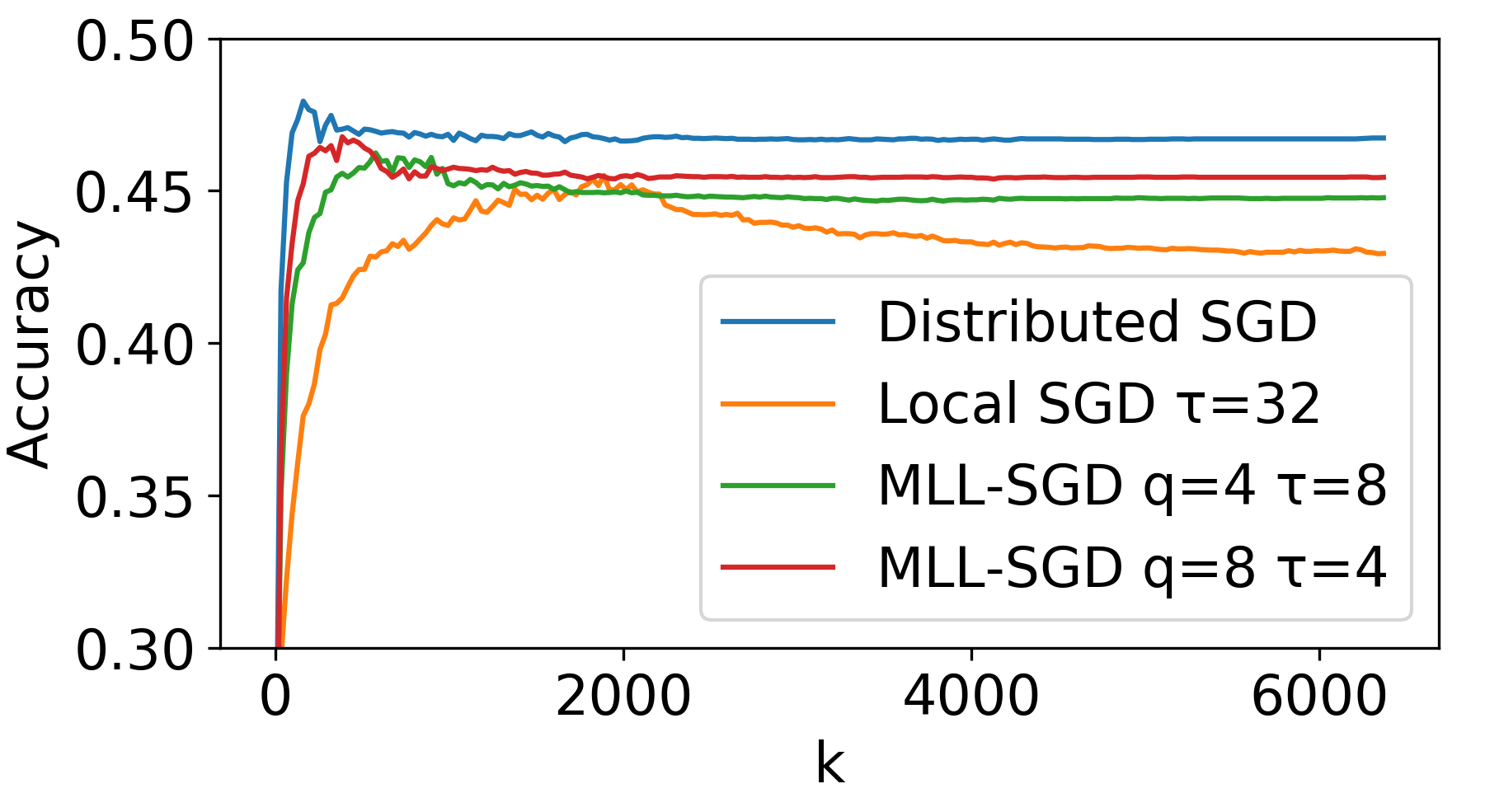}
        \caption{Test accuracy of ResNet-18 trained on CIFAR-10.}
        \label{a_qtau2.fig}
    \end{subfigure}
    \caption{Effect of a hierarchy with different values of $\tau$ and $q$.}
    \label{qtau.fig}
    \vspace{-1em}
\end{figure}

\section{Experiments}\label{exp.sec}
In this section, we show the performance of MLL-SGD compared
    to  algorithms that do not account for hierarchy and heterogeneous worker rates.  
    We also explore the impact of the different algorithm parameters
    that show up in Theorem~\ref{main.thm}.
    
We use the EMNIST~\citep{7966217} and CIFAR-10~\citep{krizhevsky2009learning} 
datasets. For all experiments, we provide results
for training a simple Convolutional Neural Network (CNN)
on EMNIST and training ResNet-18 on CIFAR-10.
The CNN has two convolutional layers and two fully connected layers.
We train the CNN with a step size of $0.01$.
For ResNet, we use a standard approach of changing the step size from
$0.1$ to $0.01$ to $0.001$ over the course of training~\citep{he2016deep}. 
We conduct experiments using Pytorch 1.4.0 and Python 3.

We compare 
MLL-SGD with Distributed SGD, Local SGD, and HL-SGD. 
Distributed SGD is equivalent to MLL-SGD when there is one hub, 
$q = \tau = 1$, and $a_i = 1/N$ and $p_i = 1$ 
for all $i$, which means Distributed SGD averages all 
worker models at every iteration. Thus, we use Distributed 
SGD as a baseline for convergence error and accuracy in some experiments.
Local SGD is equivalent to MLL-SGD when 
$\bm a_i=1/N$ and $\bm p_i=1$ for all $i$,
when the hub network is fully connected, and $q=1$.
HL-SGD extends Local SGD to allow $q>1$.
For all experiments, we let $\tau=32$ for Local SGD.
We let $q\tau=32$ for all HL-SGD and MLL-SGD variations to be
comparable with Local SGD.
In all experiments, we measure training loss and
test accuracy of the averaged model $\bm u_k$ 
every $32$ iterations.

We first explore the effect of different values of $\tau$ and $q$ in MLL-SGD.
We configure 
a multi-level network with a fully connected hub network
and with $10$ hubs, each with $10$ workers.
We use two configurations for MLL-SGD, one with $\tau=8$ and $q=4$,
and one with $\tau=4$ and $q=8$.
Distributed SGD and Local SGD treat the hubs 
as pass-throughs, and average all workers every
iteration and every $\tau$ iterations respectively.
Workers are split into five groups of $20$ workers each.
Each group is assigned a percentage of the full dataset:
$5\%$, $10\%$, $20\%$, $25\%$, and $40\%$.
Workers within a group partition the data evenly.
The workers weights are assigned based on
dataset sizes.
In Figures \ref{l_qtau1.fig} and \ref{l_qtau2.fig} 
we plot the training loss, and in 
Figures \ref{a_qtau1.fig} and \ref{a_qtau2.fig} we plot the test accuracy
for the CNN and ResNet, respectively. 
We observe that as $q$ increases, while keeping $q\tau=32$, MLL-SGD improves
    and approaches the Distributed SGD baseline. 
    Thus, increasing the number of sub-network training 
rounds improves the convergence behavior of MLL-SGD.
The benefit is more pronounced in training ResNet
on CIFAR.

\begin{figure}[t]
    \begin{minipage}{0.5\textwidth}
        \centering
        \captionsetup{width=.9\linewidth,font=footnotesize}
        \includegraphics[width=0.65\textwidth]{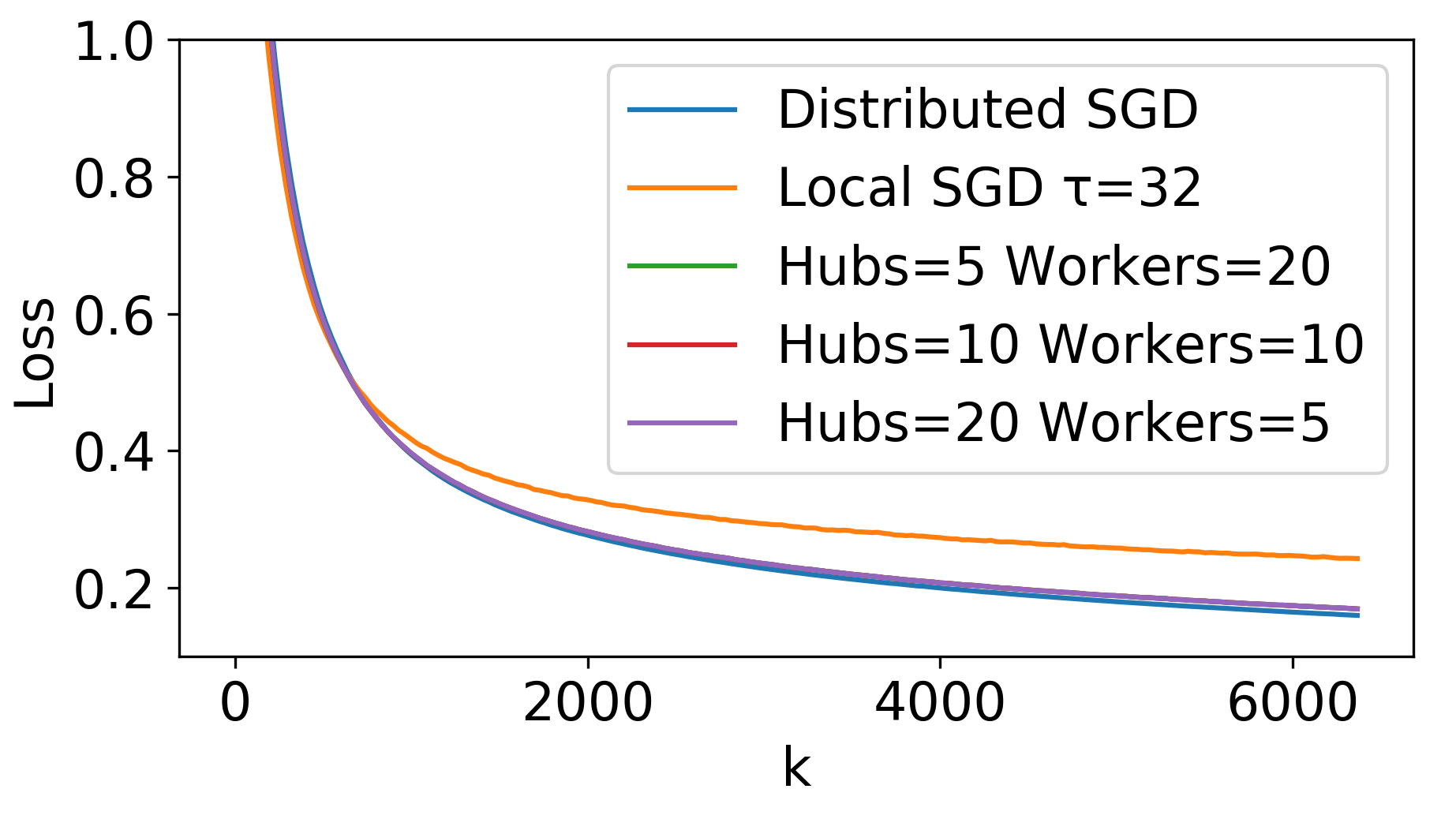}
        \caption{Effect of worker distribution on CNN trained on EMNIST.}
        \label{zetanew1.fig}
    \end{minipage}
    \begin{minipage}{0.5\textwidth}
        \centering
        \captionsetup{width=.9\linewidth,font=footnotesize}
        \includegraphics[width=0.65\textwidth]{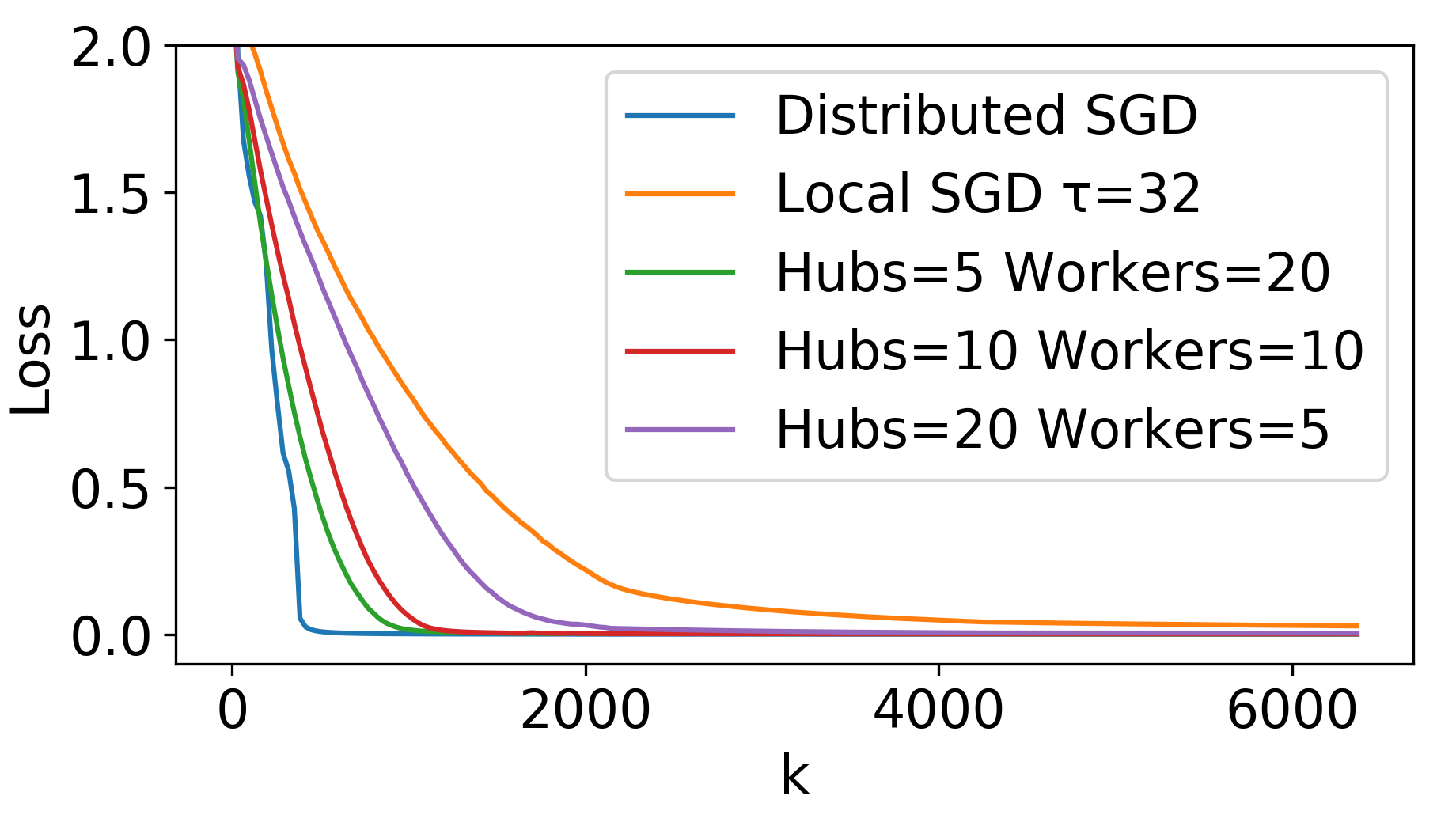}
        \caption{Effect of worker distribution on ResNet-18 trained on CIFAR-10.}
        \label{zetanew2.fig}
    \end{minipage}
    \begin{minipage}{0.5\textwidth}
        \centering
        \captionsetup{width=.9\linewidth,font=footnotesize}
        \includegraphics[width=0.65\textwidth]{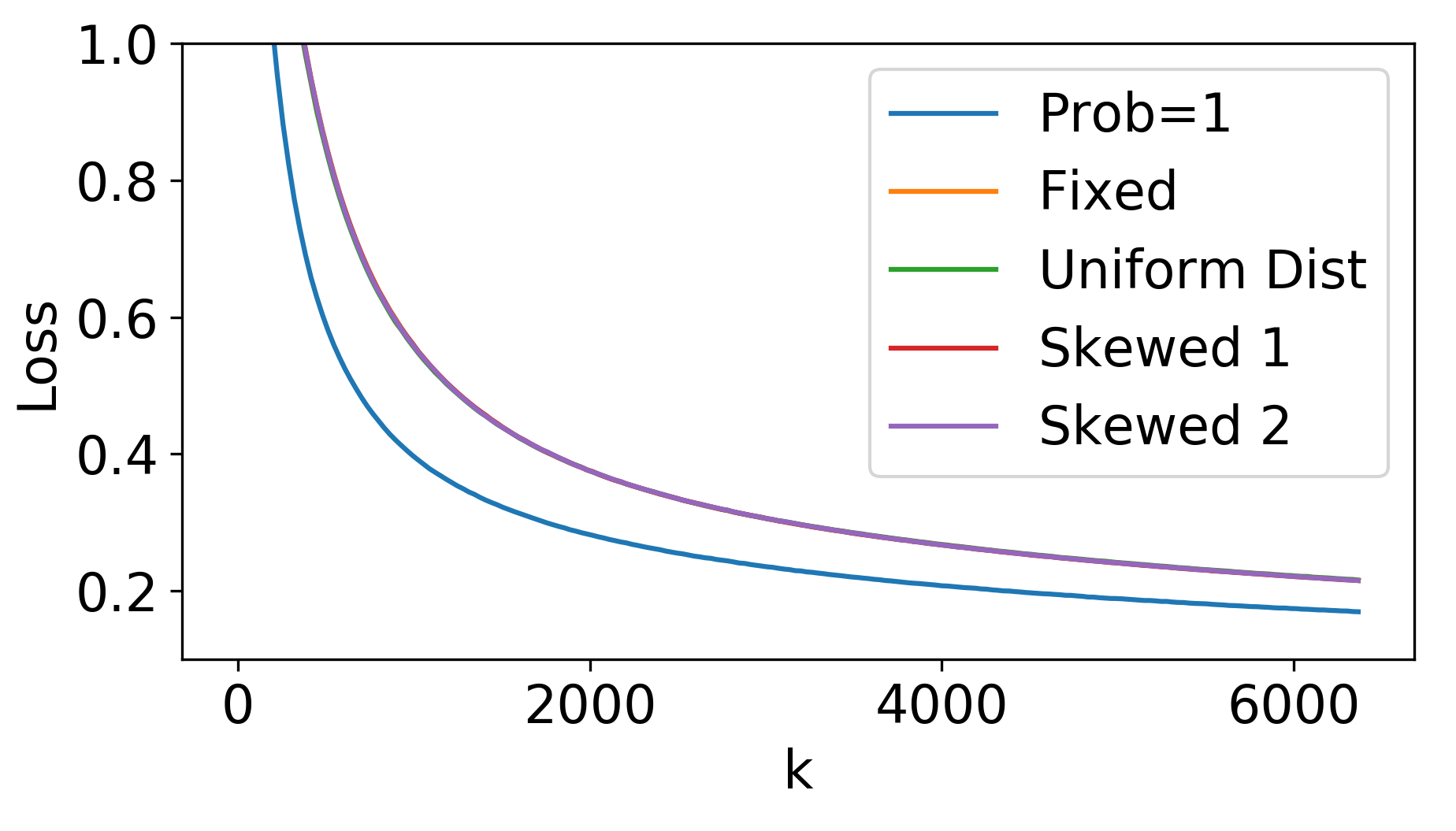}
        \caption{Effect of heterogeneous operating rates on CNN trained on EMNIST.}
        \label{probnew1.fig}
    \end{minipage}
    \begin{minipage}{0.5\textwidth}
        \centering
        \captionsetup{width=.9\linewidth,font=footnotesize}
        \includegraphics[width=0.65\textwidth]{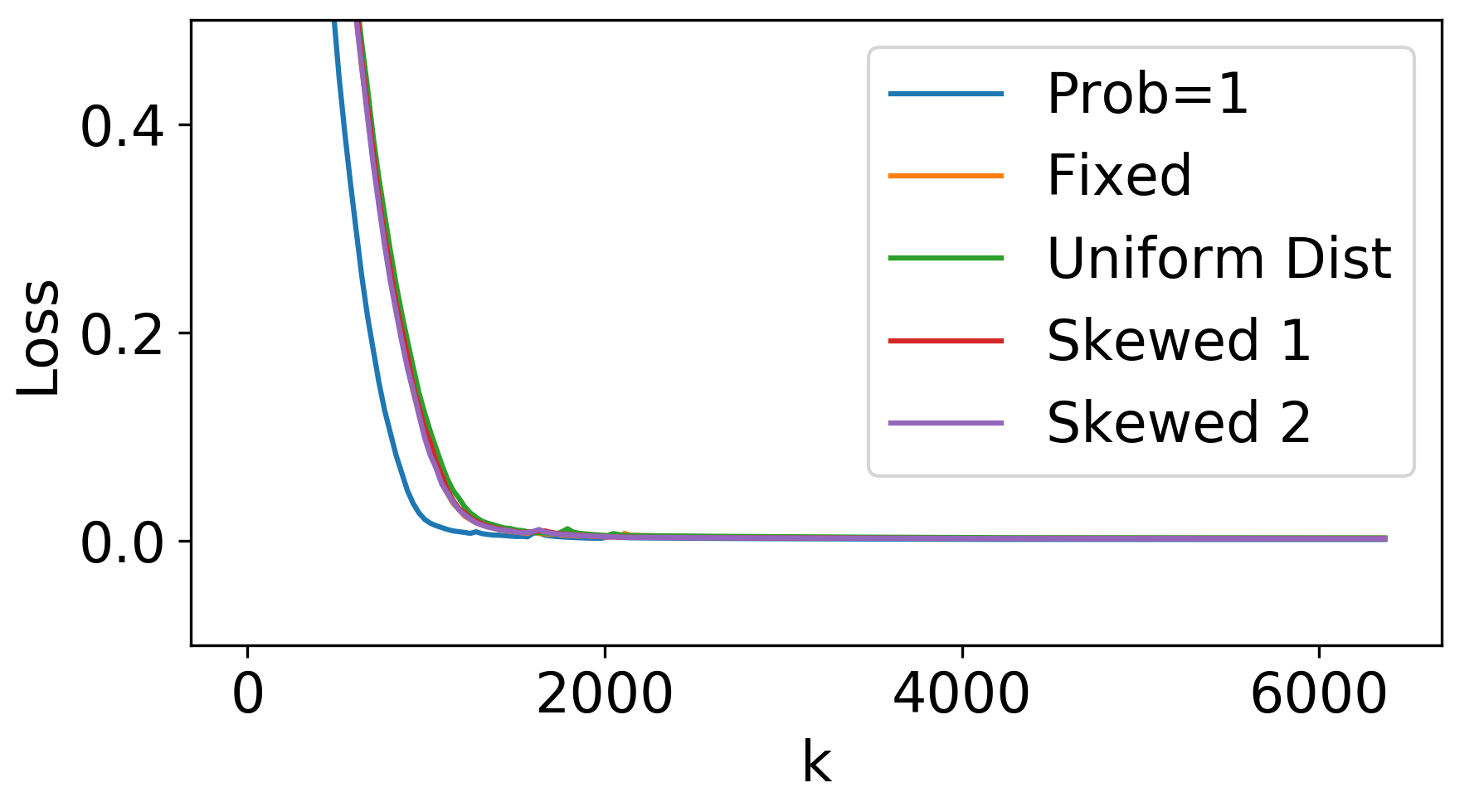}
        \caption{Effect of heterogeneous operating rates on ResNet-18 trained on CIFAR-10.}
        \label{probnew2.fig}
    \end{minipage}
    \vspace{-2em}
\end{figure}

We next investigate how the number and sizes of the sub-networks impacts the convergence of MLL-SGD.
From a pool of $100$
workers, we distribute them across $5$, $10$, and $20$ sub-networks. 
The hub network is a path graph, which yields the largest $\zeta$ 
while keeping the network connected. 
This hub network topology the worst-case scenario in terms
of the convergence bound. Note that as the number of hubs
increases, the larger $\zeta$ becomes.
We let $\bm a_i=1/N$ and $\bm p_i = 1$ for all workers $i$.
We set $q=4$ and $\tau=8$.
We also include results using Local SGD with $1$ hub and $100$ workers.
The results of this experiment are shown in Figures \ref{zetanew1.fig}
and \ref{zetanew2.fig}.
In the case of the CNN, the difference in training loss is minimal 
among the MLL-SGD variations. In the case of ResNet
we can see that as the number of hubs increase, the convergence
rate decreases. 
This is in line with Theorem~\ref{main.thm} since an increased number hubs corresponds with an increased $\zeta$.
Interestingly, despite the low hub network connectivity, MLL-SGD
outperforms Local SGD. 
This shows that MLL-SGD still benefits from a hierarchy even when hub
connectivity is sparse.

Next, we explore the impact of different distributions of 
worker operating rates.
According to Theorem~\ref{main.thm}, the average probability
across workers plays a role in the error bound.
To see if this holds in practice, we compare four
different MLL-SGD setups, all of which includes 
a complete hub network, $10$ hubs, each with $10$ 
workers, $\bm a_i=1/N$, and an average probability amongst workers of $0.55$: 
(i) all workers with a $\bm p_i=0.55$ (Fixed); 
(ii) workers in each sub-network with probability ranging
from $0.1$ to $1$ at steps of $0.1$ 
(Uniform Distribution); 
(iii) $90$ workers with $\bm p_i=0.5$
and $10$ workers with $\bm p_i=1$ (Skewed 1); 
(iv) $90$ workers 
with $\bm p_i=0.6$ and $10$ workers with $\bm p_i=0.1$ (Skewed 2).
We include a case where all workers have $\bm p_i=1$ as a baseline (Prob=1).
In Figures \ref{probnew1.fig} and \ref{probnew2.fig} 
we can see that in all cases except the baseline,
the convergence rate is similar in both models. 
This is in line with our theoretical results, since all cases have the same average worker probability.

\begin{figure}[t]
    \begin{subfigure}{0.5\textwidth}
        \centering
        \includegraphics[width=0.65\textwidth]{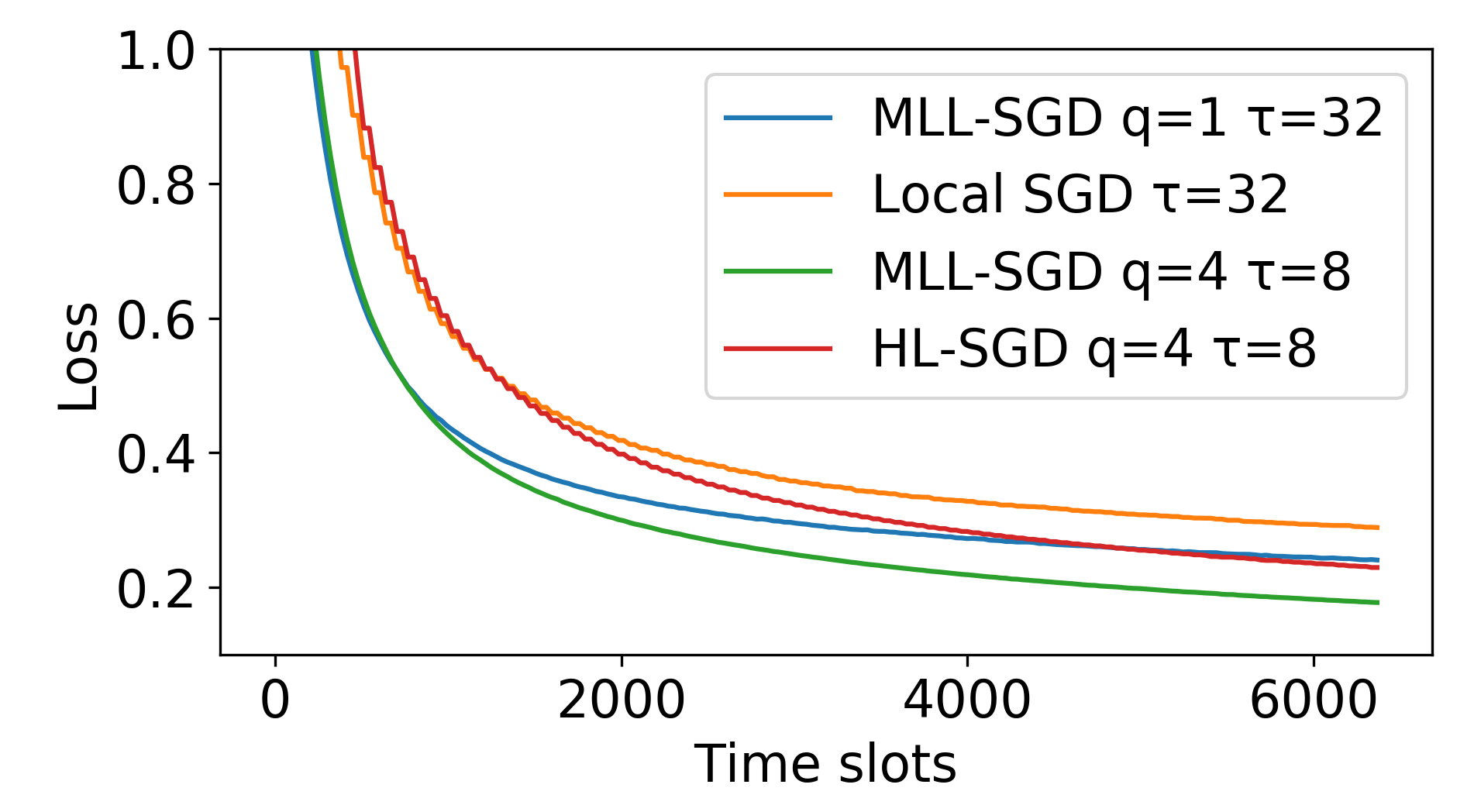}
        \caption{Training loss of CNN with respect to time slots.}
        \label{l_slot1.fig}
    \end{subfigure}
    \begin{subfigure}{0.5\textwidth}
        \centering
        \includegraphics[width=0.65\textwidth]{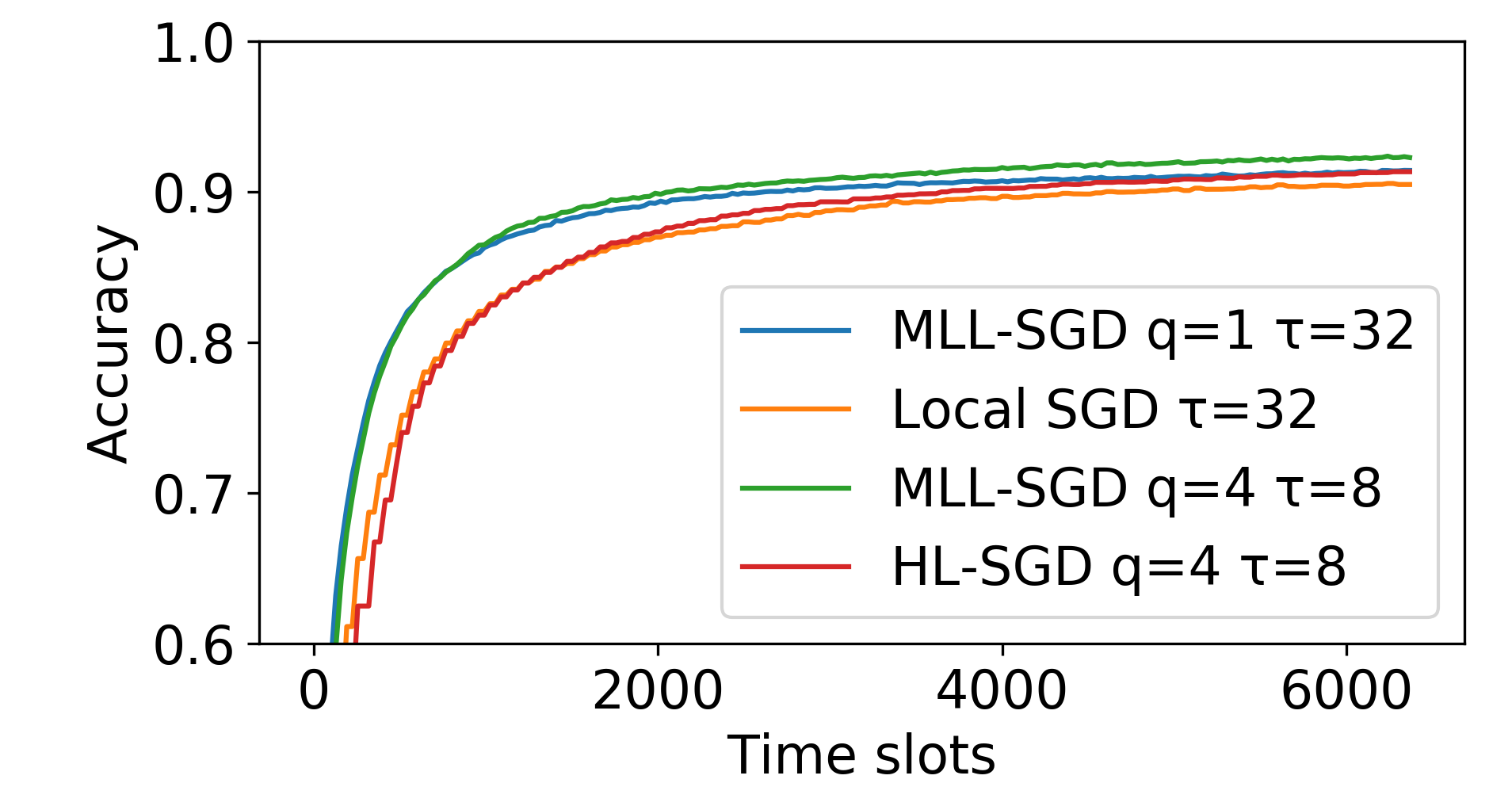}
        \caption{Test accuracy of CNN with respect to time slots.}
        \label{a_slot1.fig}
    \end{subfigure}
    \begin{subfigure}{0.5\textwidth}
        \centering
        \includegraphics[width=0.65\textwidth]{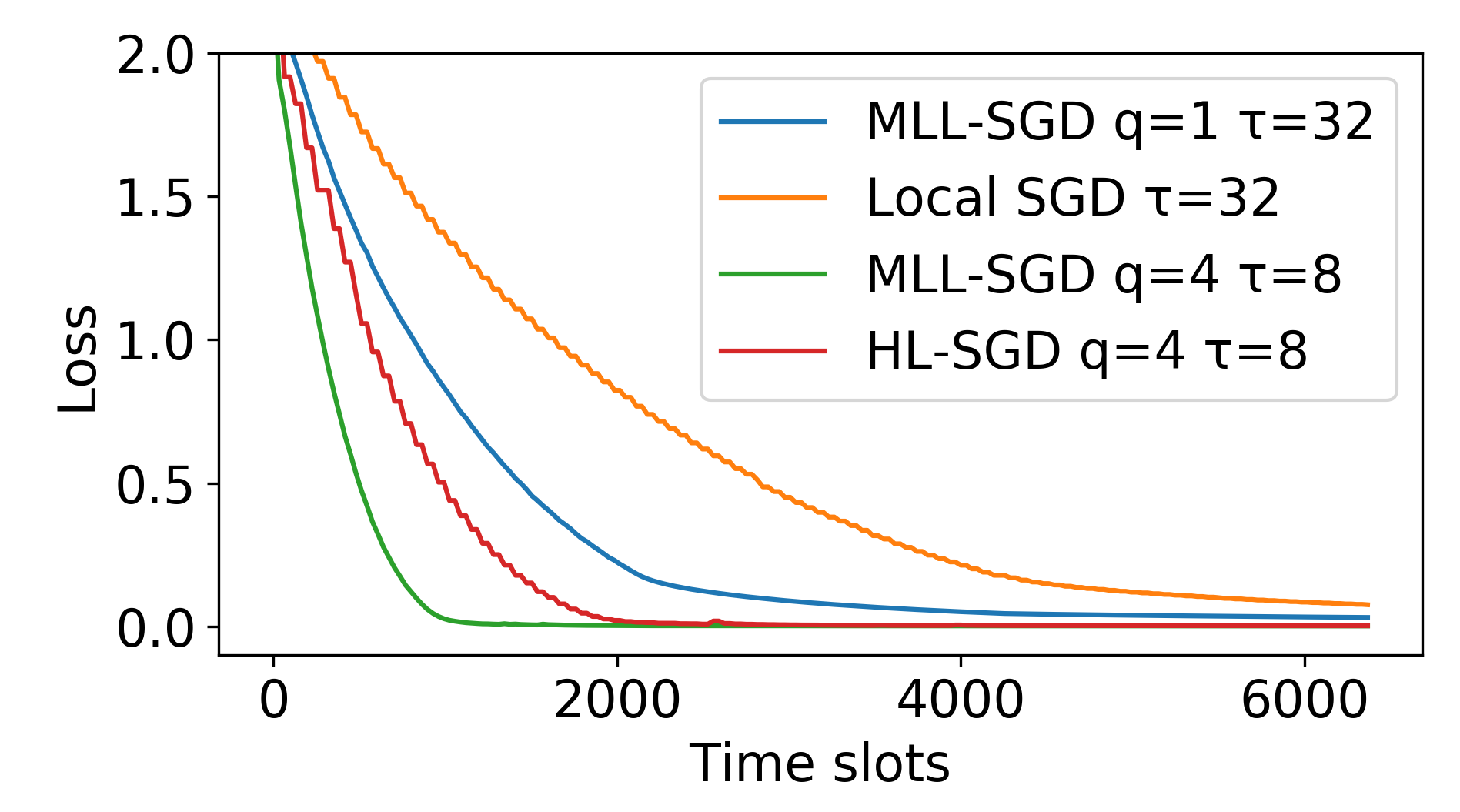}
        \caption{Training loss of ResNet with respect to time slots.}
        \label{l_slot2.fig}
    \end{subfigure}
    \begin{subfigure}{0.5\textwidth}
        \centering
        \includegraphics[width=0.65\textwidth]{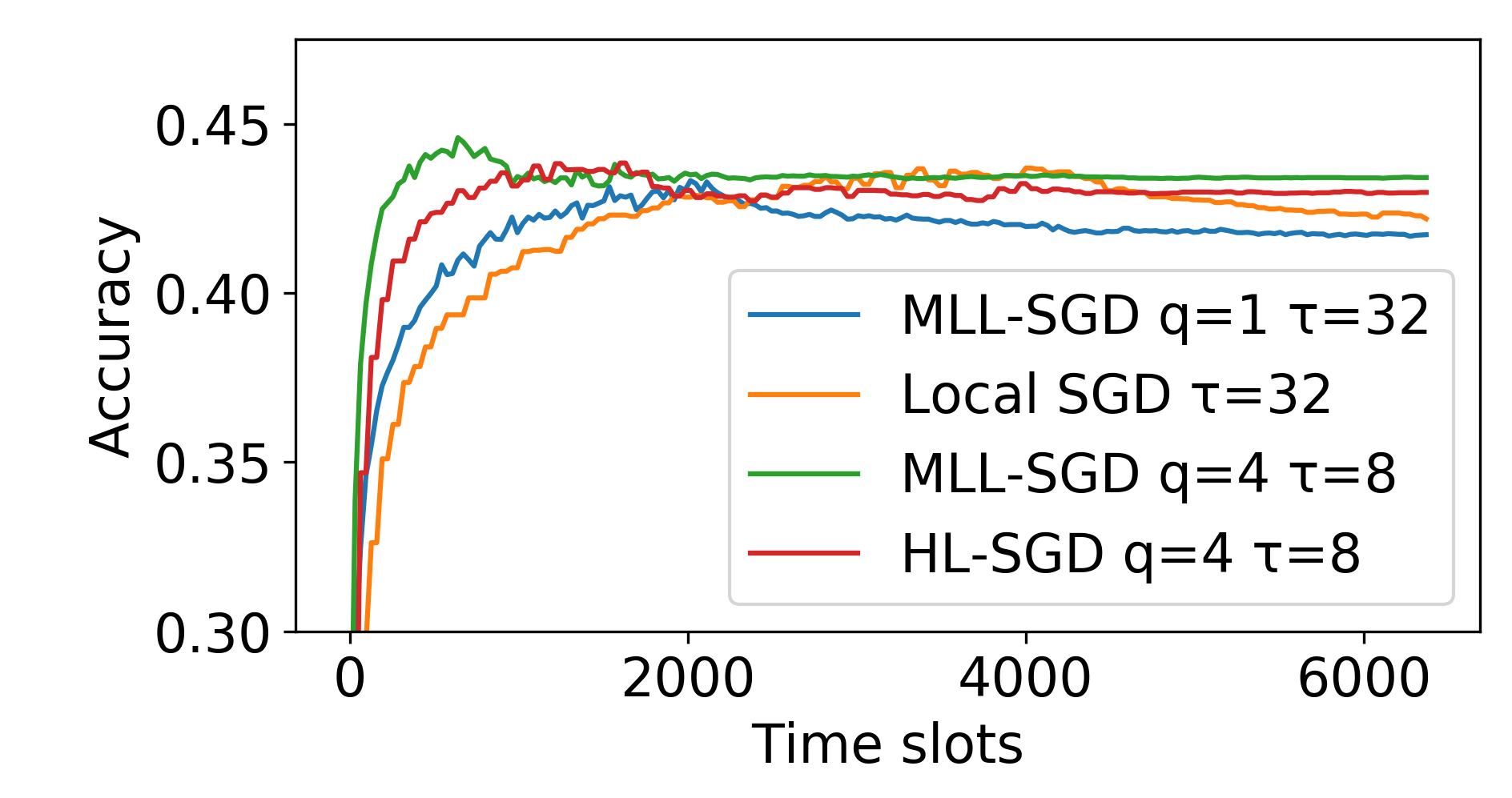}
        \caption{Test accuracy of ResNet with respect to time slots.}
        \label{a_slot2.fig}
    \end{subfigure}
    \caption{Comparing convergence time of Local SGD, HL-SGD, and MLL-SGD.}
    \label{slot.fig}
    \vspace{-1em}
\end{figure}

Finally, we compare the convergence time of MLL-SGD against algorithms
that wait for slower workers: Local SGD and HL-SGD.
We simulate real-time with time slots. 
In every time slot, each worker will take a gradient 
step with a probability $p_i$. 
Note when $p_i = 1$ for a worker $i$, the number of gradient steps 
taken will match the number of time slots $T$. Otherwise, 
the number of gradient steps taken will be $T \cdot p_i$ in expectation. 
MLL-SGD will wait $\tau$ 
time slots before averaging worker models in a sub-network, regardless of the number of gradient 
steps taken, while Local SGD and HL-SGD will wait for all workers 
to take $\tau$ gradient steps. This approach allows us to compare the 
progress of each algorithm over time.
In this experiment, we set $p_i=0.9$ for $90\%$ of workers and $p_i=0.6$
for $10\%$ of the workers. 
As in the previous experiments, we use
a multi-level network with a fully connected hub network
and with $10$ hubs, each with $10$ workers.
We study MLL-SGD
 with two parameter settings, $\tau=32$, $q=1$ and $\tau=8$, $q=4$.
We also include results for Local-SGD and HL-SGD.
By comparing MLL-SGD with $\tau=32$, $q=1$ with Local-SGD, we can evaluate the impact 
of using a local training period based on time rather than a number of worker iterations. 
By comparing MLL-SGD with $\tau=8$, $q=4$ with HL-SGD, 
we can evaluate this impact in a multi-level network.

In Figures \ref{l_slot1.fig} and \ref{l_slot2.fig}, we plot the training loss, and in 
Figures \ref{a_slot1.fig} and \ref{a_slot2.fig},
we plot the test accuracy for the CNN and ResNet, respectively. 
We can see that MLL-SGD with $q=1$ converges more quickly, in both loss and accuracy, than Local SGD, and
that MLL-SGD with $q=4$ converges more quickly than HL-SGD. These trends hold in both the CNN and ResNet models.
The results show that in this experimental setup, waiting for slow workers 
is detrimental to the overall convergence time.

\section{Conclusion}\label{conclusion.sec}
We have introduced MLL-SGD, a variation of Distributed SGD in a 
multi-level network model.
Our algorithm incorporates the heterogeneity of worker devices using a stochastic approach.
We provide theoretical analysis of the algorithm's convergence, and we show how the convergence error
depends on the average worker rate, the hub network topology, and the number of local, sub-network averaging, and hub averaging steps.
Finally, we provide experimental results that illustrate the effectiveness of
MLL-SGD over Local SGD and HL-SGD. 
In future work, we plan to analyze the effects of non-IID data on convergence error.

\subsubsection*{Acknowledgments}
This work is supported by the Rensselaer-IBM AI Research Collaboration (http://airc.rpi.edu), part of the IBM AI Horizons Network (http://ibm.biz/AIHorizons), and by the National Science Foundation under grants CNS 1553340 and CNS 1816307.

\bibliography{references}
\bibliographystyle{iclr2021_conference}

\appendix

\section{Code Repository}
The code used in our experiments can be found at: https://github.com/rpi-nsl/MLL-SGD.
This code simulates a multi-level network with heterogeneous
workers, and trains a model using MLL-SGD. 

\section{Additional Experiments}
\begin{figure}[t]
    \begin{subfigure}{0.5\textwidth}
        \centering
        \captionsetup{width=.9\linewidth,font=footnotesize}
        \includegraphics[width=0.75\textwidth]{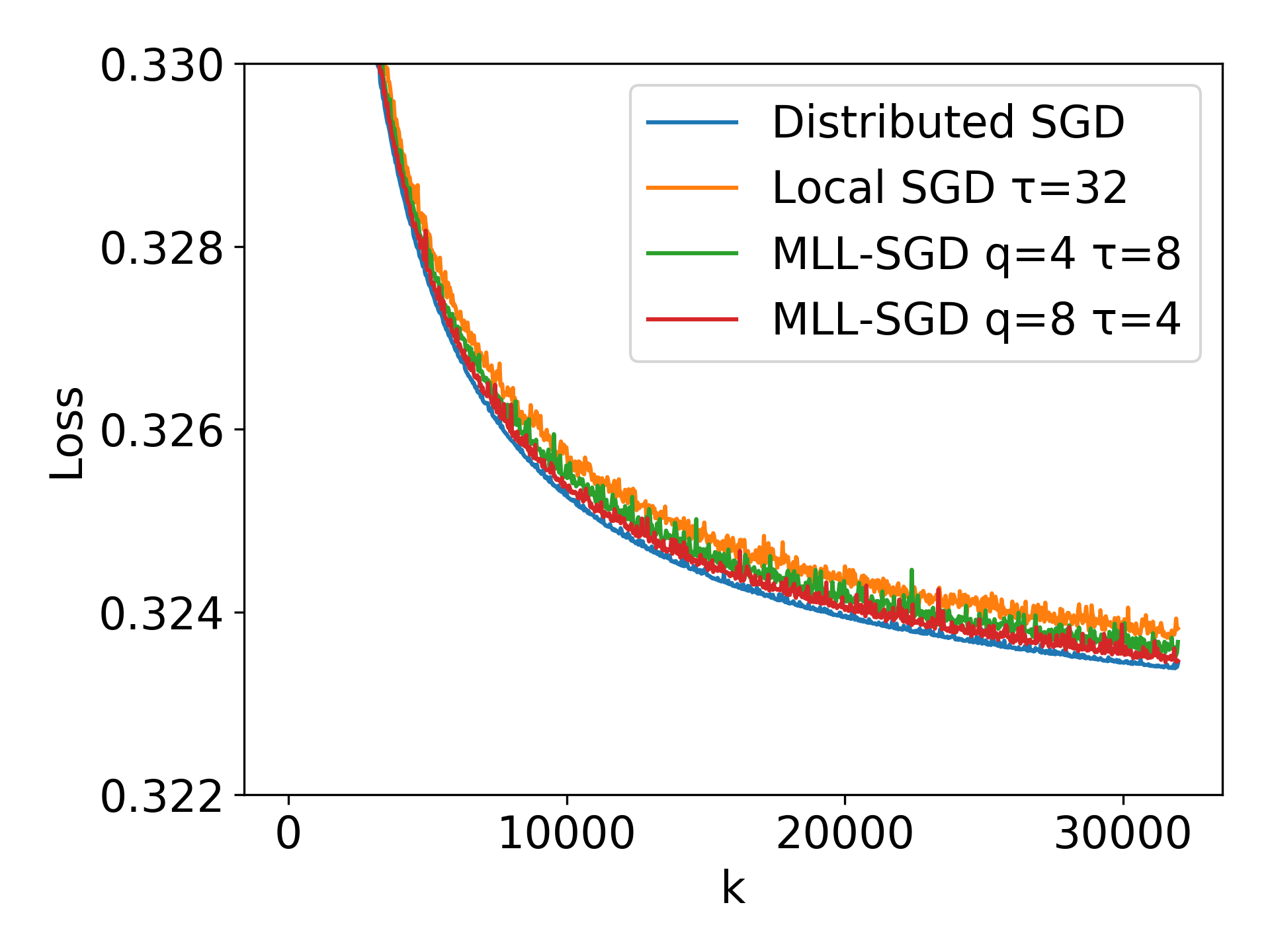}
        \caption{Training loss of logistic regression trained on MNIST.}
        \label{l_qtau3.fig}
    \end{subfigure}
    \begin{subfigure}{0.5\textwidth}
        \centering
        \captionsetup{width=.9\linewidth,font=footnotesize}
        \includegraphics[width=0.75\textwidth]{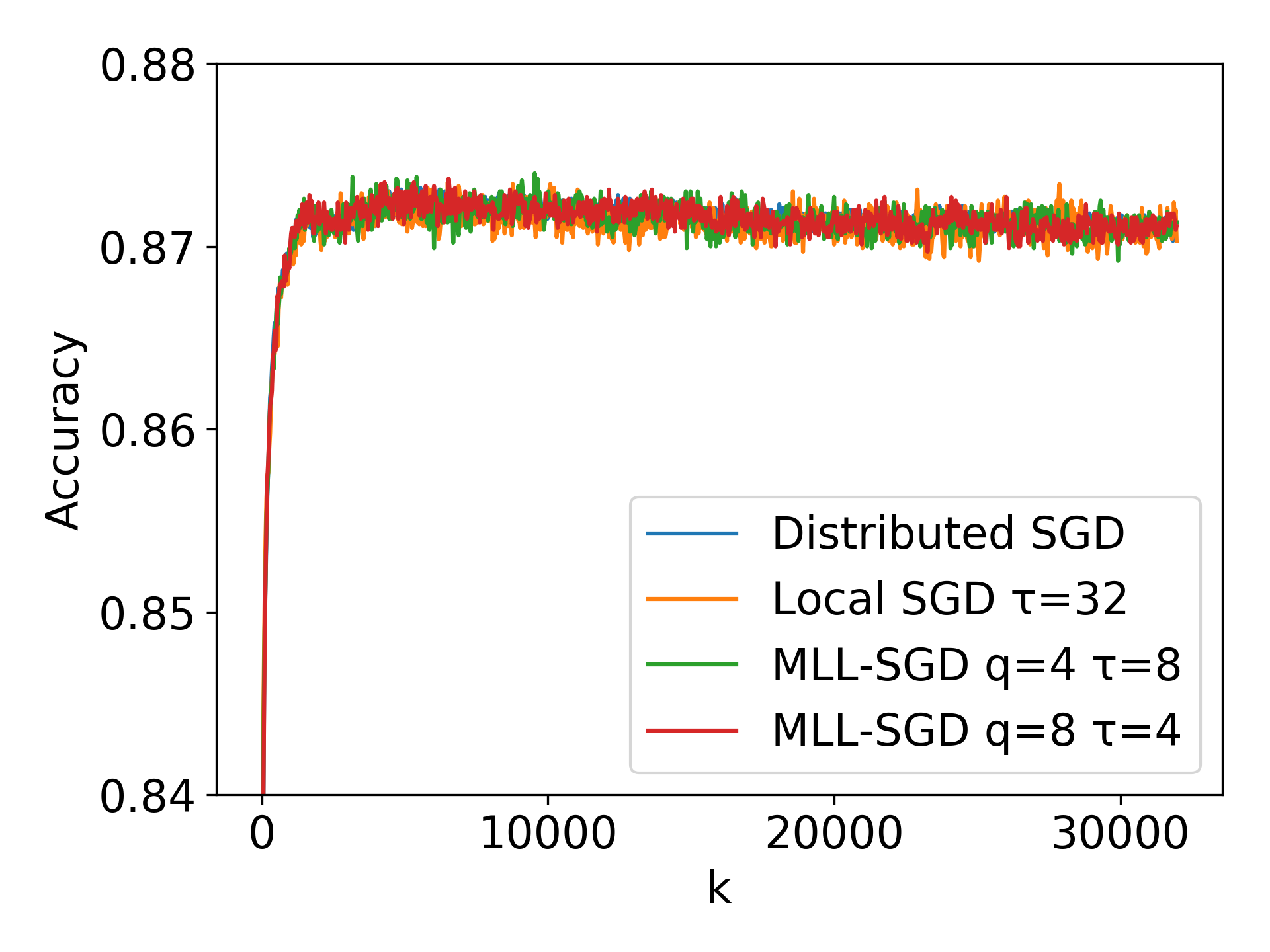}
        \caption{Test accuracy of logistic regression trained on MNIST.}
        \label{a_qtau3.fig}
    \end{subfigure}
    \caption{Effect of a hierarchy with different values of $\tau$ and $q$.}
\end{figure}

\begin{figure}[t]
    \begin{minipage}{0.5\textwidth}
        \centering
        \captionsetup{width=.9\linewidth,font=footnotesize}
        \includegraphics[width=0.75\textwidth]{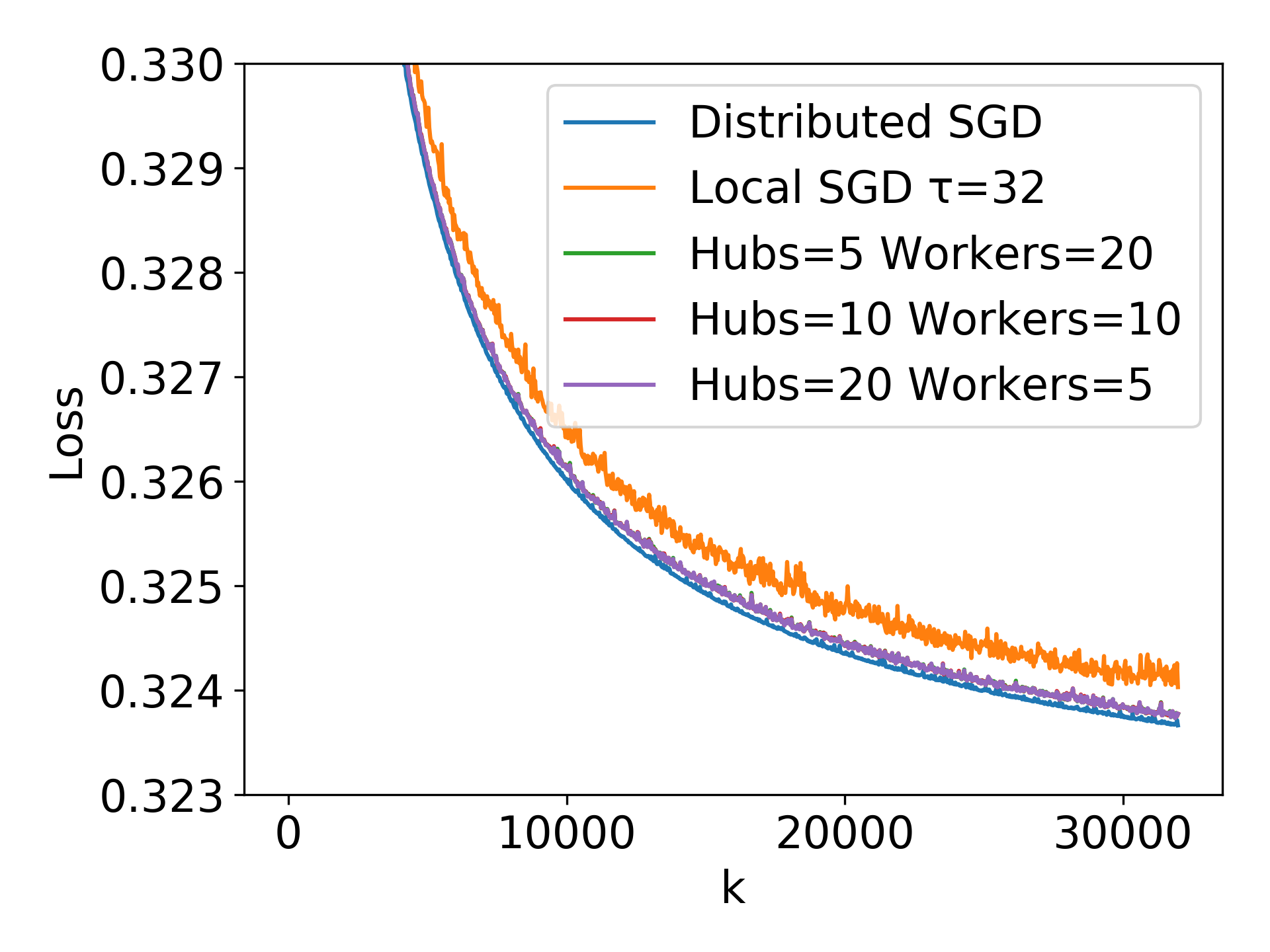}
        \caption{Effect of worker distribution on logistic regression trained on MNIST.}
        \label{zetanew3.fig}
    \end{minipage}
    \begin{minipage}{0.5\textwidth}
        \centering
        \captionsetup{width=.9\linewidth,font=footnotesize}
        \includegraphics[width=0.75\textwidth]{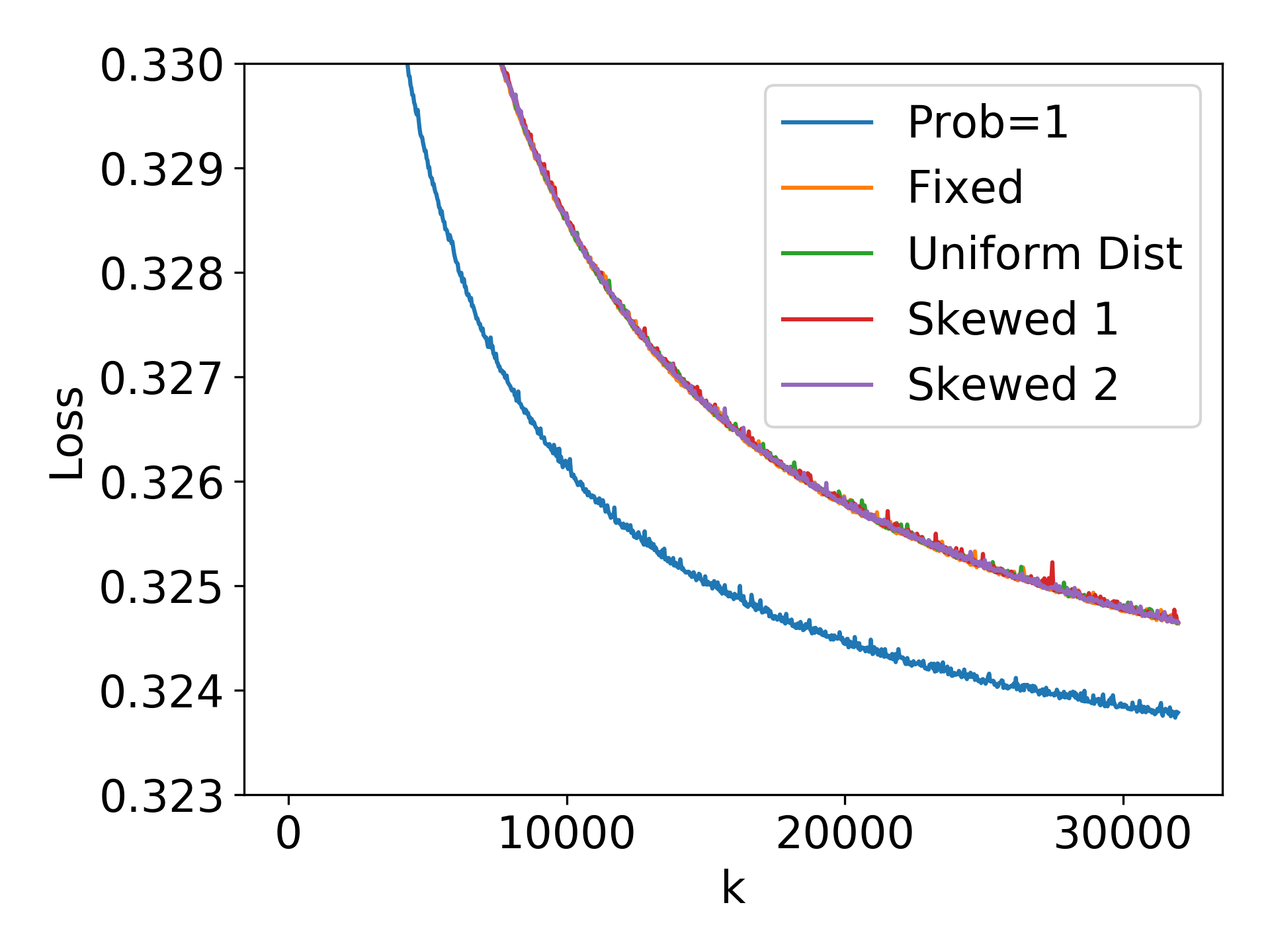}
        \caption{Effect of heterogeneous operating rates on logistic regression trained on MNIST.}
        \label{probnew3.fig}
    \end{minipage}
\end{figure}

Our experiments in Section~\ref{exp.sec} explore how changing 
MLL-SGD parameters affect training on a non-convex function. 
In this section, show the results of the same experiments on a convex loss function.
We train a logistic 
regression model on the MNIST 
dataset~\citep{bottou1994comparison}.
We train a binary classification model with half the classes being $0$ and
the other half being $1$ and use a step size of $0.2$.
We run all experiments for $32$,$000$ iterations.

We rerun our first experiment from Figure~\ref{qtau.fig} with
logistic regression trained on MNIST.
Figures \ref{l_qtau3.fig} and \ref{a_qtau3.fig} show the training
loss and test accuracy, respectively. 
As with the non-convex functions, we can see that
MLL-SGD with larger $q$ approaches the Distributed SGD baseline.

We rerun our second experiment comparing different
hub and worker distributions with
logistic regression trained on MNIST.
Figure~\ref{zetanew3.fig} shows the training
loss. 
The three variations of MLL-SGD do not show much
difference in terms of convergence rate, indicating
that $\zeta$ has little effect in this case. However,
they still outperform Local SGD due to $q$ being
larger.

We rerun our third experiment comparing different
worker operating rates distributions with
logistic regression trained on MNIST.
Figure~\ref{probnew3.fig} shows the training
loss. As with the non-convex functions, all 
MLL-SGD variations with the same average probability
have similar convergence rate.

\begin{figure}[t]
    \begin{subfigure}{0.5\textwidth}
        \centering
        \captionsetup{width=.9\linewidth,font=footnotesize}
        \includegraphics[width=0.75\textwidth]{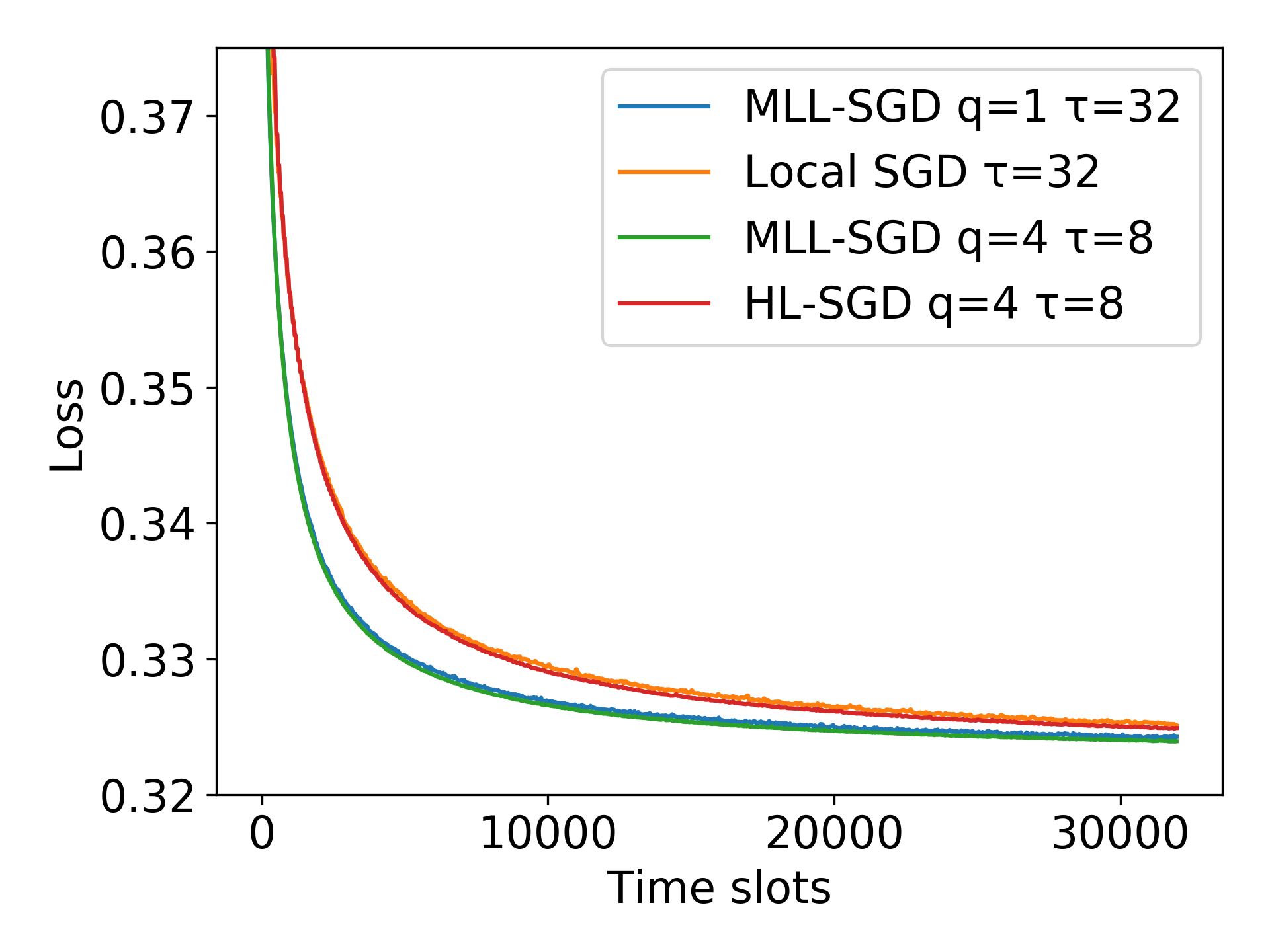}
        \caption{Training loss of logistic regression with respect to time slots.}
        \label{l_slot3.fig}
    \end{subfigure}
    \begin{subfigure}{0.5\textwidth}
        \centering
        \captionsetup{width=.9\linewidth,font=footnotesize}
        \includegraphics[width=0.75\textwidth]{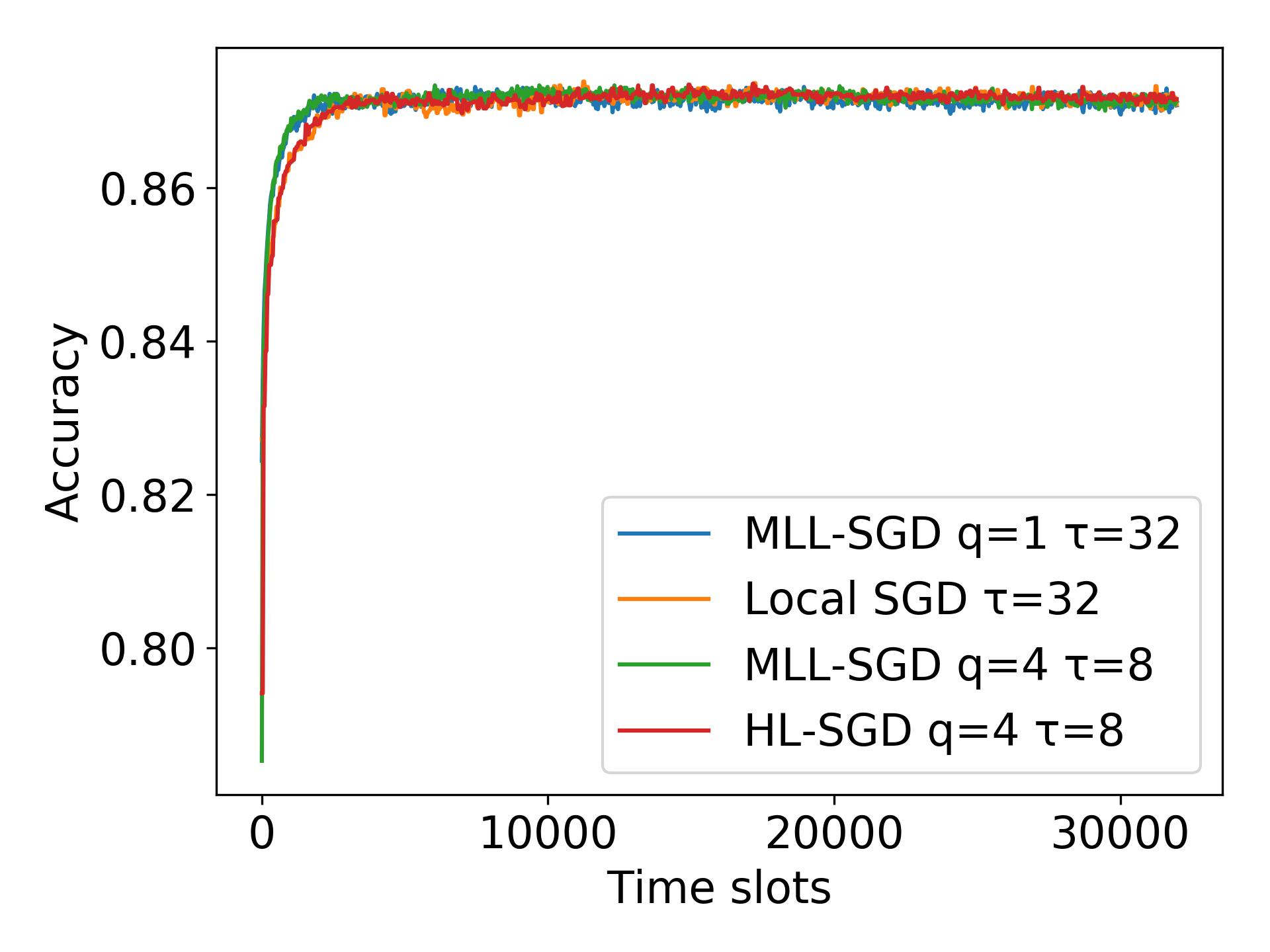}
        \caption{Test accuracy of logistic regression with respect to time slots.}
        \label{a_slot3.fig}
    \end{subfigure}
    \caption{Comparing convergence time of Local SGD and MLL-SGD.}
\end{figure}

We rerun our first experiment from Figure~\ref{slot.fig} with
logistic regression trained on MNIST.
Figures \ref{l_slot3.fig} and \ref{a_slot3.fig} show the training
loss and test accuracy, respectively. 
We can see an improvement in convergence rate of MLL-SGD
over both Local SGD and HL-SGD.

\section{Proof of Theorem~\ref{main.thm}} 

For our proof we adopt a similar approach to that 
in~\citet{wang2018cooperative}. This section is structured as follows.
We first define some notation and make some observations in Section~\ref{pre.sec}. 
Our supporting lemmas are stated in Section~\ref{lemma.sec}. We close with the full
proof of Theorem~\ref{main.thm} in Section~\ref{proof.sec}.

\subsection{Preliminaries} \label{pre.sec}
For simplicity of notation, we let 
$\Vert \cdot \Vert$ denote the $l_2$ vector norm.
Let the weighted Frobenius norm of an $N \times M$ matrix $\bm X$ with 
an $N$-vector $\bm a$ be defined as follows: 
\begin{align}
    \Vert \bm X \Vert^2_{F_{\bm a}} = \left|\Tr(\left(\diag(\bm a)\right)^{1/2} \bm X 
        \bm X^T \left(\diag(\bm a)\right)^{1/2})\right|
        = \sum_{i=1}^N \sum_{j=1}^M \bm a_i |\bm x_{i,j}|^2 .
\end{align}
The matrix operator norm for a square matrix $\Q$ is defined as:
\begin{align}
    \left\Vert \Q \right\Vert_{op} 
    &= \sqrt{\lambda_{max}(\Q^T \Q)} .
\end{align}

We define the set of Bernoulli random variables $\Theta = \{\theta^1_k, \ldots, \theta^N_k\}$,
where 
\begin{align*}
\theta^i_k =
\begin{cases} 
1 & \text{with probability}~\bm p_i \\
0 & \text{with probability}~(1- \bm p_i) . 
\end{cases}
\end{align*}
Let $\Xi_k = \{\xi_k^{(1)},...,\xi_k^{(N)}\}$ be the set of mini-batches 
used by the $N$ workers at time step $k$. Without loss of generality, we assign a mini-batch to each worker, even if it does not execute a gradient step in that iteration.
An equivalent definition of $\bm g_k^{(i)}$ is then 
\begin{align}
    \bm g_k^{(i)} = \theta_k^i g (\xi_k^{(i)}).
\end{align}
 
For simplicity of notation, let $\mathbb{E}_k$ be equivalent to
$\mathbb{E}_{\Theta_k, \Xi_k|X_k}$.


We note that Assumption~\ref{unbias.asp} implies:
\begin{align}
    \mathbb{E}_k[\bm g_k^{(i)}] &= \bm p_{i}\mathbb{E}_k[g(\bm x_k^{(i)})]\\
    &= \bm p_{i}\nabla F(\bm x_k^{(i)}). \label{gtoF.eq}
\end{align}
Further, when $i \neq j$:
\begin{align}
    \mathbb{E}_k[(\bm g_k^{(i)})^T\bm g_k^{(j)}] &= 
    \bm p_{i}\bm p_{j}\mathbb{E}_k[(g(\bm x_k^{(i)}))]^T\mathbb{E}_k[g(\bm x_k^{(j)})]\\
    &=  \bm p_{i}\bm p_{j} \nabla F(\bm x_k^{(i)})^T\nabla F(\bm x_k^{(j)}).
\end{align}
We also note that Assumption~\ref{variance.asp} implies:
\begin{align}
    \mathbb{E}_k\left\Vert \bm g_k^{(i)} - \nabla F(\bm x_k^{(i)}) \right\Vert^2
    &= \mathbb{E}_k \left[ \left\Vert \bm g_k^{(i)} \right\Vert^2 
    + \left\Vert \nabla F(\bm x_k^{(i)}) \right\Vert^2
    - 2 (\bm g_k^{(i)})^T \nabla F(\bm x_k^{(i)})  \right]  \\
    &= \mathbb{E}_k \left\Vert \bm g_k^{(i)} \right\Vert^2 
    + \left\Vert \nabla F(\bm x_k^{(i)}) \right\Vert^2
    - 2\mathbb{E}_k (\bm g_k^{(i)})^T \nabla F(\bm x_k^{(i)})  \\
    &= \bm p_{i} \mathbb{E}_k\left\Vert g(\bm x_k^{(i)}) \right\Vert^2 
    + \left\Vert \nabla F(\bm x_k^{(i)}) \right\Vert^2
    - 2\bm p_{i}\mathbb{E}_k g(\bm x_k^{(i)})^T \nabla F(\bm x_k^{(i)})  \\
    &= \bm p_{i} \mathbb{E}_k\left\Vert g(\bm x_k^{(i)}) \right\Vert^2 
    + \bm p_{i}\left\Vert \nabla F(\bm x_k^{(i)}) \right\Vert^2 \nonumber \\
    &~~~~~~~~~~~~~~~~~~~~- 
    2\bm p_{i}\mathbb{E}_k g(\bm x_k^{(i)})^T \nabla F(\bm x_k^{(i)})  
    + (1-\bm p_{i})\left\Vert \nabla F(\bm x_k^{(i)}) \right\Vert^2 \\
    &= \bm p_{i}\mathbb{E}_k\left\Vert g(\bm x_k^{(i)}) - \nabla F(\bm x_k^{(i)}) \right\Vert^2
    + (1-\bm p_{i})\left\Vert \nabla F(\bm x_k^{(i)}) \right\Vert^2 \\
    &\leq  \bm p_{i}\beta\left\Vert \nabla F(\bm x_k^{(i)}) \right\Vert^2
    + \bm p_{i}\sigma^2 + (1-\bm p_{i})\left\Vert \nabla F(\bm x_k^{(i)}) \right\Vert^2 \\
    &= \left(\bm p_{i}(\beta-1)+1\right)
    \left\Vert \nabla F(\bm x_k^{(i)}) \right\Vert^2 + \bm p_{i}\sigma^2. \label{asp5implication.eq}
\end{align}

Finally, we define the weighted average stochastic gradient and the weighted average batch gradient as:
\[
\Gk =  \sum_{i=1}^N \bm a_i \bm g_k^{(i)},~~\Hk =  \sum_{i=1}^N \bm a_i\nabla F(\bm x_k^{(i)}).
\]

\subsection{Lemmas and Propositions} \label{lemma.sec}

Next, we state our supporting lemmas and propositions.
\begin{restatable}{prop}{Za} \label{Za.prop}
The matrices $\Z$ and $\V$ satisfy the following properties:
\begin{enumerate}
\item $\Z$ and $\V$ each have a right eigenvector of $\bm a$ with eigenvalue 1.
\item $\Z$ and $\V$ each have a left eigenvector of $\one_{N}^T$ with eigenvalue 1.
\item All other eigenvalues of $\Z$ and $\V$ have magnitude strictly less than 1.
\end{enumerate}
\end{restatable}

\begin{proof}
Assumption~\ref{H.assum} indicates that $\Hm$ is a \emph{Generalized
Diffusion Matrix} as defined in~\citet{rotaru2004dynamic}. 

Recall
Assumption~\ref{H.assum}:

\textbf{Assumption 2. } \textit{The matrix $\Hm$ satisfies the following:
    \begin{enumerate}[label=2\alph*]
        \item If $(i,j) \in E$, then $\Hm_{i,j} > 0$. Otherwise, $\Hm_{i,j}  = 0$.
        \item $\Hm$ is column stochastic, i.e.,  $\sum_{i=1}^D \Hm_{i,j} = 1$. 
        \item For all $i,j \in \D$, we have $\Hm_{i,j} \bm b_j  = \Hm_{j,i} \bm b_i$. 
\end{enumerate}}

If we show this implies
that $\Z$ and $\V$ are Generalized Diffusion Matrices with the same properties
to those in Assumption~\ref{H.assum} with vector $\bm a$, 
then the properties in the proposition are satisfied.

Since $\Hm$ and $\bm b$ are non-negative, then $\Z$ is also non-negative.
It is also clear that $\Z$ is column stochastic by construction.
It is left to prove that:
\begin{align}
\Z_{i,j} \bm a_j = \Z_{j,i} \bm a_i.
\end{align}

Applying the definition of $\Z$ to the left side, we have: 
\begin{align}
    \Z_{i,j} \bm a_j &= \Hm_{d(i),d(j)} v^{(i)} \bm a_j 
\end{align}
Since we know that $\Hm$ is a Generalized Diffusion Matrix 
with vector $\bm b$, we know that:
\begin{align}
    \Hm_{i,j} \bm b_j &= \Hm_{j,i} \bm b_i \\
    \Hm_{i,j} &= \Hm_{j,i} \frac{\bm b_i}{\bm b_j} .
\end{align}
Plugging this in for $\Hm_{d(i),d(j)}$, we have: 
\begin{align}
    \Z_{i,j} \bm a_j &= \Hm_{d(j),d(i)} \frac{\bm b_{d(i)}}{\bm b_{d(j)}} v^{(i)} \bm a_j \\
     &= \Hm_{d(j),d(i)} \frac{\sum_{r \in \M^{(d(i))}} w^{(r)}}{w_{tot}} 
        \frac{w_{tot}}{\sum_{r \in \M^{(d(j))}} w^{(r)}} 
        \frac{w^{(i)}}{\sum_{r \in \M^{(d(i))}} w^{(r)}} 
        \frac{w^{(j)}}{w_{tot}} \\
     &= \Hm_{d(j),d(i)} \frac{w^{(i)}}{w_{tot}}
        \frac{w^{(j)}}{\sum_{r \in \M^{(d(j))}} w^{(r)}} \\
     &= \Hm_{d(j),d(i)} v^{(j)} \bm a_i \\
     &= \Z_{j,i} \bm a_i .
\end{align}
Therefore, $\Z$ is a Generalized Diffusion Matrix.

We can show that $\V$ is also a Generalized Diffusion Matrix with
the vector $\bm a$. $\V$ is constructed to be 
non-negative and column stochastic.
It is left to prove that
\begin{align}
\V_{i,j} \bm a_j = \V_{j,i} \bm a_i.
\end{align}

When $i,j$ are outside a block $\V^{(d)}$, then $\V_{i,j} =
\V_{j,i} = 0$, so the equation is trivially satisfied.
When within a block, in terms 
of $w$, we have:
\begin{align}
    \V_{i,j} \bm a_j &= \V_{j,i} \bm a_i \\
    \frac{w^{(i)}}{\sum_{r \in \M^{(d(i))}} w^{(r)}} \frac{w^{(j)}}{w_{tot}}
    &= 
    \frac{w^{(j)}}{\sum_{r \in \M^{(d(j))}} w^{(r)}} \frac{w^{(i)}}{w_{tot}}.
\end{align}
Noting that we are within a block, therefore $d(i) = d(j)$, we can see
that both sides are equal:
\begin{align}
    w^{(i)} w^{(j)}
    &=
    w^{(j)} w^{(i)}.
\end{align}
Therefore, $\V$ is a Generalized Diffusion Matrix.

\end{proof}

\begin{restatable}{prop}{Z2} \label{Z.prop}
    Given a diffusion matrix $\Hm$ with the properties
in Assumption~\ref{H.assum}, if $\Z$ constructed as follows, 
    \begin{align}
        \Z_{i,j} = \Hm_{d(i),d(j)} v^{(i)}
    \end{align}
    then the largest eigenvalues of $\Z$ are the eigenvalues of 
    $\Hm$, and zero otherwise.
\end{restatable}
\begin{proof}
    In order to prove the relationship of the eigenvalues of $\Z$ and $\Hm$,
    we prove the following two points separately:
    \begin{enumerate}
        \item The rank of $\Z$ is the same as $\Hm$.
        \item All non-zero eigenvalues of $\Hm$ are eigenvalues of $\Z$
            with the same multiplicity.
    \end{enumerate}

    For the rank of $\Z$, we take a look at how each column is constructed. 
    Consider column $j$ of $\Z$:
    \begin{align}
        \Z_j = [\Hm_{1,d(j)}v^{(1)}, \ldots, \Hm_{1,d(j)}v^{(N^{(1)})},\Hm_{2,d(j)}v^{(N^{(1)}+1)}, \ldots ,\Hm_{D,d(j)}v^{(N)}]^T.
    \end{align}
   For two columns $i$ and $j$ where $d(i) = d(j)$, these columns
    are identical. Therefore, the rank of $\Z$ will be, at most, 
    the number of hubs, $D$. Further, we can see 
    that the elements of a column $j$ in $\Z$ are simply scaled elements of 
    column $d(j)$ in $\Hm$. So any linearly dependent
    columns in $\Hm$ will also be linearly dependent in $\Z$. Therefore, the
    rank of the two matrices are the same.
%
    
    For the second point, we show there is a bijective mapping from eigenpairs
    of $\Hm$ to eigenpairs of $\Z$.
    Let  $(\lambda, \bm y)$ be an eigenpair of $\Hm$ (with $\lambda \neq 0$), i.e.
    \begin{align}
        \Hm \bm y = \lambda \bm y.
    \end{align}
Define the $N$-vector $\bm x$ with components $\bm x_i = v^{(i)} \bm y_{d(i)}$.
    We will show that $\Z \bm x = \lambda \bm x$. 
    Looking at the $i$-th entry of the vector
    $\Z \bm x$, we have
    \begin{align}
        (\Z \bm x)_i &= \sum_{j=1}^N \Z_{i,j} \bm x_j. 
    \end{align}
Applying the definition of $\Z$ and $\bm x$, we obtain
    \begin{align}
        (\Z \bm x)_i &= \sum_{j=1}^N \frac{1}{v^{(i)}} 
            \Hm_{d(i),d(j)} v^{(j)} \bm y_{d(j)} \\
        &= \frac{1}{v^{(i)}} \sum_{l=1}^D \Hm_{d(i),l} \bm y_l 
            \sum_{k \in \M^{(l)}} v^{(k)} \\ 
        &= \frac{1}{v^{(i)}} \sum_{l=1}^D \Hm_{d(i),l} \bm y_l.
    \end{align}
Note that the $m$-th entry of the vector $\Hm \bm y$ equals
    $\sum_{l=1}^D \Hm_{m,l} \bm y_l = \lambda \bm y_m$. Applying this equality, we obtain
    \begin{align}
        (\Z \bm x)_i &=  \frac{1}{v^{(i)}} \lambda \bm y_{d(i)} \\
        &= \lambda \bm x_i.
    \end{align}
    Therefore,  for any eigenpair $(\lambda, \bm y)$ of $\Hm$, we can find an eigenpair
    $(\lambda, \bm x)$ of $\Z$.    It is left to prove that this mapping is a bijection. 
    
    Suppose eigenvalue $\lambda$ of $\Hm$ has multiplicity $k >1$.
    We consider any two of the $k$ eigenpairs $(\lambda, \bm c)$ 
    and $(\lambda, \bm d)$. Let the
    corresponding eigenpairs of $\Z$ be $(\lambda, \bm e)$ and $(\lambda, \bm f)$.
    We know that $\bm e \neq \bm f$ because $\bm c$ and $\bm d$ are unique,
    and there must exist an index $i$ such that 
    $v^{(i)} \bm c_{d(i)} \neq v^{(i)} \bm d_{d(i)}$.
    Therefore, the mapping of eigenpairs of $\Hm$ to eigenpairs
    of $\Z$ is a bijection.
\end{proof}

\begin{restatable}{prop}{ZV} \label{ZV.prop}
    Given definition of $\Z$ and $\V$ in 
    Proposition~\ref{Za.prop}, it is the case that 
    \begin{align}
        \Z\V = \V\Z = \Z.
    \end{align}
\end{restatable}
\begin{proof}
    First, we prove that $\V\Z = \Z$. Note that the $i$-th row
    of $\V$ contains either $v_i$ or zero. 
    Looking at an arbitrary
    entry $i,j$ of $\V\Z$ we have: 
    \begin{align}
        (\V\Z)_{i,j} = v^{(i)} \sum_{r \in M_{d(i)}} \Z_{r,j} \\
            (\V\Z)_{i,j} = v^{(i)} \Hm_{d(i),d(i)} \\ 
            (\V\Z)_{i,j} = \Z_{i,j}. 
    \end{align}

    Next we prove that $\Z\V = \Z$. Note that for any row $i$ in $\Z$,
    $\Z_{i,j}=\Z_{i,k}$ when $d(j)=d(k)$.
    \begin{align}
        (\Z\V)_{i,j} = \Z_{i,j} \sum_{r=1}^N \V_{r,j}.
    \end{align}
    Since $\V$ is column stochastic:
    \begin{align}
        (\Z\V)_{i,j} = \Z_{i,j}.
    \end{align}
\end{proof}

\begin{restatable}{prop}{TkA} \label{TkA.prop}
    Let $\A = \bm a \one^T$. 
    Given our definition of $\T_k$ in (\ref{Tk.def}), 
    \begin{align}
        \T_k \A = \A \T_k = \A
    \end{align}
    for all $k$.
\end{restatable}

\begin{proof}
We prove each of the three cases of $\T_k$:
$\I$, $\V$, and $\Z$. Clearly, $\I \A = \A \I = \A$.
It is left to prove $\V \A = \A \V = \A$ and 
$\Z \A = \A \Z = \A$.

We can see that $\Z \A = \A$ 
since $\bm a$ is a right eigenvector of $\Z$ with eigenvalue $1$:
$\Z \A = \Z \bm a \one^T = \bm a \one^T = \A$.
Similarly, we can see that $\A \Z = \bm a \one^T \Z = \bm a \one^T
= \A$ as $\one^T$ is a left eigenvector of $\Z$. The same holds for
$\V$.
\end{proof}

\begin{lemma} \label{lemma41}
    Under Assumptions~\ref{unbias.asp} 
    and \ref{variance.asp}, the variance of the weighted average
    stochastic gradient is bounded as follows:
\begin{align}
    \mathbb{E}_k[\Vert \Gk - \Hk \Vert^2] 
    &\leq \sum_{i=1}^N \bm a_i^2 
        \left[\left(\bm p_{i}(\beta-1)+1\right)
        \left\Vert \nabla F(\bm x_k^{(i)}) \right\Vert^2 + \bm p_{i}\sigma^2 \right] 
        \nonumber \\ &~~~~~~~~~~~~~~~~~~~~~~~~~~~~~
        + \sum_{l=1}^N \sum_{j \neq l}^N \bm a_l \bm a_j
        (1-\bm p_{j})(1-\bm p_{l})
        \nabla F(\bm x_k^{(j)})^T\nabla F(\bm x_k^{(l)}) .
\end{align}
\end{lemma}
\begin{proof}
\begin{align}
    &\mathbb{E}_k[\Vert \Gk - \Hk \Vert^2] 
    = \mathbb{E}_k \left[ \left\Vert \sum_{i=1}^N \bm a_i(\bm g_k^{(i)} 
    - \nabla F(\bm x_k)) \right\Vert^2 \right]\\
    &= \mathbb{E}_k\left[
    \sum_{i=1}^N \bm a_i^2 \Vert \bm g_k^{(i)} - \nabla F(\bm x_k) \Vert^2
    + \sum_{l=1}^N \sum_{j \neq l}^N \bm a_l \bm a_j \left\langle \bm g_k^{(j)} - \nabla F(\bm x_k^{(j)}), 
    \bm g_k^{(l)} - \nabla F(\bm x_k^{(l)}) \right\rangle \right]\\
    &= 
    \sum_{i=1}^N \bm a_i^2 \mathbb{E}_k\Vert \bm g_k^{(i)} - \nabla F(\bm x_k) \Vert^2
    + \sum_{l=1}^N \sum_{j \neq l}^N \bm a_l \bm a_j
    \mathbb{E}_k\left[\left\langle \bm g_k^{(j)} - \nabla F(\bm x_k^{(j)}), 
    \bm g_k^{(l)} - \nabla F(\bm x_k^{(l)}) \right\rangle \right] .
    \label{needcrosslemma1a.eq}
\end{align}

Looking at the cross-terms in (\ref{needcrosslemma1a.eq}):
\begin{align}
\mathbb{E}_k\Bigg[\Big\langle \bm g_k^{(j)} - \nabla F(\bm x_k^{(j)}), 
    \bm g_k^{(l)} &- \nabla F(\bm x_k^{(l)}) \Big\rangle \Bigg] \nonumber \\
                  &= \mathbb{E}_k\left[(\bm g_k^{(j)})^T\bm g_k^{(l)}\right] 
    - \mathbb{E}_k\left[(\bm g_k^{(j)})^T\nabla F(\bm x_k^{(l)})\right]
    \nonumber \\ &~~~- 
    \mathbb{E}_k\left[\nabla F(\bm x_k^{(j)})^T\bm g_k^{(l)}\right]
    + \nabla F(\bm x_k^{(j)})^T\nabla F(\bm x_k^{(j)}) \\
    &= \bm p_{j}\bm p_{l}\nabla F(\bm x_k^{(j)})^T\nabla F(\bm x_k^{(l)})
    - \bm p_{j} \nabla F(\bm x_k^{(j)})^T\nabla F(\bm x_k^{(l)}) \nonumber \\
    &~~~
    - \bm p_{l} \nabla F(\bm x_k^{(j)})^T\nabla F(\bm x_k^{(l)})
    + \nabla F(\bm x_k^{(j)})^T\nabla F(\bm x_k^{(j)}) \\
    &= (1-\bm p_{j})(1-\bm p_{l})\nabla F(\bm x_k^{(j)})^T\nabla F(\bm x_k^{(l)}) .
    \label{crosslemma1a.eq}
\end{align}

Plugging (\ref{crosslemma1a.eq}) into (\ref{needcrosslemma1a.eq}) we have:
\begin{align}
    \mathbb{E}_k[\Vert \Gk - \Hk \Vert^2] 
    &= \sum_{i=1}^N \bm a_i^2 \mathbb{E}_k
        \Vert \bm g_k^{(i)} - \nabla F(\bm x_k) \Vert^2 
        \nonumber \\ &~~~~~~~~~~~~~~~~~~~~~~~~~~~
        + \sum_{l=1}^N \sum_{j \neq l}^N \bm a_l \bm a_j 
        (1-\bm p_{j})(1-\bm p_{l})
        \nabla F(\bm x_k^{(j)})^T\nabla F(\bm x_k^{(l)}) \\
    &\leq \sum_{i=1}^N \bm a_i^2 
        \left[\left(\bm p_{i}(\beta-1)+1\right)
        \left\Vert \nabla F(\bm x_k^{(i)}) \right\Vert^2 + \bm p_{i}\sigma^2 \right] 
        \nonumber \\ &~~~~~~~~~~~~~~~~~~~~~~~~~~~
        + \sum_{l=1}^N \sum_{j \neq l}^N \bm a_l \bm a_j
        (1-\bm p_{j})(1-\bm p_{l})
        \nabla F(\bm x_k^{(j)})^T\nabla F(\bm x_k^{(l)}) \label{aspbound.eq}
\end{align}
where (\ref{aspbound.eq}) follows from Assumption~\ref{variance.asp} and
(\ref{asp5implication.eq}).
\end{proof}

\begin{lemma} \label{lemma4}
    Under Assumptions~\ref{unbias.asp} 
    and \ref{variance.asp}, the squared norm of 
    the stochastic gradients is bounded by:
\begin{align}
    \mathbb{E}_k \left[\Vert \Gk \Vert^2 \right] 
    &\leq \sum_{i=1}^N  
        \left[\bm a_i^2\left(\bm p_{i}(\beta+1)-\bm p_i^2\right) + \bm a_i\bm p_i^2\right]
        \left\Vert \nabla F(\bm x_k^{(i)}) \right\Vert^2 + \sum_{i=1}^N \bm a_i^2\bm p_{i}\sigma^2  
\end{align}
\end{lemma}
\begin{proof}
\begin{align}
    &\mathbb{E}_k \left[\Vert \Gk \Vert^2 \right] 
    =\mathbb{E}_k \left[\Vert \Gk - \mathbb{E}_k [ \Gk ] \Vert^2 \right]
         +\Vert\mathbb{E}_k [\Gk] \Vert^2 \\
    &=\mathbb{E}_k \left[\Vert \Gk 
        - \sum_{i=1}^N \bm a_i \bm p_i \nabla F(\bm x_k^{(i)}) \Vert^2 \right]
        + \left\Vert \sum_{i=1}^N \bm a_i \bm p_i \nabla F(\bm x_k^{(i)}) \right\Vert^2 \\
    &=\mathbb{E}_k \left[\left\Vert \Gk - \sum_{i=1}^N \bm a_i \nabla F(\bm x_k^{(i)}) 
        + \sum_{i=1}^N \bm a_i (1-\bm p_i) \nabla F(\bm x_k^{(i)}) \right\Vert^2 \right]
        + \left\Vert \sum_{i=1}^N \bm a_i \bm p_i \nabla F(\bm x_k^{(i)}) \right\Vert^2 .
        \label{needGk.eq}
\end{align}

Applying the definition of $\Gk$ to (\ref{needGk.eq}) we get:
\begin{align}
    &\mathbb{E}_k \left[\Vert \Gk \Vert^2 \right] 
        \nonumber \\ 
    &=\mathbb{E}_k\left[\left\Vert  \sum_{i=1}^N \left[\bm a_i\bm g_k^{(i)} - \bm a_i\nabla F(\bm x_k^{(i)})
        +  \bm a_i (1-\bm p_i) \nabla F(\bm x_k^{(i)}) \right] \right\Vert^2  \right]
        + \left\Vert \sum_{i=1}^N \bm a_i \bm p_i \nabla F(\bm x_k^{(i)}) \right\Vert^2 \\
    &=\mathbb{E}_k\left[ \sum_{i=1}^N \left\Vert \bm a_i\bm g_k^{(i)} - \bm a_i\nabla F(\bm x_k^{(i)})
        +  \bm a_i (1-\bm p_i) \nabla F(\bm x_k^{(i)}) \right\Vert^2   \right]
        + \left\Vert \sum_{i=1}^N \bm a_i \bm p_i \nabla F(\bm x_k^{(i)}) \right\Vert^2
        \nonumber \\ &~~~
        + \mathbb{E}_k \Bigg[\sum_{j=1}^N \sum_{l=1, l \neq j}^N \bm a_l \bm a_j 
        \Big\langle (\bm g_k^{(j)} - \nabla F(\bm x_k^{(j)})) + (1-\bm p_j) \nabla F(\bm x_k^{(j)}),
        \nonumber \\ &~~~~~~~~~~~~~~~~~~~~~~~~~~~~~~~~~~~~~~~~~~~~~~~~~~~~~~~~~~~~~
        (\bm g_k^{(l)} - \nabla F(\bm x_k^{(l)})) + (1-\bm p_l) \nabla F(\bm x_k^{(l)}) \Big\rangle \Bigg] .
        \label{needcrosslemma1.eq}
\end{align}

Let the cross-terms in (\ref{needcrosslemma1.eq}) be
\begin{align}
    CR
    &=
        \mathbb{E}_k \Bigg[\sum_{j=1}^N \sum_{l=1, l \neq j}^N \bm a_l \bm a_j 
        \Big\langle (\bm g_k^{(j)} - \nabla F(\bm x_k^{(j)})) + (1-\bm p_j) \nabla F(\bm x_k^{(j)}),
        \nonumber \\ &~~~~~~~~~~~~~~~~~~~~~~~~~~~~~~~~~~~~~~~~~~~~~~~~~~~~~~~~~~~~~
        (\bm g_k^{(l)} - \nabla F(\bm x_k^{(l)})) + (1-\bm p_l) \nabla F(\bm x_k^{(l)}) \Big\rangle \Bigg] .
\end{align}
We can simplify $CR$ as follows:
\begin{align}
    CR &=    
        \sum_{j=1}^N \sum_{l=1, l \neq j}^N \bm a_l \bm a_j \mathbb{E}_k \Big[
        (\bm g_k^{(j)} - \nabla F(\bm x_k^{(j)}))^T (\bm g_k^{(l)} - \nabla F(\bm x_k^{(l)}))
        \nonumber \\ &~~~
        + (\bm g_k^{(j)} - \nabla F(\bm x_k^{(j)}))^T(1-\bm p_l) \nabla F(\bm x_k^{(l)}) 
        + (1-\bm p_j) \nabla F(\bm x_k^{(j)})^T(\bm g_k^{(l)} - \nabla F(\bm x_k^{(l)}))
        \nonumber \\ &~~~~~~~~~~~~~~~~~~~~~~~~~~~~~~~~~~~~~~~~~~~~~~~~~~~~~~~~~~~~~
        + (1-\bm p_j)(1-\bm p_l) \nabla F(\bm x_k^{(j)})^T\nabla F(\bm x_k^{(l)})^T \Big]\\
    &=
        \sum_{j=1}^N \sum_{l=1, l \neq j}^N \bm a_l \bm a_j \Big[\mathbb{E}_k \left[
        (\bm g_k^{(j)} - \nabla F(\bm x_k^{(j)}))^T (\bm g_k^{(l)} - \nabla F(\bm x_k^{(l)}))\right]
        \nonumber \\ &~~~~~~~~~~~~~~~~~~~~~~~~~~~~~~~~~~~~~~~~~~~~~~~~~~~~~~
        + (\mathbb{E}_k[\bm g_k^{(j)}] - \nabla F(\bm x_k^{(j)}))^T(1-\bm p_l) \nabla F(\bm x_k^{(l)})  
        \nonumber \\ &~~~~~~~~~~~~~~~~~~~~~~~~~~~~~~~~~~~~~~~~~~~~~~~~~~~~~~
        + (1-\bm p_j) \nabla F(\bm x_k^{(j)})^T(\mathbb{E}_k[\bm g_k^{(l)}] - \nabla F(\bm x_k^{(l)}))
        \nonumber \\ &~~~~~~~~~~~~~~~~~~~~~~~~~~~~~~~~~~~~~~~~~~~~~~~~~~~~~~
        + (1-\bm p_j)(1-\bm p_l) \nabla F(\bm x_k^{(j)})^T\nabla F(\bm x_k^{(l)})^T \Big] .
        \label{needunbias.eq}
\end{align}

Applying Assumption~\ref{unbias.asp} to (\ref{needunbias.eq}), we get:
\begin{align}
    CR &=    
        \sum_{j=1}^N \sum_{l=1, l \neq j}^N \bm a_l \bm a_j \Big[\mathbb{E}_k \left[
        (\bm g_k^{(j)} - \nabla F(\bm x_k^{(j)}))^T (\bm g_k^{(l)} - \nabla F(\bm x_k^{(l)}))\right]
        \nonumber \\ &~~~
        + (\bm p_j-1)(1-\bm p_l)\nabla F(\bm x_k^{(j)})^T\nabla F(\bm x_k^{(l)}) 
        + (\bm p_l-1)(1-\bm p_j)\nabla F(\bm x_k^{(j)})^T\nabla F(\bm x_k^{(l)}) 
        \nonumber \\ &~~~~~~~~~~~~~~~~~~~~~~~~~~~~~~~~~~~~~~~~~~~~~~~~~~~~~~~~~~~~~~~~~~
        + (1-\bm p_j)(1-\bm p_l)\nabla F(\bm x_k^{(j)})^T\nabla F(\bm x_k^{(l)}) \Big] \\
    &=
        \sum_{j=1}^N \sum_{l=1, l \neq j}^N \bm a_l \bm a_j \Big[\mathbb{E}_k \left[
        (\bm g_k^{(j)} - \nabla F(\bm x_k^{(j)}))^T (\bm g_k^{(l)} - \nabla F(\bm x_k^{(l)}))\right]
        \nonumber \\ &~~~~~~~~~~~~~~~~~~~~~~~~~~~~~~~~~~~~~~~~~~~~~~~~~~~~~~~~~~~~~~~~~~
        - (1-\bm p_l)(1-\bm p_j)\nabla F(\bm x_k^{(j)})^T\nabla F(\bm x_k^{(l)}) \Big] .
        \label{crosslemma1.eq}
\end{align}

Applying (\ref{crosslemma1.eq}) to (\ref{needcrosslemma1.eq}) we have:
\begin{align}
    \mathbb{E}_k \left[\Vert \Gk \Vert^2 \right] 
    &=\mathbb{E}_k\left[ \sum_{i=1}^N \left\Vert \bm a_i\bm g_k^{(i)} - \bm a_i\nabla F(\bm x_k^{(i)})
        +  \bm a_i (1-\bm p_i) \nabla F(\bm x_k^{(i)}) \right\Vert^2   \right]
        \nonumber \\ &~~~
        + \sum_{j=1}^N \sum_{l=1, l \neq j}^N \bm a_l \bm a_j \Big[\mathbb{E}_k \left[
        (\bm g_k^{(j)} - \nabla F(\bm x_k^{(j)}))^T (\bm g_k^{(l)} - \nabla F(\bm x_k^{(l)}))\right]
        \nonumber \\ &~~~~~~~~~~~~~~~~~~~~~~~~~~~~~~~~~~~~~~~~~~~~~~~
        - (1-\bm p_l)(1-\bm p_j)\nabla F(\bm x_k^{(j)})^T\nabla F(\bm x_k^{(l)}) \Big] 
        \nonumber \\ &~~~~~~~~~~~~~~~~~~~~~~~~~~~~~~~~~~~~~~~~~~~~~~~~~~~~~~~~~~~~~~~~~~~~~~~~~~~
        + \left\Vert \sum_{i=1}^N \bm a_i \bm p_i \nabla F(\bm x_k^{(i)}) \right\Vert^2.
        \label{Gkexpand}
\end{align}
Expanding the first term in (\ref{Gkexpand}) we have:

\begin{align}
    \mathbb{E}_k \left[\Vert \Gk \Vert^2 \right] 
    &=\mathbb{E}_k\left[ \sum_{i=1}^N 
        \left\Vert \bm a_i\bm g_k^{(i)} - \bm a_i\nabla F(\bm x_k^{(i)}) \right\Vert^2 \right]
        + \sum_{i=1}^N \bm a_i^2 (1-\bm p_i)^2 \left\Vert \nabla F(\bm x_k^{(i)}) \right\Vert^2
    \nonumber \\ &~~~
        + \sum_{i=1}^N \underbrace{\mathbb{E}_k\left[         
        \left\langle \bm a_i\bm g_k^{(i)} - \bm a_i\nabla F(\bm x_k^{(i)}),
        \bm a_i (1-\bm p_i) \nabla F(\bm x_k^{(i)}) \right\rangle \right]}_{CR1}
    \nonumber \\ &~~~~~
        + \sum_{i=1}^N \underbrace{\mathbb{E}_k\left[         
        \left\langle \bm a_i (1-\bm p_i) \nabla F(\bm x_k^{(i)}), 
        \bm a_i\bm g_k^{(i)} - \bm a_i\nabla F(\bm x_k^{(i)})\right\rangle \right]}_{CR2}
    \nonumber \\ &~~~~~~~~~
        + \sum_{j=1}^N \sum_{l=1, l \neq j}^N \bm a_l \bm a_j \Big[\mathbb{E}_k \left[
        (\bm g_k^{(j)} - \nabla F(\bm x_k^{(j)}))^T (\bm g_k^{(l)} - \nabla F(\bm x_k^{(l)}))\right]
    \nonumber \\ &~~~~~~~~~~~~~~~~~~~~~~~~~~~~~~~~~~~~~~~~~~~~~
        - (1-\bm p_l)(1-\bm p_j)\nabla F(\bm x_k^{(j)})^T\nabla F(\bm x_k^{(l)}) \Big] 
        \nonumber \\ &~~~~~~~~~~~~~~~~~~~~~~~~~~~~~~~~~~~~~~~~~~~~~~~~~~~~~~~~~~~~~~~~~~~~~~~~~~~
        + \left\Vert \sum_{i=1}^N \bm a_i \bm p_i \nabla F(\bm x_k^{(i)}) \right\Vert^2.
        \label{need2cross.eq}
\end{align}

We simplify $CR1$:
\begin{align}
    CR1 &= 
   	\left[ \left\langle \bm a_i \bm p_i \mathbb{E}_k (g(\bm x_k^{(i)}))^T - \bm a_i \nabla F(\bm x_k^{(i)}), \bm a_i(1- \bm p_i) \nabla F(\bm x_k^{(i)})  \right\rangle
    \right]\\
    &= 
   	\left[ \left\langle \bm a_i(\bm p_i -1) \nabla F(\bm x_k^{(i)}), \bm a_i(1- \bm p_i) \nabla F(\bm x_k^{(i)})  \right\rangle
    \right]\\
    &= 
     - \bm a_i^2(1-\bm p_i)^2\left\Vert \nabla F(\bm x_k^{(i)}) \right\Vert^2.
     \label{CR1.eq}
\end{align}

Similarly, for $CR2$:
\begin{align}
    CR2
    &= 
     - \bm a_i^2(1-\bm p_i)^2\left\Vert \nabla F(\bm x_k^{(i)}) \right\Vert^2 .
     \label{CR2.eq}
\end{align}

Plugging (\ref{CR1.eq}) and (\ref{CR2.eq}) back into (\ref{need2cross.eq}):
\begin{align}
    \mathbb{E}_k \left[\Vert \Gk \Vert^2 \right] 
    &=\mathbb{E}_k\left[ \sum_{i=1}^N 
        \left\Vert \bm a_i\bm g_k^{(i)} - \bm a_i\nabla F(\bm x_k^{(i)}) \right\Vert^2 \right]
        - \sum_{i=1}^N \bm a_i^2 (1-\bm p_i)^2 \left\Vert \nabla F(\bm x_k^{(i)}) \right\Vert^2
    \nonumber \\ &~~~
        + \sum_{j=1}^N \sum_{l=1, l \neq j}^N \bm a_l \bm a_j \Big[\mathbb{E}_k \left[
        (\bm g_k^{(j)} - \nabla F(\bm x_k^{(j)}))^T (\bm g_k^{(l)} - \nabla F(\bm x_k^{(l)}))\right]
    \nonumber \\ &~~~~~~~~~~~~~~~~~~~~~~~~~~~~~~~~~~~~~~~~~~~~~~~~~~~~~
        - (1-\bm p_l)(1-\bm p_j)\nabla F(\bm x_k^{(j)})^T\nabla F(\bm x_k^{(l)}) \Big] 
    \nonumber \\ &~~~~~~~~~~~~~~~~~~~~~~~~~~~~~~~~~~~~~~~~~~~~~~~~~~~~~~~~~~~~~~~~~~~~~~~~~~~~~~~~~~
        + \left\Vert \sum_{i=1}^N \bm a_i \bm p_i \nabla F(\bm x_k^{(i)}) \right\Vert^2.
\end{align}

We can simplify by observing that: 
\begin{align}
    \mathbb{E}_k \left[\Vert \Gk - \Hk \Vert^2 \right]
    &=\mathbb{E}_k\left[ \sum_{i=1}^N \left\Vert \bm a_i\bm g_k^{(i)} 
        - \bm a_i\nabla F(\bm x_k^{(i)})\right\Vert^2\right]
        \nonumber \\ &~~~
        + \sum_{j=1}^N \sum_{l=1, l \neq j}^N \bm a_l \bm a_j \mathbb{E}_k \Big[
        (\bm g_k^{(j)} - \nabla F(\bm x_k^{(j)}))^T (\bm g_k^{(l)} - \nabla F(\bm x_k^{(l)})) \Big]
\end{align}
which gives us:
\begin{align}
    \mathbb{E}_k \left[\Vert \Gk \Vert^2 \right] 
    &=\mathbb{E}_k \left[\Vert \Gk - \Hk \Vert^2 \right]
        - \sum_{i=1}^N \bm a_i^2 (1-\bm p_i)^2 \left\Vert \nabla F(\bm x_k^{(i)}) \right\Vert^2
        + \left\Vert \sum_{i=1}^N \bm a_i \bm p_i \nabla F(\bm x_k^{(i)}) \right\Vert^2
    \nonumber \\ &~~~~~~~~~~~~~~~~~~~~~~~~~~~~~~~~~~
        - \sum_{j=1}^N \sum_{l=1, l \neq j}^N \bm a_l \bm a_j 
        (1-\bm p_l)(1-\bm p_j)\nabla F(\bm x_k^{(j)})^T\nabla F(\bm x_k^{(l)}) 
        \label{needlemma1a}
\end{align}
Applying Lemma~\ref{lemma41} to (\ref{needlemma1a}):
\begin{align}
    \mathbb{E}_k \left[\Vert \Gk \Vert^2 \right] 
    &\leq \sum_{i=1}^N \bm a_i^2 
        \left[\left(\bm p_{i}(\beta-1)+1\right)
        \left\Vert \nabla F(\bm x_k^{(i)}) \right\Vert^2 + \bm p_{i}\sigma^2 \right] 
        \nonumber \\ &~~~~~~~~~~~~~~~~~~~~~~~~~~~~~
        - \sum_{i=1}^N \bm a_i^2 (1-\bm p_i)^2 \left\Vert \nabla F(\bm x_k^{(i)}) \right\Vert^2 
        + \left\Vert \sum_{i=1}^N \bm a_i \bm p_i \nabla F(\bm x_k^{(i)}) \right\Vert^2 
        \label{lemma4.eq}\\
    &\leq \sum_{i=1}^N \bm a_i^2 
        \left[\left(\bm p_{i}(\beta-1)+1\right)
        \left\Vert \nabla F(\bm x_k^{(i)}) \right\Vert^2 + \bm p_{i}\sigma^2 \right] 
        \nonumber \\ &~~~~~~~~~~~~~~~~~~~~~~~~~~~~~
        - \sum_{i=1}^N \bm a_i^2 (1-\bm p_i)^2 \left\Vert \nabla F(\bm x_k^{(i)}) \right\Vert^2 
        + \sum_{i=1}^N \bm a_i \bm p_i^2 \left\Vert \nabla F(\bm x_k^{(i)}) \right\Vert^2 
        \label{jensen.eq}\\
    &= \sum_{i=1}^N  
        \left[\bm a_i^2\left(\bm p_{i}(\beta+1)-\bm p_i^2\right) + \bm a_i\bm p_i^2\right]
        \left\Vert \nabla F(\bm x_k^{(i)}) \right\Vert^2 + \sum_{i=1}^N \bm a_i^2\bm p_{i}\sigma^2  
\end{align}
where equation (\ref{jensen.eq}) follows from Jensen's inequality.
\end{proof}

\begin{lemma} \label{lemma5}
    Under Assumption~\ref{unbias.asp}, the expected inner product of the batch
    gradient and the weighted average stochastic gradient is equal to:
\begin{align}
    \mathbb{E}_k[\langle \nabla F(\bm u_k), \Gk \rangle ] 
    &=\frac{1}{2}\Vert\nabla F(\bm u_k)\Vert^2 
        + \sum_{i=1}^N\frac{\bm a_i}{2}\left\Vert \bm p_{i}\nabla F(\bm x_k^{(i)}) \right\Vert^2 
        \nonumber \\ &~~~~~~~~~~~~~~~~~~~~~~~~~~~~~~~~~~~~~~~~~~~~~~~~~~~~~~
        - \sum_{i=1}^N\frac{\bm a_i}{2}\left\Vert \nabla F(\bm u_k) 
        - \bm p_{i}\nabla F(\bm x_k^{(i)}) \right\Vert^2 
\end{align}
\end{lemma}
\begin{proof}
\begin{align}
    \mathbb{E}_k[\langle \nabla F(\bm u_k), \Gk \rangle ] 
    &=\mathbb{E}_k\left[\left\langle \nabla F(\bm u_k), 
        \sum_{i=1}^N \bm a_i \bm g_k^{(i)} \right\rangle \right]\\
    &=\left\langle \nabla F(\bm u_k), 
        \sum_{i=1}^N \bm p_{i}\bm a_i\nabla F(\bm x_k^{(i)}) \right\rangle \label{usegtoF.eq} \\
    &=\sum_{i=1}^N \bm a_i\left\langle \nabla F(\bm u_k), 
        \bm p_{i}\nabla F(\bm x_k^{(i)}) \right\rangle \\
    &=\sum_{i=1}^N\frac{\bm a_i}{2}\left[\Vert\nabla F(\bm u_k)\Vert^2 
        + \Vert \bm p_{i}\nabla F(\bm x_k^{(i)}) \Vert^2 
        - \Vert \nabla F(\bm u_k) - \bm p_{i}\nabla F(\bm x_k^{(i)}) \Vert^2 \right] 
        \label{ab.eq} \\
    &=\frac{1}{2}\Vert\nabla F(\bm u_k)\Vert^2 
        + \sum_{i=1}^N\frac{\bm a_i}{2}\left\Vert \bm p_{i}\nabla F(\bm x_k^{(i)}) \right\Vert^2 
        \nonumber \\ &~~~~~~~~~~~~~~~~~~~~~~~~~~~~~~~~~~~~~~~~~~~~~~~~~~~~~~
        - \sum_{i=1}^N\frac{\bm a_i}{2} \left\Vert \nabla F(\bm u_k) 
        - \bm p_{i}\nabla F(\bm x_k^{(i)}) \right\Vert^2 
\end{align}
where (\ref{usegtoF.eq}) follows from (\ref{gtoF.eq}), and (\ref{ab.eq}) 
follows from the fact that, for arbitrary vectors $\bm y$ and $\bm z$,  $2\bm y^T \bm z = ||\bm y||^2 + 
||\bm z ||^2 - ||\bm y - \bm z||^2$.
\end{proof}

\begin{lemma} \label{lemma3}
    Under Assumption~\ref{F.assum}, following the update rule given in 
    (\ref{x_recur.eq}), if all model parameters are initialized at
    the same $\bm x_1$, the expected weighted average 
    gradient is bounded as follows: 
\begin{align}
    \mathbb{E}\Bigg[& \frac{1}{K} \sum_{k=1}^K  \Vert\nabla F(\bm u_k)\Vert^2 \Bigg]
    \leq \frac{2\left(F(\bm x_1) - F_{inf}]\right)}{\eta K}
        + \sigma^2 \eta L\sum_{i=1}^N  \bm a_i^2 \bm p_i  
        + \frac{2L^2}{K}\sum_{k=1}^K \mathbb{E}\Vert \X_k(\I - \A)  \Vert_{F_{\bm a}}^2
    \nonumber \\ &~~~
        - \frac{1}{K}\sum_{k=1}^K
        \sum_{i=1}^N \bm a_i  
        \left((4\bm p_i - \bm p_i^2 - 2) - \eta L \left(\bm a_i\bm p_{i}(\beta+1)-\bm a_i\bm p_i^2 + \bm p_i^2\right) \right)
        \mathbb{E} \left\Vert \nabla F(\bm x_k^{(i)}) \right\Vert^2 
\end{align}
where $\A = \bm a \one^T$.
\end{lemma}

\begin{proof}
According to Assumption~\ref{smooth.asp},
\begin{align}
\mathbb{E}_k[F(\bm u_{k+1})] - F(\bm u_k) 
&\leq \mathbb{E}_k \left[ \langle \nabla F(\bm u_k)), \bm u_{k+1} - \bm u_k \rangle + \frac{L}{2} \| \bm u_{k+1}  - \bm u_k \|_2^2 \right] \\
&= -\eta \mathbb{E}_k[\langle \nabla F(\bm u_k), \Gk \rangle] + \frac{\eta^2 L}{2} \mathbb{E}_k \left[\| \Gk \|_2^2 \right].
\end{align}

Plugging in Lemmas \ref{lemma4} and \ref{lemma5}, we get:
\begin{align}
    \mathbb{E}_k&[F(\bm u_{k+1})] - F(\bm u_k) \nonumber \\
   & \leq -\eta \Bigg[
        \frac{1}{2}\Vert\nabla F(\bm u_k)\Vert^2 
        + \sum_{i=1}^N\frac{\bm a_i}{2}\bm p_{i}^2\left\Vert \nabla F(\bm x_k^{(i)}) \right\Vert^2 
        - \sum_{i=1}^N\frac{\bm a_i}{2}\Vert \nabla F(\bm u_k) 
        - \bm p_{i}\nabla F(\bm x_k^{(i)}) \Vert^2 \Bigg] 
    \nonumber \\ &~~~~~~~~~~~
        + \frac{\eta^2 L}{2}\sum_{i=1}^N  
        \left(\bm a_i^2\bm p_{i}(\beta+1)-\bm a_i^2\bm p_i^2 + \bm a_i\bm p_i^2\right)
        \left\Vert \nabla F(\bm x_k^{(i)}) \right\Vert^2
        + \frac{\sigma^2 \eta^2 L}{2}\sum_{i=1}^N  \bm a_i^2 \bm p_i  \\
    &=
    -\frac{\eta}{2}\Vert\nabla F(\bm u_k)\Vert^2 
        + \frac{\eta}{2}\sum_{i=1}^N \bm a_i\Vert \nabla F(\bm u_k) 
        - \bm p_{i}\nabla F(\bm x_k^{(i)}) \Vert^2 
        + \frac{\sigma^2 \eta^2 L}{2}\sum_{i=1}^N  \bm a_i^2 \bm p_i 
    \nonumber \\ &~~~~~~~~~~~~~~~~~~~~~~~~~
        - \frac{\eta}{2}\sum_{i=1}^N \bm a_i  
        \left(\bm p_i^2 - \eta L \left(\bm a_i\bm p_{i}(\beta+1)-\bm a_i\bm p_i^2 + \bm p_i^2\right) \right)
        \left\Vert \nabla F(\bm x_k^{(i)}) \right\Vert^2 .
\end{align}

After some rearranging, we obtain:
\begin{align}
    \Vert\nabla F(\bm u_k)\Vert^2 
    &\leq
        \frac{2\left(F(\bm u_k) - \mathbb{E}_k[F(\bm u_{k+1})]\right)}{\eta}
        + \sigma^2 \eta L\sum_{i=1}^N  \bm a_i^2 \bm p_i  
        + \sum_{i=1}^N \bm a_i \Vert \nabla F(\bm u_k) 
        - \bm p_{i}\nabla F(\bm x_k^{(i)}) \Vert^2 
    \nonumber \\ &~~~~~~~~~~~~~
        - \sum_{i=1}^N \bm a_i  
        \left(\bm p_i^2 - \eta L \left(\bm a_i\bm p_{i}(\beta+1)-\bm a_i\bm p_i^2 + \bm p_i^2\right) \right)
        \left\Vert \nabla F(\bm x_k^{(i)}) \right\Vert^2 .
\end{align}

 Taking the total expectation over all iterations:
\begin{align}
    \mathbb{E}\left[ \frac{1}{K} \sum_{k=1}^K \Vert\nabla F(\bm u_k)\Vert^2 \right]
    &\leq \frac{2\left(F(\bm x_1) - F_{inf}]\right)}{\eta K}
        + \sigma^2 \eta L\sum_{i=1}^N  \bm a_i^2 \bm p_i  
    \nonumber \\ &
        + \frac{1}{K}\sum_{k=1}^K\sum_{i=1}^N \bm a_i
        \mathbb{E}\Vert \nabla F(\bm u_k) 
        - \bm p_{i}\nabla F(\bm x_k^{(i)}) \Vert^2 
    \nonumber \\ &
        - \frac{1}{K}\sum_{k=1}^K
        \sum_{i=1}^N \bm a_i  
        \left(\bm p_i^2 - \eta L \left(\bm a_i\bm p_{i}(\beta+1)-\bm a_i\bm p_i^2 + \bm p_i^2\right) \right)
        \mathbb{E} \left\Vert \nabla F(\bm x_k^{(i)}) \right\Vert^2 .
    \label{need_smoothness.eq}
\end{align}

The third term in (\ref{need_smoothness.eq}) can be bounded as:
\begin{align}
    \sum_{i=1}^N \bm a_i \mathbb{E}\Vert \nabla F(\bm u_k) &- \bm p_{i}\nabla F(\bm x_k^{(i)}) \Vert^2 \nonumber \\
   &= \sum_{i=1}^N \bm a_i \mathbb{E}\Vert \nabla F(\bm u_k) - \nabla F(\bm x_k^{(i)}) 
        + (1-\bm p_{i})\nabla F(\bm x_k^{(i)}) \Vert^2 \\
    &\leq \sum_{i=1}^N \left[ 2\bm a_i\mathbb{E}\Vert \nabla F(\bm u_k) - \nabla F(\bm x_k^{(i)}) \Vert^2
    + 2\bm a_i(1-\bm p_{i})^2 \mathbb{E}\Vert\nabla F(\bm x_k^{(i)}) \Vert^2 \right] \label{tri2.eq}
        \\
    &\leq \sum_{i=1}^N 2\bm a_iL^2\mathbb{E}\Vert \bm u_k - \bm x_k^{(i)} \Vert^2
        + \sum_{i=1}^N 2\bm a_i(1-\bm p_{i})^2 \mathbb{E}\Vert\nabla F(\bm x_k^{(i)}) \Vert^2 
        \label{smooth.eq}
\end{align}
where (\ref{tri2.eq}) follows from the fact that $\| \bm y + \bm z\|^2 \leq 2\| \bm y\|^2 + 2\| \bm z\|^2$,
and (\ref{smooth.eq}) follows from (\ref{tri2.eq}) by Assumption~\ref{smooth.asp}.

Recalling the definition of the weighted Frobenius norm and the definition
of $u$, we can simplify the first term in (\ref{smooth.eq}):
\begin{align}
    \sum_{i=1}^N 2\bm a_iL^2\mathbb{E}\Vert \bm u_k - \bm x_k^{(i)} \Vert^2  
    &= 2L^2\mathbb{E}\Vert \bm u_k \bm1^T - \X_k \Vert_{F_{\bm a}}^2 \\
    &= 2L^2 \mathbb{E}\left\Vert \X_k \bm a \bm1^T - \X_k \right\Vert_{F_{\bm a}}^2 \\
    &= 2L^2\mathbb{E}\Vert \X_k(\I - \A)  \Vert_{F_{\bm a}}^2 \label{smooth2.eq} .
\end{align}

Plugging (\ref{smooth.eq}) and (\ref{smooth2.eq}) back into (\ref{need_smoothness.eq}), we obtain:
\begin{align}
    \mathbb{E}&\left[ \frac{1}{K} \sum_{k=1}^K \Vert\nabla F(\bm u_k)\Vert^2 \right]
    \leq \frac{2\left(F(\bm x_1) - F_{inf}]\right)}{\eta K}
        + \sigma^2 \eta L\sum_{i=1}^N  \bm a_i^2 \bm p_i  
    \nonumber \\ &~~~~~~~~~~~~~~~~~~
        + \frac{2L^2}{K}\sum_{k=1}^K \mathbb{E}\Vert \X_k(\I - \A)  \Vert_{F_{\bm a}}^2
        + \frac{1}{K}\sum_{k=1}^K
        \sum_{i=1}^N 2\bm a_i(1-\bm p_{i})^2 \mathbb{E}\Vert\nabla F(\bm x_k^{(i)}) \Vert^2  
    \nonumber \\ &~~~~~~~~~~~~~~~~~
        - \frac{1}{K}\sum_{k=1}^K
        \sum_{i=1}^N \bm a_i  
        \left(\bm p_i^2 - \eta L \left(\bm a_i\bm p_{i}(\beta+1)-\bm a_i\bm p_i^2 + \bm p_i^2\right) \right)
        \mathbb{E} \left\Vert \nabla F(\bm x_k^{(i)}) \right\Vert^2 \\ 
    &= \frac{2\left(F(\bm x_1) - F_{inf}]\right)}{\eta K}
        + \sigma^2 \eta L\sum_{i=1}^N  \bm a_i^2 \bm p_i  
        + \frac{2L^2}{K}\sum_{k=1}^K \mathbb{E}\Vert \X_k(\I - \A)  \Vert_{F_{\bm a}}^2
    \nonumber \\ &~~~
        - \frac{1}{K}\sum_{k=1}^K
        \sum_{i=1}^N \bm a_i  
        \left((4\bm p_i - \bm p_i^2 - 2) - \eta L \left(\bm a_i\bm p_{i}(\beta+1)-\bm a_i\bm p_i^2 + \bm p_i^2\right) \right)
        \mathbb{E} \left\Vert \nabla F(\bm x_k^{(i)}) \right\Vert^2 .
\end{align}
\end{proof}

\begin{lemma} \label{lemma9}
    Given the properties of $\Z$ and $\V$ given in Propositions \ref{Za.prop}
    and \ref{Z.prop}, 
    it is the case that: 
    \begin{align}
        \left\Vert \Z^j - \A \right\Vert_{op} = \zeta^j,~~~
        \left\Vert \V - \A \right\Vert_{op} = 1,~~~
        \left\Vert \I - \A \right\Vert_{op} = 1
    \end{align}
where $\A = \bm a \one^T$ and 
$\zeta = \max \{|\lambda_2(\Hm)|, |\lambda(\Hm)|\}$.
\end{lemma}
\begin{proof}
According to the definition of the matrix operator norm,
\begin{align}
    \left\Vert \Z^j - \A \right\Vert_{op} 
    &= \sqrt{\lambda_{max}((\Z^j - \A)^T(\Z^j - \A))} \\
    &= \sqrt{\lambda_{max}(\Z^{2j} - \A\Z^j - \Z^j \A + \A)} \label{aexp.eq} \\
    &= \sqrt{\lambda_{max}(\Z^{2j} - \A)} \label{commute.eq} 
\end{align}
where (\ref{aexp.eq}) follows from $\A^j = \A$, and (\ref{commute.eq})
follows from $\A\Z = \Z\A = \A$.

We can simplify (\ref{commute.eq}) further:
\begin{align}
    &= \sqrt{\lambda_{max}(\Z^{2j} - \A^{2j})} \\
    &= \sqrt{\lambda_{max}(\Z - \A)^{2j}} \label{eigza.eq}
\end{align}
where (\ref{eigza.eq}) follows from the commutability of $\Z$ and $\A$.

Based on Proposition~\ref{Z.prop}, the non-zero eigenvalues of $\Z$ are
the same as $\Hm$.
As shown in Lemma 6 of \citet{rotaru2004dynamic}, for a matrix $\Z$ with
the properties in Proposition~\ref{Za.prop}, the spectral 
norm of $\Z - \A$ is equal to $\zeta$.

Therefore:
\begin{align}
    \left\Vert \Z^j - \A \right\Vert_{op} 
    &= \sqrt{\zeta^{2j}}\\
    &= \zeta^{j} .
\end{align}

Similarly for $\V$:
\begin{align}
    \left\Vert \V - \A \right\Vert_{op} 
    &= \sqrt{\lambda_{max}((\V - \A)^T(\V - \A))} \\
    &= \sqrt{\lambda_{max}(\V - \A\V - \V \A + \A)} \label{aexpv.eq} \\
    &= \sqrt{\lambda_{max}(\V - \A)} \label{commutev.eq} \\
\end{align}
where (\ref{aexpv.eq}) follows from $\A^j = \A$ and $\V^j = \V$, 
and (\ref{commutev.eq}) follows from $\A\V = \V\A = \A$.

Note that the eigenvalues of each block $\V^{(d)}$ 
are $N^{(d)}-1$ zeros and a one. The set of eigenvalues of $\V$ will
include $D$ ones.
If $D > 1$, then 
based on Lemma 6 of \citet{rotaru2004dynamic} and 
Proposition~\ref{Za.prop}, the spectral norm of $\V - \A$ is $1$, so
\begin{align}
    \left\Vert \V - \A \right\Vert_{op} 
    &= \sqrt{1}\\
    &= 1.
\end{align}

Since the eigenvalues of $\I$ are all $1$, and $\I$ is commutable
with $\A$, we can similarly say:
\begin{align}
    \left\Vert \I - \A \right\Vert_{op} 
    &= 1.
\end{align}
\end{proof}

\begin{lemma} \label{lemma8}
    Given two matrices $\Cm \in \mathbb{R}^{N \times M}$ and
    $\Dm \in \mathbb{R}^{M \times N}$, and an $N$-vector $\bm a$, 
\begin{align}
    \left| \Tr (\left(\diag(\bm a)\right)^{1/2} \Cm  \Dm \left(\diag(\bm a)\right)^{1/2}) \right| \leq 
    \Vert \Cm \Vert_{\F} \Vert \Dm \Vert_{\F} .
\end{align}
\end{lemma}
\begin{proof}
    We define the $i$-th row of $\Cm$ as $\bm c_i^T$ and the $i$-th column of
    $\Dm$ as $\bm d_i$. We can rewrite the trace as:
\begin{align}
    \Tr (\left(\diag(\bm a)\right)^{1/2} \Cm  \Dm \left(\diag(\bm a)\right)^{1/2}) &= 
        \sum_{i=1}^N \sum_{j=1}^M \bm a_i \Cm_{i,j} \Dm_{j,i} \\
        &=  \sum_{i=1}^N  \bm a_i \bm c_i^T \bm d_i .
    \label{rowcolsum.eq}
\end{align}
Placing a squared norm around (\ref{rowcolsum.eq}), we can apply the Cauchy-Schwartz inequality:
\begin{align}
    \left| \sum_{i=1}^N \bm a_i \bm c_i^T \bm d_i \right|^2 
    &\leq \left( \sum_{i=1}^N \bm a_i \Vert \bm c_i^T \Vert^2  \right)
        \left( \sum_{i=1}^N \bm a_i \Vert \bm d_i \Vert^2 \right)\\
    &= \left( \sum_{i=1}^N \sum_{j=1}^M \bm a_i \Cm_{i,j}^2  \right)
        \left( \sum_{i=1}^N \sum_{j=1}^M \bm a_i \Dm_{i,j}^2 \right)\\
    &= \Vert \Cm \Vert_{\F}^2 \Vert \Dm \Vert_{\F}^2 .
\end{align}
\end{proof}

\begin{lemma} \label{lemma7}
    Given two matrices $\Cm \in \mathbb{R}^{M \times N}$ and
    $\Dm \in \mathbb{R}^{N \times N}$, and an $N$-vector $\bm a$, 
    then
\begin{align}
    \Vert \Cm \Dm \Vert_{F_{\bm a}} &\leq 
        \Vert \Cm \Vert_{F_{\bm a}} \Vert \Dm \Vert_{op}.
\end{align}
\end{lemma}
\begin{proof}
    We define the $i$-th row of $\Cm$ as $\bm c_i^T$ and 
    the set $\mathcal{I}=\{i \in [1,M] : \Vert \bm c_i^T \Vert \neq 0 \}$. 
    We can rewrite the squared Frobenius norm as:
\begin{align}
    \Vert \Cm \Dm \Vert_{F_{\bm a}}^2 
    &= \sum_{i=1}^M \Vert \bm c_i^T \Dm \left(\diag(\bm a)\right)^{1/2} \Vert^2 \\
    &= \sum_{i \in \mathcal{I}}^M \Vert \bm c_i^T \Dm \left(\diag(\bm a)\right)^{1/2} \Vert^2 \\
    &= \sum_{i \in \mathcal{I}}^M 
        \Vert \bm c_i^T \left(\diag(\bm a)\right)^{1/2} \Vert^2
        \frac{\Vert \bm c_i^T \Dm \left(\diag(\bm a)\right)^{1/2} \Vert^2}
        {\Vert \bm c_i^T \left(\diag(\bm a)\right)^{1/2} \Vert^2}\\
    &\leq \sum_{i \in \mathcal{I}}^M 
        \Vert \bm c_i^T \left(\diag(\bm a)\right)^{1/2} \Vert^2
        \Vert \Dm \Vert_{op}^2\\
    &= \Vert \Cm \Vert_{F_{\bm a}}^2\Vert \Dm \Vert_{op}^2 .
\end{align}
\end{proof}

\subsection{Proof of Theorem \ref{main.thm}} \label{proof.sec}
We recall Theorem~\ref{main.thm}.
\main*

We now give the proof of Theorem \ref{main.thm} using 
Lemmas \ref{lemma4}-\ref{lemma7}.

\begin{proof}
    Using our intermediate result from Lemma \ref{lemma3}, 
    we decompose $\X_k(\I - \A)$ using our recursive definition of $\X_k$:
\begin{align}
    \X_k(\I - \A) 
    &= (\X_{k-1} - \eta \G_{k-1})\T_{k-1}(\I - \A) \\
    &= \X_{k-1}(\I - \A)\T_{k-1} - \eta \G_{k-1}(\T_{k-1} - \A) \label{commutetk.eq} \\
    &= [(\X_{k-2} - \eta \G_{k-2})\T_{k-2}(\I - \A)]
        \T_{k-1} - \eta \G_{k-1}(\T_{k-1} - \A) \\
    &= [\X_{k-2}(\I - \A)\T_{k-2} - \eta \G_{k-2}(\T_{k-2} - \A)]
        \T_{k-1} - \eta \G_{k-1}(\T_{k-1} - \A) \\
    &= \X_{k-2}(\I - \A)\T_{k-2}\T_{k-1} 
        - \eta \G_{k-2}(\T_{k-2}\T_{k-1} - \A)
         - \eta \G_{k-1}(\T_{k-1} - \A).
\end{align}
where (\ref{commutetk.eq}) follows from the commutability of 
$\T_k$ and $\A$ by Proposition~\ref{TkA.prop}.

Continuing this, we end up with:
\begin{align}
    \X_k(\I - \A) 
    &= \X_{1}(\I - \A)\prod_{l=1}^{k-1} \T_l
    - \eta \sum_{s=1}^{k-1} \G_s \left(\prod_{l=s}^{k-1}\T_l - \A\right).
\end{align}

Since all workers initialize their models to the same vector, 
$\X_{1}(\I - \A)\prod_{l=1}^{k-1}\T_k=\textbf{0}$, and thus we have:
\begin{align}
    \mathbb{E} \left\Vert \X_k(\I - \A) \right\Vert_{F_{\bm a}}^2
    &= \eta^2 \mathbb{E} \left\Vert\sum_{s=1}^{k-1} 
    \G_s \left(\prod_{l=s}^{k-1}\T_l - \A\right)\right\Vert_{F_{\bm a}}^2.
    \label{phi_G.eq}
\end{align}

Let $k=jq\tau + l\tau + f$, where $j$ is the number of
hub network averaging rounds, $l$ is the number of sub-network averaging rounds
since the last hub network averaging round, and $f$ is the number of local iterations
since the last sub-network averaging round.
Define:
 \[
 \Phi_{s,k-1} = \prod_{l=s}^{k-1}\T_l .
 \] 

Noting that $\V^j = \V$, 
and $\V \Z = \Z \V = \Z$ by Proposition~\ref{ZV.prop}, $\Phi_{s,k-1}$ can be expressed as:
\begin{align}
\Phi_{s,k-1} = 
\begin{cases}
    \I & jq\tau+l\tau < s < jq\tau+l\tau + f \\
    \V & jq\tau < s \leq jq\tau+l\tau \\
    \Z & (j-1)q\tau < s \leq jq\tau \\ 
    \Z^2 & (j-2)q\tau < s \leq (j-1)q\tau \\
 \vdots \\
    \Z^{j} & 1 \leq s \leq q\tau .
\end{cases}
\end{align}

 For $r < j$, let
 \[
\Y_r = 
  \sum_{s=rq\tau+1}^{(r+1)q\tau} \G_s  \\
 ,~~~~\Q_r = 
  \sum_{s=rq\tau+1}^{(r+1)q\tau} \nabla F(\X_s) \\
 \]
We also let $\Y_{j_1} = \sum_{s=jq\tau+1}^{jq\tau+l\tau} \G_s$, 
$\Y_{j_2} = \sum_{s=jq\tau+l\tau+1}^{jq\tau+l\tau+f} \G_s$,
$\Q_{j_1} = \sum_{s=jq\tau+1}^{jq\tau+l\tau} \nabla F(\X_s)$,
and $\Q_{j_2} = \sum_{s=jq\tau+l\tau+1}^{jq\tau+l\tau+f} \nabla F(\X_s)$.
With this in mind, 
we can split the sum in (\ref{phi_G.eq}) into 
batches for each hub network averaging period: 
\begin{align}
    \sum_{s=1}^{q\tau}\G_s \left(\Phi_{s,k-1} - \A\right) 
    &= \Y_0(\Z^j - \A) \\  
    \sum_{s=q\tau+1}^{2q\tau}\G_s \left(\Phi_{s,k-1} - \A\right) 
    &= \Y_1(\Z^{j-1} - \A) \\  
    ... \nonumber \\
    \sum_{s=(j-1)q\tau+1}^{jq\tau}\G_s \left(\Phi_{s,k-1} - \A\right) 
    &= \Y_{j-1}(\Z - \A) \\  
    \sum_{s=jq\tau+1}^{jq\tau+l\tau+f}\G_s \left(\Phi_{s,k-1} - \A\right) 
    &= \Y_{j_1}(\V - \A) + \Y_{j_2}(\I - \A).
\end{align}
Summing this all together, we get:
\begin{align}
    \sum_{s=1}^{k-1}\G_s \left(\Phi_{s,k-1} - \A\right) 
    &= \sum_{r=0}^{j-1} \Y_r(\Z^{j-r} - \A) 
    + \Y_{j_1}(\V - \A) + \Y_{j_2}(\I - \A). \label{sumys.eq}
\end{align}

Plugging (\ref{sumys.eq}) into (\ref{phi_G.eq}):
\begin{align}
 &   \mathbb{E} \left\Vert \X_k(\I - \A) \right\Vert_{F_{\bm a}}^2
   = \eta^2\mathbb{E} \left\Vert \sum_{r=0}^{j-1} 
        \Y_r(\Z^{j-r} - \A) + \Y_{j_1}(\V - \A) + \Y_{j_2}(\I - \A)
        \right\Vert_{F_{\bm a}}^2 \\
    &= \eta^2\mathbb{E} \Bigg\Vert \sum_{r=0}^{j-1} 
        (\Y_r - \Q_r)(\Z^{j-r} - \A) 
        + (\Y_{j_1}-\Q_{j_1})(\V - \A) \nonumber \\
 &~~~+(\Y_{j_2}-\Q_{j_2})(\I - \A)
        + \sum_{r=0}^{j-1} \Q_r(\Z^{j-r} - \A) 
        + \Q_{j_1}(\V - \A)+\Q_{j_2}(\I - \A)\Bigg\Vert_{F_{\bm a}}^2 \\
    &\leq \underbrace{2\eta^2\mathbb{E} \Bigg\Vert 
        \sum_{r=0}^{j-1} (\Y_r - \Q_r)(\Z^{j-r} - \A) 
        + (\Y_{j_1}-\Q_{j_1})(\V - \A)
        +(\Y_{j_2}-\Q_{j_2})(\I - \A)) 
        \Bigg\Vert_{F_{\bm a}}^2}_{T_1} \nonumber \\
      &~~~+ \underbrace{2\eta^2\mathbb{E}\Bigg\Vert
        \sum_{r=0}^{j-1} \Q_r(\Z^{j-r} - \A)
        + \Q_{j_1}(\V - \A)+\Q_{j_2}(\I - \A)\Bigg\Vert_{F_{\bm a}}^2}_{T_2}
        \label{t1t2.eq}
\end{align}
where (\ref{t1t2.eq}) follows from the fact that $\| \bm y + \bm z\|^2 \leq 2\| \bm y\|^2 + 2\| \bm z\|^2$. 

We first put a bound on $T_1$: 
\begin{align}
    &T_1 
    = 2\eta^2\mathbb{E} \left\Vert 
        \sum_{r=0}^{j-1} (\Y_r - \Q_r)(\Z^{j-r} - \A) 
        + (\Y_{j_1}-\Q_{j_1})(\V - \A)+(\Y_{j_2}-\Q_{j_2})(\I - \A) 
        \right\Vert_{F_{\bm a}}^2\\
    &= 2\eta^2 \Bigg(\sum_{r=0}^{j-1}
        \mathbb{E} \left\Vert (\Y_r - \Q_r) 
        (\Z^{j-r} - \A) \right\Vert_{F_{\bm a}}^2
        + \mathbb{E} \left\Vert (\Y_{j_1} - \Q_{j_1})
        (\V - \A)  \right\Vert_{F_{\bm a}}^2
        \nonumber \\ &~~~~~~~~~~~~~~~~~~~~~~~~~~~~~~~~~~~~~~~~~~~~~~~~~~~~~~~~~~~~~~~~~~~~~~~~~~~~~~~~
        + \mathbb{E} \left\Vert (\Y_{j_2} - \Q_{j_2})
        (\I - \A) \right\Vert_{F_{\bm a}}^2 \Bigg) 
        \nonumber \\
    &
        + \underbrace{2\eta^2\sum_{n=0}^{j-1}\sum_{l=0,l \neq n}^{j-1}
        \mathbb{E} \left| \underbrace{\Tr\left(
        \left(\diag(\bm a)\right)^{1/2}(\Z^{j-n} - \A)(\Y_n - \Q_n)^T 
        (\Y_l - \Q_l)(\Z^{j-l} - \A)\left(\diag(\bm a)\right)^{1/2}\right)}_{TR}\right|}_{TR_0}
        \nonumber \\
    &
        + \underbrace{4\eta^2\sum_{l=0}^{j-1}
        \mathbb{E} \left| \Tr\left(
        \left(\diag(\bm a)\right)^{1/2}(\V - \A)(\Y_{j_1} - \Q_{j_1})^T  
        (\Y_l - \Q_l)(\Z^{j-l} - \A)\left(\diag(\bm a)\right)^{1/2}\right)\right|}_{TR_1}
        \nonumber \\
    &
        + \underbrace{4\eta^2\sum_{l=0}^{j-1}
        \mathbb{E} \left| \Tr\left(
        \left(\diag(\bm a)\right)^{1/2}(\I - \A)(\Y_{j_2} - \Q_{j_2})^T  
        (\Y_l - \Q_l)(\Z^{j-l} - \A)\left(\diag(\bm a)\right)^{1/2}\right)\right|}_{TR_2}
        \nonumber \\
    &~~~
        + \underbrace{4\eta^2
        \mathbb{E} \left| \Tr\left(
        \left(\diag(\bm a)\right)^{1/2}(\V - \A)(\Y_{j_1} - \Q_{j_1})^T 
        (\Y_{j_2} - \Q_{j_2})(\I - \A)\left(\diag(\bm a)\right)^{1/2}\right)\right|}_{TR_3}.
        \label{needtrbound1.eq}
\end{align}

$TR$ can be bounded as:
\begin{align}
    TR
    &\leq \left\Vert(\Z^{j-n} - \A)(\Y_n - \Q_n)^T\right\Vert_{F_{\bm a}} 
    \left\Vert (\Y_l - \Q_l)(\Z^{j-l} - \A) \right\Vert_{F_{\bm a}} \label{uselemma81.eq}\\
    &\leq \left\Vert(\Z^{j-n} - \A)\right\Vert_{op} 
        \left\Vert \Y_n - \Q_n\right\Vert_{F_{\bm a}} 
        \left\Vert \Y_l - \Q_l \right\Vert_{F_{\bm a}}
        \left\Vert (\Z^{j-l} - \A) \right\Vert_{op} \label{uselemma71.eq}  \\
    &\leq \zeta^{2j - n - l} 
        \left\Vert \Y_n - \Q_n\right\Vert_{F_{\bm a}}
        \left\Vert \Y_l - \Q_l \right\Vert_{F_{\bm a}} \label{uselemma91.eq}\\
    &\leq \frac{1}{2}\zeta^{2j - n - l} \left[ 
        \left\Vert \Y_n - \Q_n\right\Vert_{F_{\bm a}}^2 +
        \left\Vert \Y_l - \Q_l \right\Vert_{F_{\bm a}}^2 \right]
\end{align}
where (\ref{uselemma81.eq}) follows from Lemma~\ref{lemma8},
(\ref{uselemma71.eq}) follows from Lemma~\ref{lemma7},
and (\ref{uselemma91.eq}) follows from Lemma~\ref{lemma9}.
We can similarly bound $TR_1$ and $TR_3$:
\begin{align}
    TR_1 &\leq
    2\eta^2\sum_{l=0}^{j-1} \zeta^{j - l} \left[
        \mathbb{E} \left\Vert \Y_{j_1} - \Q_{j_1}\right\Vert_{F_{\bm a}}^2 + 
        \mathbb{E} \left\Vert \Y_l - \Q_l \right\Vert_{F_{\bm a}}^2 \right]\\
    TR_2 &\leq
    2\eta^2\sum_{l=0}^{j-1} \zeta^{j - l} \left[
        \mathbb{E} \left\Vert \Y_{j_2} - \Q_{j_2}\right\Vert_{F_{\bm a}}^2 +
        \mathbb{E} \left\Vert \Y_l - \Q_l \right\Vert_{F_{\bm a}}^2 \right]\\
    TR_3 &\leq
    2\eta^2  \left[
        \mathbb{E} \left\Vert \Y_{j_1} - \Q_{j_1}\right\Vert_{F_{\bm a}}^2 + 
        \mathbb{E} \left\Vert \Y_{j_2} - \Q_{j_2}\right\Vert_{F_{\bm a}} \right].
\end{align}

Summing $TR_0$ through $TR_3$, we get:
\begin{align}
    \sum_{t=0}^3 TR_t 
    &\leq
        \eta^2\sum_{n=0}^{j-1}\sum_{l=0,l \neq n}^{j-1}
        \zeta^{2j - n - l}
        \left[ \mathbb{E} \left\Vert \Y_n - \Q_n\right\Vert_{F_{\bm a}}^2 
        + \mathbb{E} \left\Vert \Y_l - \Q_l \right\Vert_{F_{\bm a}}^2 \right] 
    \nonumber \\ &~~~~~~~~~~~~~~~~~~~
        + 2\eta^2 
        \sum_{l=0}^{j-1} \zeta^{j - l}  
        \left[\mathbb{E} \left\Vert \Y_{j_1} - \Q_{j_1}\right\Vert_{F_{\bm a}}^2 
        +  \mathbb{E} \left\Vert \Y_l - \Q_l \right\Vert_{F_{\bm a}}^2 \right]
    \nonumber \\ &~~~~~~~~~~~~~~~~~~~
        + 2\eta^2 
        \sum_{l=0}^{j-1} \zeta^{j - l}  
        \left[\mathbb{E}\left\Vert \Y_{j_2} - \Q_{j_2}\right\Vert_{F_{\bm a}}^2 
        +  \mathbb{E}\left\Vert \Y_l - \Q_l \right\Vert_{F_{\bm a}}^2 \right]
    \nonumber \\ &~~~~~~~~~~~~~~~~~~~
        + 2\eta^2 \mathbb{E}\left\Vert \Y_{j_1} - \Q_{j_1}\right\Vert_{F_{\bm a}}^2 
        + 2\eta^2 \mathbb{E}\left\Vert \Y_{j_2} - \Q_{j_2} \right\Vert_{F_{\bm a}}^2 \\
    &\leq 
        2\eta^2\sum_{n=0}^{j-1}\sum_{l=0,l \neq n}^{j-1}
        \zeta^{2j - n - l}
        \mathbb{E}\left\Vert \Y_n - \Q_n\right\Vert_{F_{\bm a}}^2 
    \nonumber \\ &~~~
        + 2\eta^2 
        \sum_{l=0}^{j-1} \zeta^{j - l}  
        \mathbb{E}\left\Vert \Y_l - \Q_l\right\Vert_{F_{\bm a}}^2 
        + 2\eta^2 
        \sum_{l=0}^{j-1} \zeta^{j - l}  
        \mathbb{E}\left\Vert \Y_l - \Q_l\right\Vert_{F_{\bm a}}^2 
    \nonumber \\ &~~~~~~~~~~~~~~~~~~~
        + 2\eta^2 
        \sum_{l=0}^{j} \zeta^{j - l}  
        \mathbb{E}\left\Vert \Y_{j_1} - \Q_{j_1}\right\Vert_{F_{\bm a}}^2 
        + 2\eta^2 
        \sum_{l=0}^{j} \zeta^{j - l}  
        \mathbb{E}\left\Vert \Y_{j_2} - \Q_{j_2}\right\Vert_{F_{\bm a}}^2 
    \label{sym.eq} \\
    &= 
        2\eta^2\sum_{n=0}^{j-1} \zeta^{j - n}
        \mathbb{E}\left\Vert \Y_n - \Q_n\right\Vert_{F_{\bm a}}^2 
        \sum_{l=0,l \neq n}^{j-1} \zeta^{j - l} 
        + 4\eta^2 
        \sum_{l=0}^{j-1} \zeta^{j - l}  
        \mathbb{E}\left\Vert \Y_l - \Q_l\right\Vert_{F_{\bm a}}^2 
    \nonumber \\ &~~~~~~~~~~~~~~~~~~~
        + 2\eta^2 
        \mathbb{E}\left\Vert \Y_{j_1} - \Q_{j_1}\right\Vert_{F_{\bm a}}^2 
        \sum_{l=0}^{j} \zeta^{j - l}  
        + 2\eta^2 
        \mathbb{E}\left\Vert \Y_{j_2} - \Q_{j_2}\right\Vert_{F_{\bm a}}^2 
        \sum_{l=0}^{j} \zeta^{j - l}  
\label{trbound1.eq}
\end{align}
where (\ref{sym.eq}) follows from the symmetry of the $n$ and $l$ indices.

Plugging (\ref{trbound1.eq}) back into (\ref{needtrbound1.eq}):
\begin{align}
    T_1 
    &\leq 2\eta^2 \sum_{r=0}^{j-1}\mathbb{E} 
        \left\Vert (\Y_r - \Q_r) \right\Vert_{F_{\bm a}}^2
        \left\Vert (\Z^{j-r} - \A) \right\Vert_{op}^2
        + 2\eta^2 \mathbb{E} \left\Vert (\Y_{j_1} - \Q_{j_1}) \right\Vert_{F_{\bm a}}^2
        \left\Vert \V - \A \right\Vert_{op}^2
        \nonumber \\ &~~~~~~~~
        + 2\eta^2 \mathbb{E} \left\Vert (\Y_{j_2} - \Q_{j_2})\right\Vert_{F_{\bm a}}^2
        \left\Vert \I - \A \right\Vert_{op}^2 
         + 2\eta^2\sum_{n=0}^{j-1} \zeta^{j - n} \mathbb{E}
        \left\Vert \Y_n - \Q_n\right\Vert_{F_{\bm a}}^2 
        \sum_{l=0,l \neq n}^{j-1} \zeta^{j - l} 
    \nonumber \\ &~~~~~~~~~~~~~~~~~~
        + 2\eta^2 
        \mathbb{E}\left\Vert \Y_{j_1} - \Q_{j_1}\right\Vert_{F_{\bm a}}^2 
        \sum_{l=0}^{j} \zeta^{j - l}  
        + 2\eta^2 
        \mathbb{E}\left\Vert \Y_{j_2} - \Q_{j_2}\right\Vert_{F_{\bm a}}^2 
        \sum_{l=0}^{j} \zeta^{j - l} 
    \nonumber \\ &~~~~~~~~~~~~~~~~~~~~~~~~~~~~~~~~~~~~~~~~~~~~~~~~~~~~~~~~~~~~~~~~~~~~~~~~~~~
        + 4\eta^2 
        \sum_{l=0}^{j-1} \zeta^{j - l}  
        \mathbb{E}\left\Vert \Y_l - \Q_l\right\Vert_{F_{\bm a}}^2 
        \label{needlemma71.eq} \\
    &\leq 2\eta^2 \sum_{r=0}^{j-1}\mathbb{E} 
        \left\Vert (\Y_r - \Q_r) \right\Vert_{F_{\bm a}}^2
        \zeta^{2(j-r)}
        + 2\eta^2\mathbb{E} \left\Vert (\Y_{j_1} - \Q_{j_1}) \right\Vert_{F_{\bm a}}^2
        \nonumber \\ &~~~~~~~~
        + 2\eta^2\mathbb{E} \left\Vert (\Y_{j_2} - \Q_{j_2})\right\Vert_{F_{\bm a}}^2
        + 2\eta^2\sum_{n=0}^{j-1} \zeta^{j - n} \mathbb{E}
        \left\Vert \Y_n - \Q_n\right\Vert_{F_{\bm a}}^2 
        \sum_{l=0,l \neq n}^{j-1} \zeta^{j - l} 
    \nonumber \\ &~~~~~~~~~~~~~~~~~~
        + 2\eta^2 
        \mathbb{E}\left\Vert \Y_{j_1} - \Q_{j_1}\right\Vert_{F_{\bm a}}^2 
        \sum_{l=0}^{j} \zeta^{j - l}  
        + 2\eta^2 
        \mathbb{E}\left\Vert \Y_{j_2} - \Q_{j_2}\right\Vert_{F_{\bm a}}^2 
        \sum_{l=0}^{j} \zeta^{j - l} 
    \nonumber \\ &~~~~~~~~~~~~~~~~~~~~~~~~~~~~~~~~~~~~~~~~~~~~~~~~~~~~~~~~~~~~~~~~~~~~~~~~~~~
        + 4\eta^2 
        \sum_{l=0}^{j-1} \zeta^{j - l}  
        \mathbb{E}\left\Vert \Y_l - \Q_l\right\Vert_{F_{\bm a}}^2 
        \label{needcombinej.eq}
\end{align}
where (\ref{needlemma71.eq}) follows from Lemma~\ref{lemma7},
and (\ref{needcombinej.eq}) follows from Lemma~\ref{lemma9}.

We further bound $T_1$:
\begin{align}
    T_1 
    &\leq 2\eta^2 \sum_{r=0}^{j-1}\mathbb{E} 
        \left\Vert (\Y_r - \Q_r) \right\Vert_{F_{\bm a}}^2
        \zeta^{2(j-r)}
        + 2\eta^2\mathbb{E} \left\Vert (\Y_{j_1} - \Q_{j_1}) \right\Vert_{F_{\bm a}}^2
        \nonumber \\ &~~~~~~~~
        + 2\eta^2\mathbb{E} \left\Vert (\Y_{j_2} - \Q_{j_2})\right\Vert_{F_{\bm a}}^2
        + 2\eta^2\sum_{n=0}^{j-1} \zeta^{j - n} \mathbb{E}
        \left\Vert \Y_n - \Q_n\right\Vert_{F_{\bm a}}^2 
        \frac{\zeta}{1-\zeta} 
    \nonumber \\ &~~~~~~~~~~~~~~~~~~
        + 2\eta^2 
        \mathbb{E}\left\Vert \Y_{j_1} - \Q_{j_1}\right\Vert_{F_{\bm a}}^2 
        \frac{1}{1-\zeta} 
        + 2\eta^2 
        \mathbb{E}\left\Vert \Y_{j_2} - \Q_{j_2}\right\Vert_{F_{\bm a}}^2 
        \frac{1}{1-\zeta} 
    \nonumber \\ &~~~~~~~~~~~~~~~~~~~~~~~~~~~~~~~~~~~~~~~~~~~~~~~~~~~~~~~~~~~~~~~~~~~~~~~~~~~
        + 4\eta^2 
        \sum_{l=0}^{j-1} \zeta^{j - l}  
        \mathbb{E}\left\Vert \Y_l - \Q_l\right\Vert_{F_{\bm a}}^2 
        \label{zeta_sum1.eq}\\
    &= 2\eta^2\sum_{r=0}^{j-1} 
        \left(\zeta^{2(j-r)} + 2\zeta^{j-r} + \frac{\zeta^{j-r+1}}{1-\zeta}\right)
        \mathbb{E} \left\Vert(\Y_r - \Q_r)\right\Vert_{F_{\bm a}}^2 
    \nonumber \\ &~~~~~~~~~~~~~~~~~~
        + 2\eta^2 \left(\frac{2-\zeta}{1-\zeta} \right)
        \mathbb{E}\left\Vert \Y_{j_1} - \Q_{j_1}\right\Vert_{F_{\bm a}}^2 
        + 2\eta^2 \left(\frac{2-\zeta}{1-\zeta} \right)
        \mathbb{E}\left\Vert \Y_{j_2} - \Q_{j_2}\right\Vert_{F_{\bm a}}^2 
        \label{needymq.eq}
\end{align}
where (\ref{zeta_sum1.eq}) follows from the summation formulae of a power series: 
\begin{align}
    \sum_{l=0}^{j} \zeta^{j-l} 
    \leq \sum_{l=-\infty}^{j} \zeta^{j - l} 
    \leq \frac{1}{1-\zeta},~~~~~~~~~~
    \sum_{l=0}^{j-1} \zeta^{j-l} 
    \leq \sum_{l=-\infty}^{j-1} \zeta^{j - l} 
    \leq \frac{\zeta}{1-\zeta}. \label{seriessum.eq}
\end{align}

Taking a closer look at $\mathbb{E}\left\Vert(\Y_r - \Q_r)\right\Vert_{F_{\bm a}}^2$
for $0 \leq r < j$:
\begin{align}
    \mathbb{E}\left\Vert(\Y_r - \Q_r)\right\Vert_{F_{\bm a}}^2
    &= \mathbb{E}\left\Vert\sum_{s=rq\tau+1}^{(r+1)q\tau}
        (\G_s - \nabla F(\X_s))\right\Vert_{F_{\bm a}}^2  \\
    &= \sum_{i=1}^N \bm a_i \mathbb{E}\left\Vert\sum_{s=rq\tau+1}^{(r+1)q\tau}
        (g_s^i - \nabla F(\bm x_s^{(i)}))\right\Vert^2  \label{bad2.eq} \\
    &\leq \sum_{i=1}^N \bm a_i q\tau \sum_{s=rq\tau+1}^{(r+1)q\tau}
       \mathbb{E}\left\Vert(g_s^i - \nabla F(\bm x_s^{(i)}))\right\Vert^2 
       \label{tri3.eq} \\
    &\leq q\tau \left( \sum_{i=1}^N \bm a_i \sum_{s=rq\tau+1}^{(r+1)q\tau}
        \left(\bm p_{i}(\beta-1)+1\right)
        \mathbb{E}\left\Vert \nabla F(\bm x_s^{(i)}) \right\Vert^2 \right)
        + q^2\tau^2 \sigma^2\sum_{i=1}^N \bm a_i \bm p_{i} \label{varasp1.eq} \\
    &= q\tau \left( \sum_{i=1}^N \bm a_i \sum_{s=rq\tau+1}^{(r+1)q\tau}
        \left(\bm p_{i}(\beta-1)+1\right)
    \mathbb{E}\left\Vert \nabla F(\bm x_s^{(i)}) \right\Vert^2 \right) + q^2\tau^2 \sigma^2\Ppos .
    \label{ymq.eq}
\end{align}
where 
(\ref{varasp1.eq}) follows from Assumption~\ref{variance.asp}
and (\ref{asp5implication.eq}).

Similarly, for $r=j_1$ and $r=j_2$:
\begin{align}
    \mathbb{E}\left\Vert(\Y_{j_1} - \Q_{j_1})\right\Vert_{F_{\bm a}}^2
    &\leq l\tau \left( \sum_{i=1}^N \bm a_i \sum_{s=jq\tau+1}^{jq\tau+l\tau}
        \left(\bm p_{i}(\beta-1)+1\right)
        \mathbb{E}\left\Vert \nabla F(\bm x_s^{(i)}) \right\Vert^2 \right)
        + l^2\tau^2\sigma^2\Ppos
    \label{ymqj1.eq}\\
    \mathbb{E}\left\Vert(\Y_{j_2} - \Q_{j_2})\right\Vert_{F_{\bm a}}
    &\leq (f-1)\left(\sum_{i=1}^N \bm a_i \sum_{s=jq\tau+l\tau+1}^{jq\tau+l\tau+f-1}
        \left(\bm p_{i}(\beta-1)+1\right)
        \mathbb{E}\left\Vert \nabla F(\bm x_s^{(i)}) \right\Vert^2  \right)
    \nonumber \\ &~~~~~~~~~~~~~~~~~~~~~~~~~~~~~~~~~~~~~~~~~~~~~~~~~~~~~~~~~~~~~~~~~~~~~~~~~~
        + (f-1)^2\sigma^2\Ppos .
    \label{ymqj2.eq}
\end{align}


Plugging (\ref{ymq.eq}), (\ref{ymqj1.eq}), and (\ref{ymqj2.eq}) into (\ref{needymq.eq}), 
we can bound $T_1$ as follows:
\begin{align}
    T_1 
    &\leq 
    2\eta^2\sigma^2 
        \left(\left(q^2\tau^2 \sum_{r=0}^{j-1} 
        \left(\zeta^{2(j-r)} + 2\zeta^{j-r} + \frac{\zeta^{j-r+1}}{1-\zeta}\right) \right)
        + \left(\frac{2-\zeta}{1-\zeta}\right)\left(l^2\tau^2+(f-1)^2\right)\right)
        \Ppos   
    \nonumber \\  &~~~
        + 2\eta^2 q\tau \sum_{r=0}^{j-1}  
        \left(\zeta^{2(j-r)} + 2\zeta^{j-r} + \frac{\zeta^{j-r+1}}{1-\zeta}\right)
        \sum_{s=rq\tau+1}^{(r+1)q\tau} \sum_{i=1}^N \bm a_i 
        \left(\bm p_{i}(\beta-1)+1\right)\mathbb{E}\left\Vert \nabla F(\bm x_s^{(i)}) \right\Vert^2
    \nonumber \\  &~~~
        + 2\eta^2 \left(\frac{2-\zeta}{1-\zeta} \right) l\tau \left(  
        \sum_{i=1}^N \bm a_i \sum_{s=jq \tau+1}^{jq\tau+l\tau}
        \left(\bm p_{i}(\beta-1)+1\right)\mathbb{E}\left\Vert \nabla F(\bm x_s^{(i)}) \right\Vert^2 \right)
    \nonumber \\  &~~~
        + 2\eta^2 \left(\frac{2-\zeta}{1-\zeta} \right) (f-1)\left(  
                \sum_{i=1}^N \bm a_i \sum_{s=jq\tau+l\tau+1}^{jq\tau+l\tau+f-1}
        \left(\bm p_{i}(\beta-1)+1\right)\mathbb{E}\left\Vert \nabla F(\bm x_s^{(i)}) \right\Vert^2 \right).
\end{align}

Referring back to Lemma~\ref{lemma3}, our goal is to sum $T_1$ over
$k=1, \ldots, K$ iterations. First, we sum over the $j$-th sub-network 
update period up to the $j$-th hub network averaging, 
for ${l=0, \ldots ,q-1}$ and ${f=1, \ldots,\tau}$:
\begin{align}
 &  \sum_{l=0}^{q-1} \sum_{f=1}^{\tau} T_1
    \leq 2\eta^2\sigma^2 q^3\tau^3 
        \sum_{r=0}^{j-1} \left(\zeta^{2(j-r)} + 2\zeta^{j-r} + \frac{\zeta^{j-r+1}}{1-\zeta}\right)\Ppos \nonumber \\
  &~~~~~~~  + 2\eta^2\sigma^2 \left( \frac{2-\zeta}{1-\zeta} \right)
        \left(\tau^3\frac{q(q-1)(2q-1)}{6}
        +q\frac{\tau(\tau-1)(2\tau-1)}{6}\right)  \Ppos \nonumber \\
 &~~~~~~~  + 2\eta^2 q^2\tau^2 \sum_{r=0}^{j-1} 
        \left(\zeta^{2(j-r)} + 2\zeta^{j-r} + \frac{\zeta^{j-r+1}}{1-\zeta}\right) 
        \sum_{s=rq\tau+1}^{(r+1)q\tau} \sum_{i=1}^N \bm a_i 
        \left(\bm p_{i}(\beta-1)+1\right)\mathbb{E}\left\Vert \nabla F(\bm x_s^{(i)}) \right\Vert^2 \nonumber \\
&~~~~~~~ + \eta^2 \left( \frac{2-\zeta}{1-\zeta} \right)   q(q-1)\tau^2
        \sum_{i=1}^N \bm a_i  \sum_{s=jq \tau+1}^{j(q\tau+1)}
        \left(\bm p_{i}(\beta-1)+1\right)\mathbb{E}\left\Vert \nabla F(\bm x_s^{(i)}) \right\Vert^2  \nonumber \\
&~~~~~~~ + \eta^2 \left(\frac{2-\zeta}{1-\zeta} \right)  q^2\tau(\tau-1)
        \sum_{i=1}^N \bm a_i  \sum_{s=j(q\tau+1)+1}^{j(q\tau+1)+\tau-1}
        \left(\bm p_{i}(\beta-1)+1\right)\mathbb{E}\left\Vert \nabla F(\bm x_s^{(i)}) \right\Vert^2 .
        \label{T1bound1.eq}
\end{align}

Let: 
\begin{align}
\Gamma_r = 
        \left(\zeta^{2(j-r)} + 2\zeta^{j-r} + \frac{\zeta^{j-r+1}}{1-\zeta}\right).
\end{align}
Note that $\Gamma_j = \frac{3-2\zeta}{1-\zeta} > \frac{2-\zeta}{1-\zeta}$.
Using this inequality, we can bound the sum of the last three terms of (\ref{T1bound1.eq})  to get
$2q^2\tau^2 \sum_{r=0}^{j}\Gamma_r$:
\begin{align}
 &  \sum_{l=0}^{q-1} \sum_{f=1}^{\tau} T_1
    \leq 2\eta^2\sigma^2 q^3\tau^3 
        \sum_{r=0}^{j-1} \left(\zeta^{2(j-r)} + 2\zeta^{j-r} + \frac{\zeta^{j-r+1}}{1-\zeta}\right)\Ppos \nonumber \\
  &~~~~~~~  + 2\eta^2\sigma^2 \left( \frac{2-\zeta}{1-\zeta} \right)
        \left(\tau^3\frac{q(q-1)(2q-1)}{6}
        +q\frac{\tau(\tau-1)(2\tau-1)}{6}\right)  \Ppos \nonumber \\
 &~~~~~~~  + 2\eta^2q^2\tau^2 \sum_{r=0}^{j} \Gamma_r 
        \sum_{s=rq\tau+1}^{(r+1)q\tau} \sum_{i=1}^N \bm a_i 
        \left(\bm p_{i}(\beta-1)+1\right)\mathbb{E}\left\Vert \nabla F(\bm x_s^{(i)}) \right\Vert^2 .
        \label{T1bound12.eq}
\end{align}

Summing (\ref{T1bound12.eq}) over the hub network averaging periods $j=0, \ldots ,K/(q\tau)-1$, we obtain:
\begin{align}
&    \sum_{j=0}^{K/(q\tau)-1}  \sum_{l=0}^{q-1} \sum_{f=1}^{\tau} T_1 
    \leq 
    2\eta^2\sigma^2  
        q^3\tau^3 
            \sum_{j=0}^{K/(q\tau)-1}\sum_{r=0}^{j-1} 
            \left(\zeta^{2(j-r)} + 2\zeta^{j-r} + \frac{\zeta^{j-r+1}}{1-\zeta}\right) \Ppos 
    \nonumber \\ &~~~
            + 2\eta^2\sigma^2 K\left( \frac{2-\zeta}{1-\zeta} \right)
            \left(\tau^2\frac{(q-1)(2q-1)}{6}
            +\frac{(\tau-1)(2\tau-1)}{6}\right) \Ppos 
    \nonumber \\ &~~~
        + 2\eta^2q^2\tau^2 \sum_{j=0}^{K/(q\tau)-1}
        \sum_{r=0}^{j} \Gamma_r
        \sum_{s=rq\tau+1}^{(r+1)q\tau} \sum_{i=1}^N \bm a_i 
        \left(\bm p_{i}(\beta-1)+1\right)\mathbb{E}\left\Vert \nabla F(\bm x_s^{(i)}) \right\Vert^2 \\
    &= 2\eta^2\sigma^2  
        q^3\tau^3 
            \sum_{r=0}^{K/(q\tau)-2}\sum_{j=r+1}^{K/(q\tau)-1} 
            \left(\zeta^{2(j-r)} + 2\zeta^{j-r} + \frac{\zeta^{j-r+1}}{1-\zeta}\right) \Ppos 
    \nonumber \\ &~~~
            + 2\eta^2\sigma^2 K\left( \frac{2-\zeta}{1-\zeta} \right)
            \left(\tau^2\frac{(q-1)(2q-1)}{6}
            +\frac{(\tau-1)(2\tau-1)}{6}\right) \Ppos 
    \nonumber \\ &~~~
        + 2\eta^2q^2\tau^2 \sum_{r=0}^{K/(q\tau)-1} 
        \left( \sum_{j=r}^{K/(q\tau)-1} \Gamma_j \right)
        \left( \sum_{s=rq\tau+1}^{(r+1)q\tau} \sum_{i=1}^N \bm a_i 
        \left(\bm p_{i}(\beta-1)+1\right)\mathbb{E}\left\Vert \nabla F(\bm x_s^{(i)}) \right\Vert^2 \right)  .
    \label{T1boundsum.eq}
\end{align}

Applying the following summation formula to sum over $\Gamma_j$, we obtain
\begin{align}
\sum_{j=r}^{K/(q\tau)-1}         
\left(\zeta^{2(j-r)} + 2\zeta^{j-r} + \frac{\zeta^{j-r+1}}{1-\zeta}\right) 
&\leq 
\sum_{j=r}^{\infty} 
\left(\zeta^{2(j-r)} + 2\zeta^{j-r} + \frac{\zeta^{j-r+1}}{1-\zeta}\right) \\
&\leq 
\frac{1}{1 - \zeta^2} + \frac{2}{1-\zeta} + \frac{\zeta}{(1-\zeta)^2}.
\label{summation2Gamma.eq}
\end{align}
We let $\Gamma = \frac{1}{1 - \zeta^2} + \frac{2}{1-\zeta} + \frac{\zeta}{(1-\zeta)^2}$. 
We can also apply this following summation formula to the first term in (\ref{T1boundsum.eq}):
\begin{align}
\sum_{j=r+1}^{K/(q\tau)-1}         
\left(\zeta^{2(j-r)} + 2\zeta^{j-r} + \frac{\zeta^{j-r+1}}{1-\zeta}\right) 
&\leq 
\sum_{j=r+1}^{\infty} 
\left(\zeta^{2(j-r)} + 2\zeta^{j-r} + \frac{\zeta^{j-r+1}}{1-\zeta}\right) \\
&\leq 
\frac{\zeta^2}{1 - \zeta^2} + \frac{2\zeta}{1-\zeta} + \frac{1}{(1-\zeta)^2}.
\label{summation2.eq}
\end{align}

Applying the summation formula in (\ref{summation2.eq}), 
plugging $\Gamma$ in, 
and indexing the iterations in terms of $k$, we bound (\ref{T1boundsum.eq}) as:
\begin{align}
    &\sum_{k=1}^{K} T_1 
    \leq 2\eta^2\sigma^2  
        q^3\tau^3 \left(\frac{K}{q\tau}-1 \right) 
        \left(\frac{\zeta^2}{1 - \zeta^2} + \frac{2\zeta}{1-\zeta} + \frac{1}{(1-\zeta)^2} \right) \Ppos
    \nonumber \\ &~~~
            + 2\eta^2\sigma^2 K\left( \frac{2-\zeta}{1-\zeta} \right)
            \left(\tau^2\frac{(q-1)(2q-1)}{6}
            +\frac{(\tau-1)(2\tau-1)}{6}\right) \Ppos 
    \nonumber \\ &~~~
        + 2\eta^2q^2\tau^2 \Gamma 
        \sum_{k=1}^{K} \sum_{i=1}^N \bm a_i 
        \left(\bm p_{i}(\beta-1)+1\right)\mathbb{E}\left\Vert \nabla F(\bm x_k^{(i)}) \right\Vert^2 .
    \label{T1bound2.eq}
\end{align}

%

Now we bound $T_2$:
\begin{align}
    T_2 &= 2\eta^2\mathbb{E}\left\Vert
        \sum_{r=0}^{j-1} \Q_r(\Z^{j-r} - \A)
        + \Q_{j_1}(\V - \A)+\Q_{j_2}(\I - \A)\right\Vert_{F_{\bm a}}^2 \\ 
    &= 2\eta^2\sum_{r=0}^{j-1}\mathbb{E}\left\Vert 
        \Q_r(\Z^{j-r} - \A)\right\Vert_{F_{\bm a}}^2
        +2\eta^2\mathbb{E}\left\Vert \Q_{j_1}(\V - \A)\right\Vert_{F_{\bm a}}^2
        +2\eta^2\mathbb{E}\left\Vert \Q_{j_2}(\I - \A)\right\Vert_{F_{\bm a}}^2  
        \nonumber \\
    &~~~
        + \underbrace{2\eta^2\sum_{n=0}^{j-1}\sum_{l=0,l \neq n}^{j-1}
        \mathbb{E} \left[ \underbrace{\Tr\left(\left(\diag(\bm a)\right)^{1/2}
        (\Z^{j-n} - \A)\Q_n^T \Q_l(\Z^{j-l} - \A)\left(\diag(\bm a)\right)^{1/2}\right)}_{TR'}\right]}_{TR'_0}
        \nonumber \\
    &~~~~~~~~~~~~~~~~~~~~~~~~~
        + \underbrace{4\eta^2\sum_{l=0}^{j-1}
        \mathbb{E} \left[ \Tr\left(\left(\diag(\bm a)\right)^{1/2}
        (\V - \A)\Q_{j_1}^T \Q_l(\Z^{j-l} - \A)\left(\diag(\bm a)\right)^{1/2}\right)\right]}_{TR'_1}
        \nonumber \\
    &~~~~~~~~~~~~~~~~~~~~~~~~~
        + \underbrace{4\eta^2\sum_{l=0}^{j-1}
        \mathbb{E} \left[ \Tr\left(\left(\diag(\bm a)\right)^{1/2}
        (\I - \A)\Q_{j_2}^T \Q_l(\Z^{j-l} - \A)\left(\diag(\bm a)\right)^{1/2}\right)\right]}_{TR'_2}
        \nonumber \\
    &~~~~~~~~~~~~~~~~~~~~~~~~~~~~~
        + \underbrace{4\eta^2
        \mathbb{E} \left[ \Tr\left(\left(\diag(\bm a)\right)^{1/2}
        (\V - \A)\Q_{j_1}^T \Q_{j_2}(\I - \A)\left(\diag(\bm a)\right)^{1/2}\right)\right]}_{TR'_3}.
        \label{needtrbound.eq}
\end{align}

$TR'$ can be bounded as:
\begin{align}
    TR'
    &\leq \left\Vert(\Z^{j-n} - \A)\Q_n^T\right\Vert_{F_{\bm a}} 
    \left\Vert \Q_l(\Z^{j-l} - \A) \right\Vert_{F_{\bm a}} \label{uselemma8.eq}\\
    &\leq \left\Vert(\Z^{j-n} - \A)\right\Vert_{op} 
        \left\Vert \Q_n\right\Vert_{F_{\bm a}} 
        \left\Vert \Q_l \right\Vert_{F_{\bm a}}
        \left\Vert (\Z^{j-l} - \A) \right\Vert_{op} \\
    &\leq \frac{1}{2} \zeta^{2j - n - l} 
        \left[ \left\Vert \Q_n\right\Vert_{F_{\bm a}}^2 
        + \left\Vert \Q_l \right\Vert_{F_{\bm a}}^2 \right] 
\end{align}
where (\ref{uselemma8.eq}) follows from Lemma~\ref{lemma8}.
We can similarly bound $TR'_1$ through $TR'_3$:
\begin{align}
    TR'_1 &\leq
    2\eta^2\sum_{l=0}^{j-1} \zeta^{j - l} \left[
        \mathbb{E}\left\Vert \Q_{j_1}\right\Vert_{F_{\bm a}}^2 + 
        \mathbb{E}\left\Vert \Q_l \right\Vert_{F_{\bm a}}^2 \right]\\
    TR'_2 &\leq
    2\eta^2\sum_{l=0}^{j-1} \zeta^{j - l} \left[
        \mathbb{E}\left\Vert \Q_{j_2}\right\Vert_{F_{\bm a}}^2 + 
        \mathbb{E}\left\Vert \Q_l \right\Vert_{F_{\bm a}}^2 \right] \\
    TR'_3 &\leq
        2\eta^2 \mathbb{E}\left\Vert \Q_{j_1}\right\Vert_{F_{\bm a}}^2 +
        2\eta^2 \mathbb{E}\left\Vert \Q_{j_2}\right\Vert_{F_{\bm a}}^2 .
\end{align}

Summing $TR'_0$ through $TR'_3$, we get:
\begin{align}
    \sum_{t=0}^3 TR'_t &\leq 
        \eta^2\sum_{n=0}^{j-1}\sum_{l=0,l \neq n}^{j-1}
        \zeta^{2j - n - l} \mathbb{E}
        \left[ \mathbb{E}\left\Vert \Q_n\right\Vert_{F_{\bm a}}^2 
        + \mathbb{E}\left\Vert \Q_l \right\Vert_{F_{\bm a}}^2 \right]
    \nonumber \\ &~~~
        + 2\eta^2 
        \sum_{l=0}^{j-1} \zeta^{j - l}  
        \mathbb{E}\left\Vert \Q_l\right\Vert_{F_{\bm a}}^2 
        + 2\eta^2 
        \sum_{l=0}^{j-1} \zeta^{j - l}  
        \mathbb{E}\left\Vert \Q_l\right\Vert_{F_{\bm a}}^2 
    \nonumber \\ &~~~~~~~~~~~~~~~~~~~
        + 2\eta^2 
        \sum_{l=0}^{j} \zeta^{j - l}  
        \mathbb{E}\left\Vert \Q_{j_1}\right\Vert_{F_{\bm a}}^2 
        + 2\eta^2 
        \sum_{l=0}^{j} \zeta^{j - l}  
        \mathbb{E}\left\Vert \Q_{j_2}\right\Vert_{F_{\bm a}}^2 
    \\
    &\leq 2\eta^2\sum_{n=0}^{j-1}\zeta^{j - n}
        \mathbb{E}\left\Vert \Q_n\right\Vert_{F_{\bm a}}^2 
        \sum_{l=0,l \neq n}^{j-1}\zeta^{j - l} 
        + 4\eta^2 
        \sum_{l=0}^{j-1} \zeta^{j - l}  
        \mathbb{E}\left\Vert \Q_l\right\Vert_{F_{\bm a}}^2 
    \nonumber \\ &~~~~~~~~~~~~~~~~~~~
        + 2\eta^2 
        \mathbb{E}\left\Vert \Q_{j_1}\right\Vert_{F_{\bm a}}^2 
        \sum_{l=0}^{j} \zeta^{j - l}  
        + 2\eta^2 
        \mathbb{E}\left\Vert \Q_{j_2}\right\Vert_{F_{\bm a}}^2 
        \sum_{l=0}^{j} \zeta^{j - l}  
\label{tr2bound.eq}
\end{align}
where (\ref{tr2bound.eq}) follows from the symmetry of the indices $n$ and $l$.

Plugging (\ref{tr2bound.eq}) back into (\ref{needtrbound.eq}): 
\begin{align}
    &T_2 
    \leq 2\eta^2\sum_{r=0}^{j-1}\mathbb{E}\left\Vert 
        \Q_r(\Z^{j-r} - \A)\right\Vert_{F_{\bm a}}^2
        +2\eta^2\mathbb{E}\left\Vert \Q_{j_1}(\V - \A)\right\Vert_{F_{\bm a}}^2
        +2\eta^2\mathbb{E}\left\Vert \Q_{j_2}(\I - \A)\right\Vert_{F_{\bm a}}^2  
    \nonumber \\ &~~~~~~~~~~~~~~~~~~~~~~~~
        + 2\eta^2\sum_{n=0}^{j-1}\zeta^{j - n}
        \mathbb{E}\left\Vert \Q_n\right\Vert_{F_{\bm a}}^2 
        \sum_{l=0,l \neq n}^{j-1}\zeta^{j - l} 
        + 4\eta^2 
        \sum_{l=0}^{j-1} \zeta^{j - l}  
        \mathbb{E}\left\Vert \Q_l\right\Vert_{F_{\bm a}}^2 
    \nonumber \\ &~~~~~~~~~~~~~~~~~~~~~~~~~~~~~~~~~~~~~~~~~~~~~~~~~~~~~~~~~
        + 2\eta^2 
        \mathbb{E}\left\Vert \Q_{j_1}\right\Vert_{F_{\bm a}}^2 
        \sum_{l=0}^{j} \zeta^{j - l}  
        + 2\eta^2 
        \mathbb{E}\left\Vert \Q_{j_2}\right\Vert_{F_{\bm a}}^2 
        \sum_{l=0}^{j} \zeta^{j - l} \\ 
    &\leq 2\eta^2\sum_{r=0}^{j-1}\mathbb{E}\left\Vert \Q_r \right\Vert_{F_{\bm a}}^2
    \left\Vert(\Z^{j-r} - \A)\right\Vert_{op}^2
        +2\eta^2\mathbb{E}\left\Vert \Q_{j_1}\right\Vert_{F_{\bm a}}^2 
        \left\Vert \V - \A \right\Vert_{op}^2
        +2\eta^2\mathbb{E}\left\Vert \Q_{j_2} \right\Vert_{F_{\bm a}}^2  
        \left\Vert \I - \A \right\Vert_{op}^2 
    \nonumber \\ &~~~~~~~~~~~~~~~~~~~~~~~~
        + 2\eta^2\sum_{n=0}^{j-1}\zeta^{j - n}
        \mathbb{E}\left\Vert \Q_n\right\Vert_{F_{\bm a}}^2 
        \sum_{l=0,l \neq n}^{j-1}\zeta^{j - l} 
        + 4\eta^2 
        \sum_{l=0}^{j-1} \zeta^{j - l}  
        \mathbb{E}\left\Vert \Q_l\right\Vert_{F_{\bm a}}^2 
    \nonumber \\ &~~~~~~~~~~~~~~~~~~~~~~~~~~~~~~~~~~~~~~~~~~~~~~~~~~~~~~~~~
        + 2\eta^2 
        \mathbb{E}\left\Vert \Q_{j_1}\right\Vert_{F_{\bm a}}^2 
        \sum_{l=0}^{j} \zeta^{j - l}  
        + 2\eta^2 
        \mathbb{E}\left\Vert \Q_{j_2}\right\Vert_{F_{\bm a}}^2 
        \sum_{l=0}^{j} \zeta^{j - l} 
         \label{lemma8t2.eq}\\
     &\leq 2\eta^2\sum_{r=0}^{j-1} \zeta^{j-r}
         \mathbb{E}\left\Vert \Q_r \right\Vert_{F_{\bm a}}^2
        +2\eta^2\mathbb{E}\left\Vert \Q_{j_1}\right\Vert_{F_{\bm a}}^2 
        +2\eta^2\mathbb{E}\left\Vert \Q_{j_2} \right\Vert_{F_{\bm a}}^2  
    \nonumber \\ &~~~~~~~~~~~~~~~~~~~~~~~~
        + 2\eta^2\sum_{n=0}^{j-1}\zeta^{j - n}
        \mathbb{E}\left\Vert \Q_n\right\Vert_{F_{\bm a}}^2 
        \sum_{l=0,l \neq n}^{j-1}\zeta^{j - l} 
        + 4\eta^2 
        \sum_{l=0}^{j-1} \zeta^{j - l}  
        \mathbb{E}\left\Vert \Q_l\right\Vert_{F_{\bm a}}^2 
    \nonumber \\ &~~~~~~~~~~~~~~~~~~~~~~~~~~~~~~~~~~~~~~~~~~~~~~~~~~~~~~~~~
        + 2\eta^2 
        \mathbb{E}\left\Vert \Q_{j_1}\right\Vert_{F_{\bm a}}^2 
        \sum_{l=0}^{j} \zeta^{j - l}  
        + 2\eta^2 
        \mathbb{E}\left\Vert \Q_{j_2}\right\Vert_{F_{\bm a}}^2 
        \sum_{l=0}^{j} \zeta^{j - l} 
         \label{lemma9t2.eq}
\end{align}
where (\ref{lemma8t2.eq}) follows from Lemma~\ref{lemma7},
and (\ref{lemma9t2.eq}) follows from Lemma~\ref{lemma9}.

We further bound $T_2$:
\begin{align}
    &T_2 
    \leq 2\eta^2\sum_{r=0}^{j-1} \zeta^{j-r}
        \mathbb{E}\left\Vert \Q_r \right\Vert_{F_{\bm a}}^2
        +2\eta^2\mathbb{E}\left\Vert \Q_{j_1}\right\Vert_{F_{\bm a}}^2 
        +2\eta^2\mathbb{E}\left\Vert \Q_{j_2} \right\Vert_{F_{\bm a}}^2  
    \nonumber \\ &~~~~~~~~~~~~~~~~~~~~~~~~
        + 2\eta^2\sum_{n=0}^{j-1}\zeta^{j - n}
        \mathbb{E}\left\Vert \Q_n\right\Vert_{F_{\bm a}}^2 
        \frac{\zeta}{1-\zeta}
        + 4\eta^2 
        \sum_{l=0}^{j-1} \zeta^{j - l}  
        \mathbb{E}\left\Vert \Q_l\right\Vert_{F_{\bm a}}^2 
    \nonumber \\ &~~~~~~~~~~~~~~~~~~~~~~~~~~~~~~~~~~~~~~~~~~~~~~~~~~~~~~~~~~~~~~~
        + 2\eta^2 
        \mathbb{E}\left\Vert \Q_{j_1}\right\Vert_{F_{\bm a}}^2 
        \frac{1}{1-\zeta}
        + 2\eta^2 
        \mathbb{E}\left\Vert \Q_{j_2}\right\Vert_{F_{\bm a}}^2 
        \frac{1}{1-\zeta}
        \label{zeta_sum.eq}\\
    &\leq 2\eta^2\sum_{r=0}^{j-1}
        \left(\zeta^{2(j-r)} + 2\zeta^{j-r} + \frac{\zeta^{j-r+1}}{1-\zeta}\right)
        \mathbb{E}\left\Vert \Q_r \right\Vert_{F_{\bm a}}^2
    \nonumber \\ &~~~~~~~~~~~~~~~~~~~~~~~~~~~~
        +2\eta^2 \left(\frac{2-\zeta}{1-\zeta} \right) 
        \mathbb{E}\left\Vert \Q_{j_1}\right\Vert_{F_{\bm a}}^2 
        +2\eta^2 \left(\frac{2-\zeta}{1-\zeta} \right) 
        \mathbb{E}\left\Vert \Q_{j_2} \right\Vert_{F_{\bm a}}^2  
\end{align}
where (\ref{zeta_sum.eq}) follows from the summation formulae of a power series in (\ref{seriessum.eq}).

After applying the definition of $\Q$ to (\ref{zeta_sum.eq}), we obtain:
\begin{align}
    &T_2 
    = 2\eta^2 \sum_{r=0}^{j-1} 
        \left(\zeta^{2(j-r)} + 2\zeta^{j-r} + \frac{\zeta^{j-r+1}}{1-\zeta}\right)
        \mathbb{E}\left\Vert \sum_{s=1}^{q\tau} \nabla F(\X_{rq\tau+s}) \right\Vert_{F_{\bm a}}^2
    \nonumber \\ &~~~
        + 2\eta^2 \left(\frac{2-\zeta}{1-\zeta} \right) 
        \mathbb{E}\left\Vert \sum_{s=1}^{l\tau} 
        \nabla F(\X_{jq\tau+s}) \right\Vert_{F_{\bm a}}^2 
        + 2\eta^2 \left(\frac{2-\zeta}{1-\zeta} \right) 
        \mathbb{E}\left\Vert \sum_{s=1}^{f-1} 
        \nabla F(\X_{jq\tau+l\tau+s}) \right\Vert_{F_{\bm a}}^2 
        \label{t2_prejensen.eq} \\
    &\leq 2\eta^2q\tau \sum_{r=0}^{j-1} 
        \left(\zeta^{2(j-r)} + 2\zeta^{j-r} + \frac{\zeta^{j-r+1}}{1-\zeta}\right)
        \sum_{s=1}^{q\tau} \mathbb{E}\left\Vert \nabla F(\X_{rq\tau+s}) \right\Vert_{F_{\bm a}}^2
    \nonumber \\ &~~~~~~~~~~~~~~~~~~~~~~~~~~~~~~~~~~~~~~~~~~~~~
        + 2\eta^2l\tau \left(\frac{2-\zeta}{1-\zeta} \right) 
        \sum_{s=1}^{l\tau} \mathbb{E}\left\Vert 
        \nabla F(\X_{jq\tau+s}) \right\Vert_{F_{\bm a}}^2 
    \nonumber \\ &~~~~~~~~~~~~~~~~~~~~~~~~~~~~~~~~~~~~~~~~~~~~~~~~~~~~~~
        + 2\eta^2(f-1) \left(\frac{2-\zeta}{1-\zeta} \right) 
        \sum_{s=1}^{f-1} \mathbb{E}\left\Vert 
        \nabla F(\X_{jq\tau+l\tau+s}) \right\Vert_{F_{\bm a}}^2 
        \label{t2_jensen.eq}
\end{align}
where (\ref{t2_jensen.eq}) follows from (\ref{t2_prejensen.eq}) by Jensen's inequality.

Summing over all iterates in the $j$-th sub-network update period, we obtain:
\begin{align}
    \sum_{l=0}^{q-1} \sum_{f=1}^{\tau} T_2 
    &\leq 2\eta^2 q^2\tau^2 \sum_{r=0}^{j-1} \left( 
        \left(\zeta^{2(j-r)} + 2\zeta^{j-r} + \frac{\zeta^{j-r+1}}{1-\zeta}\right)
        \sum_{s=1}^{q\tau} \mathbb{E}\left\Vert 
        \nabla F(\X_{rq\tau+s}) \right\Vert_{F_{\bm a}}^2 \right)
    \nonumber \\ &~~~~~~~~~~~~~~~~~~~~~~~~
        + \eta^2q\tau(q-1) \left(\frac{2-\zeta}{1-\zeta} \right) 
        \sum_{s=1}^{q\tau} \mathbb{E}\left\Vert 
        \nabla F(\X_{jq\tau+s}) \right\Vert_{F_{\bm a}}^2 
    \nonumber \\ &~~~~~~~~~~~~~~~~~~~~~~~~
        + \eta^2q\tau(\tau-1) \left(\frac{2-\zeta}{1-\zeta} \right) 
        \sum_{s=1}^{\tau-1} \mathbb{E}\left\Vert 
        \nabla F(\X_{jq\tau+q\tau+s}) \right\Vert_{F_{\bm a}}^2 \\
    &\leq 2\eta^2 q^2\tau^2 \sum_{r=0}^{j} \Gamma_r 
        \sum_{s=1}^{q\tau} \mathbb{E}\left\Vert 
        \nabla F(\X_{rq\tau+s}) \right\Vert_{F_{\bm a}}^2 .
        \label{zetasum2.eq}
\end{align}

Summing over all iterations and applying the 
summation bound in (\ref{summation2Gamma.eq}) to (\ref{zetasum2.eq}):
\begin{align}
    \sum_{j=0}^{K/(q\tau)-1} \sum_{l=0}^{q-1} \sum_{f=1}^{\tau} T_2
    &\leq  2\eta^2q^2\tau^2 \Gamma 
        \sum_{k=1}^{K} \mathbb{E}\left\Vert 
        \nabla F(\X_{k}) \right\Vert_{F_{\bm a}}^2. 
\end{align}

Summing $T_1$ and $T_2$, we obtain
\begin{align}
    \frac{2L^2}{K}&\sum_{k=1}^K \mathbb{E}\Vert \X_k(\I - \A) \Vert_{F_{\bm a}}^2 
    \leq  \frac{2L^2}{K}\sum_{k=1}^K T_1 
        + \frac{2L^2}{K}\sum_{k=1}^K T_2 \\
    &\leq 
    4L^2 \eta^2\sigma^2  
        q^3\tau^3 \left(\frac{1}{q\tau}-\frac{1}{K} \right) 
        \left(\frac{\zeta^2}{1 - \zeta^2} + \frac{2\zeta}{1-\zeta} + \frac{1}{(1-\zeta)^2} \right) \Ppos
    \nonumber \\ &~~~
        + 4L^2\eta^2\sigma^2 \left( \frac{2-\zeta}{1-\zeta} \right)
        \left(\tau^2\frac{(q-1)(2q-1)}{6}
        +\frac{(\tau-1)(2\tau-1)}{6}\right) \Ppos 
    \nonumber \\ &~~~
        + 8L^2\eta^2q^2\tau^2 \Gamma
        \frac{1}{K}\sum_{k=1}^{K} \sum_{i=1}^N \bm a_i 
        \left(\bm p_{i}(\beta-1)+1\right)\mathbb{E}\left\Vert \nabla F(\bm x_k^{(i)}) \right\Vert^2 .
\end{align}

Plugging $T_1$ and $T_2$ back into Lemma \ref{lemma3}, we arrive at
\begin{align}
    &\mathbb{E}\left[ \frac{1}{K} \sum_{k=1}^K \Vert\nabla F(\bm u_k)\Vert^2 \right]
    \leq 
        \frac{2\left(F(\bm x_1) - F_{inf}]\right)}{\eta K}
        + \sigma^2 \eta L\sum_{i=1}^N  \bm a_i^2 \bm p_i  
        + \frac{2L^2}{K}\sum_{k=1}^K T_1 
        + \frac{2L^2}{K}\sum_{k=1}^K T_2 
    \nonumber \\ &~~~
        - \frac{1}{K}\sum_{k=1}^K
        \sum_{i=1}^N \bm a_i  
        \left((4\bm p_i - \bm p_i^2 - 2) - \eta L \left(\bm a_i\bm p_{i}(\beta+1)-\bm a_i\bm p_i^2 + \bm p_i^2\right) \right)
        \mathbb{E} \left\Vert \nabla F(\bm x_k^{(i)}) \right\Vert^2 \\
    &= 
        \frac{2\left(F(\bm x_1) - F_{inf}]\right)}{\eta K}
        + \sigma^2 \eta L\sum_{i=1}^N  \bm a_i^2 \bm p_i  
    \nonumber \\ &
        + 4L^2 \eta^2\sigma^2  
        q^3\tau^3 \left(\frac{1}{q\tau}-\frac{1}{K} \right) 
        \left(\frac{\zeta^2}{1 - \zeta^2} + \frac{2\zeta}{1-\zeta} + \frac{1}{(1-\zeta)^2} \right) \Ppos
    \nonumber \\ &
        + 4L^2\eta^2\sigma^2 \left( \frac{2-\zeta}{1-\zeta} \right)
        \left(\tau^2\frac{(q-1)(2q-1)}{6}
        +\frac{(\tau-1)(2\tau-1)}{6}\right) \Ppos 
    \nonumber \\ &
        - \frac{1}{K}\sum_{k=1}^K
        \sum_{i=1}^N \bm a_i  
        \left((4\bm p_i - \bm p_i^2 - 2) - \eta L \left(\bm a_i\bm p_{i}(\beta+1)-\bm a_i\bm p_i^2 + \bm p_i^2\right)
        - 8L^2\eta^2q^2\tau^2 \Gamma \right)
        \mathbb{E} \left\Vert \nabla F(\bm x_k^{(i)}) \right\Vert^2 
        \label{needeta.eq}
\end{align}

If $\eta$ satisfies the following for  $i=1, \ldots, N$,
\begin{align}
       (4\bm p_i - \bm p_i^2 - 2) \geq \eta L \left(\bm a_i\bm p_{i}(\beta+1)-\bm a_i\bm p_i^2 + \bm p_i^2\right) + 8L^2\eta^2q^2\tau^2 \Gamma
\end{align}
then we can simplify (\ref{needeta.eq}):
\begin{multline}
    \mathbb{E}\left[ \frac{1}{K} \sum_{k=1}^K \Vert\nabla F(\bm u_k)\Vert^2 \right]
    \leq \frac{2\left(F(\bm x_1) - F_{inf}]\right)}{\eta K}
        + \sigma^2 \eta L\sum_{i=1}^N  \bm a_i^2 \bm p_i  \\
        + 4L^2 \eta^2\sigma^2  
        q^3\tau^3 \left(\frac{1}{q\tau}-\frac{1}{K} \right) 
        \left(\frac{\zeta^2}{1 - \zeta^2} + \frac{2\zeta}{1-\zeta} + \frac{1}{(1-\zeta)^2} \right) \Ppos \\
        + 4L^2\eta^2\sigma^2 \left( \frac{2-\zeta}{1-\zeta} \right)
        \left(\tau^2\frac{(q-1)(2q-1)}{6}
        +\frac{(\tau-1)(2\tau-1)}{6}\right) \Ppos .
\end{multline}
\end{proof}

\subsection{Comparison to Cooperative SGD} \label{coop.sec}
We note that when setting $a_i=1/N$ and $p_i=1$ for all workers $i$, and 
setting $q=1$, MLL-SGD reduces to Cooperative SGD~\citep{wang2018cooperative}. 
However, the bound in Theorem~\ref{main.thm}
differs when compared to the bound of Cooperative SGD. 
Specifically, Theorem~\ref{main.thm} 
has error terms dependent on $\tau^2$ as opposed to $\tau$.

This is due to the formulation of $\bm g_k^{(i)}$. Namely:
\begin{align}
    \mathbb{E}_k[\bm g_k^{(i)}] &= \bm p_{i}\mathbb{E}_k[g(\bm x_k^{(i)})]\\
    &= \bm p_{i}\nabla F(\bm x_k^{(i)}). 
\end{align}

Because we cannot assume $p_i=1$, there are cross terms in the expressions 
in equations (\ref{needtrbound1.eq}) and (\ref{bad2.eq}) that do not cancel out.
Thus, we needed to use a more conservative analysis at these steps on the proof. 
This is the reason that plugging in a value of $p_i=1$ is not enough to 
recover the same bound as in Cooperative SGD. 
A similar discrepancy can be observed when comparing with~\citet{koloskova2020unified}.

\end{document}